\documentclass[nohyperref]{article}

\usepackage{amsthm}
\usepackage{amsmath}
\usepackage{amssymb}
\usepackage{graphicx}
\usepackage{mathtools}
\usepackage{enumerate}
\usepackage{enumitem}
\usepackage{footnote}
\usepackage{float}
\usepackage{xspace}
\usepackage{multirow}
\usepackage{nicefrac}
\usepackage{wrapfig}
\usepackage{framed}
\usepackage{url}
\usepackage{units}
\usepackage[colorlinks=true, linkcolor=blue, citecolor=blue]{hyperref}
\usepackage{blkarray}
\PassOptionsToPackage{round}{natbib}

\usepackage{tikz}
\usetikzlibrary{arrows,chains,matrix,positioning,scopes}

\let\hat\widehat
\let\tilde\widetilde

\newtheorem{lemma}{Lemma}
\newtheorem{theorem}{Theorem}
\newtheorem{cor}{Corollary}
\newtheorem{remark}{Remark}

\newtheorem{definition}{Definition}
\newtheorem{assumption}{Assumption}

\DeclareMathOperator*{\minimize}{minimize}

\newcommand{\DefinedAs}[0]{\mathrel{\mathop:}=}

\newcommand{\calA}{{\mathcal{A}}}

\newcommand{\calS}{{\mathcal{S}}}

\newcommand{\calR}{{\mathcal{R}}}

\newcommand{\calP}{{\mathcal{P}}}

\newcommand{\calV}{{\mathcal{V}}}
\newcommand{\calQ}{{\mathcal{Q}}}

\newcommand{\one}{\boldsymbol{1}}

\newcommand{\coeff}{\tilde{\kappa}}

\DeclareMathOperator*{\argmin}{argmin}
\DeclareMathOperator*{\argmax}{argmax}

\newcommand{\field}[1]{\mathbb{#1}}

\newcommand{\E}{\field{E}}

\newcommand{\diff}{\text{diff}}

\newcommand{\norm}[1]{\left\|{#1}\right\|}

\newcommand{\sbr}[1]{\left[#1\right]}

\DeclareFontFamily{OMX}{MnSymbolE}{}
\DeclareFontShape{OMX}{MnSymbolE}{m}{n}{
    <-6>  MnSymbolE5
   <6-7>  MnSymbolE6
   <7-8>  MnSymbolE7
   <8-9>  MnSymbolE8
   <9-10> MnSymbolE9
  <10-12> MnSymbolE10
  <12->   MnSymbolE12}{}
\DeclareSymbolFont{mnlargesymbols}{OMX}{MnSymbolE}{m}{n}
\SetSymbolFont{mnlargesymbols}{bold}{OMX}{MnSymbolE}{b}{n}
\DeclareMathDelimiter{\llangle}{\mathopen}{mnlargesymbols}{'164}{mnlargesymbols}{'164}
\DeclareMathDelimiter{\rrangle}{\mathclose}{mnlargesymbols}{'171}{mnlargesymbols}{'171}

\usepackage{prettyref}
\newcommand{\pref}[1]{\prettyref{#1}}
\newcommand{\pfref}[1]{Proof of \prettyref{#1}}
\newcommand{\savehyperref}[2]{\texorpdfstring{\hyperref[#1]{#2}}{#2}}
\newrefformat{eq}{\savehyperref{#1}{\textup{(\ref*{#1})}}}
\newrefformat{eqn}{\savehyperref{#1}{Equation~\eqref{#1}}}
\newrefformat{lem}{\savehyperref{#1}{Lemma~\ref*{#1}}}
\newrefformat{def}{\savehyperref{#1}{Definition~\ref*{#1}}}
\newrefformat{line}{\savehyperref{#1}{line~\ref*{#1}}}
\newrefformat{thm}{\savehyperref{#1}{Theorem~\ref*{#1}}}
\newrefformat{corr}{\savehyperref{#1}{Corollary~\ref*{#1}}}
\newrefformat{cor}{\savehyperref{#1}{Corollary~\ref*{#1}}}
\newrefformat{sec}{\savehyperref{#1}{Section~\ref*{#1}}}
\newrefformat{app}{\savehyperref{#1}{Appendix~\ref*{#1}}}
\newrefformat{ap}{\savehyperref{#1}{Appendix~\ref*{#1}}}
\newrefformat{assum}{\savehyperref{#1}{Assumption~\ref*{#1}}}
\newrefformat{as}{\savehyperref{#1}{Assumption~\ref*{#1}}}
\newrefformat{ex}{\savehyperref{#1}{Example~\ref*{#1}}}
\newrefformat{fig}{\savehyperref{#1}{Figure~\ref*{#1}}}
\newrefformat{alg}{\savehyperref{#1}{Algorithm~\ref*{#1}}}
\newrefformat{rem}{\savehyperref{#1}{Remark~\ref*{#1}}}
\newrefformat{conj}{\savehyperref{#1}{Conjecture~\ref*{#1}}}
\newrefformat{prop}{\savehyperref{#1}{Proposition~\ref*{#1}}}
\newrefformat{proto}{\savehyperref{#1}{Protocol~\ref*{#1}}}
\newrefformat{prob}{\savehyperref{#1}{Problem~\ref*{#1}}}
\newrefformat{claim}{\savehyperref{#1}{Claim~\ref*{#1}}}
\newrefformat{que}{\savehyperref{#1}{Question~\ref*{#1}}}
\usepackage{microtype}
\usepackage{graphicx}
\usepackage{subfigure}
\usepackage{booktabs} 
\usepackage[colorinlistoftodos]{todonotes}
\usepackage{multicol}
\usepackage{algorithm}
\usepackage{algorithmic}
\usepackage{float}

\definecolor{applegreen}{rgb}{0.55,0.71, 0.0}
\definecolor{awesome}{rgb}{1.0, 0.13, 0.32}
\definecolor{bittersweet}{rgb}{1.0, 0.44, 0.37}

\usepackage[accepted]{icml2022}

\icmltitlerunning{Independent Policy Gradient for Large-Scale Markov Potential Games
}

\begin{document}

\twocolumn[
\icmltitle{Independent Policy Gradient for Large-Scale Markov Potential Games: \\ Sharper Rates, Function Approximation, and Game-Agnostic Convergence 
}



\icmlsetsymbol{equal}{*}

\begin{icmlauthorlist}
\icmlauthor{Dongsheng Ding}{equal,yyy}
\icmlauthor{Chen-Yu Wei}{equal,yyy}
\icmlauthor{Kaiqing Zhang}{equal,comp}
\icmlauthor{Mihailo~R.~Jovanovi\'{c}}{yyy}
\end{icmlauthorlist}

\icmlaffiliation{yyy}{University of Southern California}
\icmlaffiliation{comp}{Massachusetts Institute of Technology}

\icmlcorrespondingauthor{Dongsheng Ding}{dongshed@usc.edu}
\icmlcorrespondingauthor{Chen-Yu Wei}{chenyu.wei@usc.edu}
\icmlcorrespondingauthor{Kaiqing Zhang}{kaiqing@mit.edu}
\icmlcorrespondingauthor{Mihailo~R.~Jovanovi\'{c}}{mihailo@usc.edu}
\icmlkeywords{Machine Learning, ICML}

\vskip 0.3in
]



\printAffiliationsAndNotice{\icmlAlphabeticalOrder} 



\begin{abstract}

	We examine global non-asymptotic convergence properties of policy gradient  methods for multi-agent reinforcement learning (RL) problems in Markov potential games (MPGs). To learn a Nash equilibrium of an MPG in which the size of state space and/or the number of players can be very large, we propose new independent policy gradient algorithms that are run by all players in tandem. When there is no uncertainty in the gradient evaluation, we show that our algorithm finds an $\epsilon$-Nash equilibrium with $O(1/\epsilon^2)$ iteration complexity which does not explicitly depend on the state space size. When the exact gradient is not available, we establish $O(1/\epsilon^5)$ sample complexity bound in a potentially infinitely large state space for a sample-based algorithm that utilizes function approximation. Moreover, we identify a class of independent policy gradient algorithms that enjoy convergence for both zero-sum Markov games and Markov cooperative games with the players that are oblivious to the types of games being played. Finally, we provide computational experiments to corroborate the merits and the effectiveness of our theoretical developments.
\end{abstract}

\section{Introduction}
\label{sec: introduction}

Multi-agent reinforcement learning (RL) studies how multiple players learn to maximize their long-term returns in a setup where players' actions influence the environment and other agents' returns \cite{busoniu2008comprehensive,zhang2021multi}. Recently, multi-agent RL has achieved significant success in various multi-agent learning scenarios, e.g., competitive game-playing \cite{silver2016mastering,silver2018general,vinyals2019grandmaster}, autonomous robotics \cite{shalev2016safe,levine2016end}, and economic policy-making \cite{zheng2020ai,trott2021building}. In the framework of stochastic games~\cite{shapley1953stochastic,fink1964equilibrium}, most results are established for {\it fully-competitive} (i.e., two-player zero-sum) games; e.g., see~\citet{daskalakis2020independent,Wei2021LastiterateCO,cen2021fast}. However, to achieve social welfare for AI~\cite{dafoe2020open,dafoe2021cooperative,stastny2021normative}, it is imperative to establish theoretical  guarantees for multi-agent RL in Markov games with cooperation.

Policy gradient methods~\cite{williams1992simple,sutton2000policy} have received significant attention for both single-agent~\cite{bhandari2019global,agarwal2021theory} and multi-agent RL problems~\cite{zhang2019policy,daskalakis2020independent,Wei2021LastiterateCO}.  Independent policy gradient ~\cite{zhang2021multi,ozdaglar2021independent} is probably the most practical protocol in multi-agent RL, where each player behaves myopically by only observing her own rewards and actions (as well as the system states), while individually optimizing its own policy. More importantly, independent learning dynamics do not scale exponentially with the number of players in the game. Recently, \citet{daskalakis2020independent,leonardos2021global,zhang2021gradient} have in fact shown that multi-agent RL players could perform policy gradient updates independently, while enjoying global non-asymptotic convergence. However, these results are only focused on the basic tabular setting in which the value functions are represented by tables; they do not carry over to large-scale multi-agent RL problems in which the state space size is potentially infinite and the number of players is large. This motivates the following question:
\begin{center}
	\emph{Can we design independent policy gradient methods for large-scale Markov games,  with non-asymptotic global convergence guarantees?} 
\end{center}
In this paper, we provide the first affirmative answer to this question for a class of mixed cooperative/competitive Markov games: Markov potential games (MPGs)~\cite{macua2018learning,leonardos2021global,zhang2021gradient}. In particular, we make the following contributions:

\begin{itemize}
	\item We propose an independent policy gradient algorithm -- \pref{alg: PMA} -- for learning an $\epsilon$-Nash equilibrium of MPGs with $O(1/\epsilon^2)$ iteration complexity. In contrast to the existing results~\cite{leonardos2021global, zhang2021gradient}, such iteration complexity does not explicitly depend on the {\it state space size}.
	
	\item We consider a linear function approximation setting and design an independent sample-based policy gradient algorithm -- \pref{alg: PMA fa} -- that learns an $\epsilon$-Nash equilibrium with $O(1/\epsilon^5)$ sample complexity. This appears to be the first result for learning MPGs with function approximation.
	
	\item We establish the convergence of an independent optimistic policy gradient algorithm -- \pref{alg: OPMA smooth Q} that has been proved to converge in learning zero-sum Markov games~\cite{Wei2021LastiterateCO} -- for learning a subclass of MPGs: Markov cooperative games. We show that the same type of optimistic policy learning algorithm provides an $\epsilon$-Nash equilibrium in both zero-sum Markov games and Markov cooperative games while the players are oblivious to the types of games being played. To the best of our knowledge, this appears to be the first {\em game-agnostic\/} convergence result in Markov games.
\end{itemize}

We next discuss some related work.

\noindent\textbf{Markov potential games (MPGs)}. In stochastic optimal control, the MPG model dates back to \citet{dechert2006stochastic,gonzalez2013discrete}. More recent studies include~\citet{zazo2016dynamic,mazalov2017linear,macua2018learning,mguni2018decentralised} and all of these studies focus on systems with known dynamics. MPGs have also attracted attention in multi-agent RL. In the infinite-horizon setting, \citet{leonardos2021global,zhang2021gradient} extended the policy gradient method \cite{agarwal2021theory,kakade2001natural} for multiple players and established the iteration/sample complexity that scales with the size of state space; \citet{fox2021independent} generalized the natural policy gradient method \cite{kakade2001natural,agarwal2021theory} and established the global asymptotic convergence. In the finite-horizon setting, \citet{song2021can} built on the single-agent Nash-VI~\cite{liu2021sharp} to propose a sample efficient turn-based algorithm and~\citet{mao2022on} studied the policy gradient method. Earlier,  \citet{wang2002reinforcement,lowe2017multi} studied Markov cooperative games and \citet{kleinberg2009multiplicative,palaiopanos2017multiplicative,cohen2017learning} studied one-state MPGs; both of these are special cases of MPGs. We note that the term: Markov potential game has also been used to refer to state-based potential MDPs~\cite{marden2012state,mguni2021learning}, which are different from the MPGs that we study; see counterexamples in \citet{leonardos2021global}.

\noindent\textbf{Policy gradient methods for Markov games}. Despite recent advances on the theory of policy gradient \cite{bhandari2019global,agarwal2021theory}, the theory of policy gradient methods for multi-agent RL is relatively less studied. In the basic two-player zero-sum Markov games, \citet{zhang2019policy,bu2019global,daskalakis2020independent,zhao2021provably} established global convergence guarantees for policy gradient methods for learning an (approximate) Nash equilibrium. More recently,  \citet{cen2021fast,Wei2021LastiterateCO} examined variants of policy gradient methods and provided last-iterate convergence guarantees. However, it is much harder for the policy gradient methods to work in general Markov games \cite{mazumdar2019policy,hambly2021policy}. The effectiveness of (natural) policy gradient methods for tabular MPGs was demonstrated in~\citet{leonardos2021global,zhang2021gradient,fox2021independent,zhang2022effect}.  Moreover,~\citet{xie2020semicentralized,wang2020off,yu2021surprising,peng2021facmac} reported impressive empirical performance of multi-agent policy gradient methods with function approximation in cooperative Markov games, but the theoretical foundation has not been provided.  

Independent learning recently received attention in multi-agent RL   \cite{daskalakis2020independent,zhang2021multi,ozdaglar2021independent,sayin2021decentralized,jin2021v,song2021can,kao2021decentralized}, because it only requires local information for learning and naturally yields algorithms that scale to a large number of players. The algorithms in 
\citet{leonardos2021global,zhang2021gradient,fox2021independent,zhang2022effect} can also be generally categorized as independent learning algorithms for MPGs. 

\noindent\textbf{Game-agnostic convergence}. Being game-agnostic is a desirable property for independent learning in which players are oblivious to the types of games being played. In particular, classical fictitious-play warrants average-iterate convergence for several games~\cite{robinson1951iterative,ref:Monderer96b,hofbauer2002global}. Although online learning algorithms, e.g., the one based on multiplicative weight updates (MWU)~\cite{cesa2006prediction}, offer average-iterate convergence in zero-sum matrix games, they often do not provide last-iterate
convergence guarantees~\cite{bailey2018multiplicative}, which motivates recent studies~\cite{daskalakis2018limit,mokhtari2020unified,wei2020linear}. Interestingly, while MWU converges in last-iterate for potential games \cite{palaiopanos2017multiplicative,cohen2017learning}, this is not the case for zero-sum matrix games~\cite{cheung2020chaosvideo}. Recently,~\citet{leonardos2021exploration,leonardos2022exploration} established last-iterate convergence of $Q$-learning dynamics for both zero-sum and potential/cooperative matrix games. However, it is open question whether an algorithm can have last-iterate convergence for both zero-sum and potential/cooperative Makov games.

\section{Preliminaries}
\label{sec: problem setup}

In this section, we introduce Markov potential games (MPGs), define the Nash equilibrium, and describe the problem setting.  

We consider an $N$-player, infinite-horizon, discounted Markov potential game \cite{macua2018learning,leonardos2021global,zhang2021gradient},
\begin{equation}\label{eq: Markov potential game}
    \text{MPG}\,(\,\calS,\,\{\calA_i\}_{i\,=\,1}^N,\,\mathbb{P},\, \{r_i\}_{i\,=\,1}^N, \,\gamma, \,\rho \,)
\end{equation}
where $\calS$ is the state space, $\calA_i$ is the action space for the $i$th player, with the joint action space of $N\geq 2$ players denoted as $\calA\DefinedAs\calA_1\times\ldots\times\calA_N$, $\mathbb{P}$ is the transition probability measure specified by a distribution $\mathbb{P}(\,\cdot\,|\,s, a)$ over $\calS$ if $N$ players jointly take an action $a$ from $\calA$ in state $s$, $r_i$: $\calS\times\calA\to [0,1]$ is the $i$th player immediate reward function, $\gamma \in [0,1)$ is the discount factor, and $\rho$ is the initial state distribution over $\calS$. We assume that all action spaces are finite with the same size $A=A_i = |\calA_i|$ for all $i=1,\ldots,N$. It is straightforward to apply our analysis to the general case in which players' finite action spaces have different sizes. 

For the $i$th player, $\Delta(\calA_i)$ represents the probability simplex over the action set $\calA_i$. A stochastic policy for player~$i$ is given by $\pi_i$: $\calS\to\Delta(\calA_i)$ that specifies the action distribution $\pi_i(\cdot\,\vert\,s) \in \Delta(\calA_i)$ for each state $s\in\calS$. The set of stochastic policies for player~$i$ is denoted by $\Pi_i\DefinedAs(\Delta(\calA_i))^{\calS}$, the joint probability simplex is given by $\Delta(\calA)\DefinedAs\Delta(\calA_1)\times\ldots\times\Delta(\calA_N)$, and the joint policy space is $\Pi \DefinedAs (\Delta(\calA))^{\calS}$. A Markov product policy $\pi \DefinedAs \{ \pi_i \}_{i\,=\,1}^N \in \Pi$ for $N$ players consists of the policy $\pi_i\in \Pi$ for all players $i = 1,\ldots,N$. We use the shorthand $\pi_{-i} =  \{\pi_k\}_{k \,=\,1,\,k\,\neq\,i}^N$ to represent the policy of all but the $i$th player. We denote by $V_i^{\pi}$ : $\calS\to\mathbb{R}$ the $i$th player value function under the joint policy $\pi$, starting from an initial state $s^{(0)} = s$:
\[
V_i^{\pi} (s) \;\DefinedAs\; \mathbb{E}^\pi \sbr{\,\sum_{t\,=\,0}^\infty \gamma^t r_i(s^{(t)}, a^{(t)}) \,\bigg|\, s^{(0)}=s \,}
\]
where the expectation $\mathbb{E}^\pi$ is over $a^{(t)}\sim\pi(\cdot\,|\,s^{(t)})$ and $s^{(t+1)}\sim \mathbb{P}(\cdot\,|\,s^{(t)},a^{(t)})$. Finally, $V_i^\pi(\mu)$ denotes the expected value function of $V_i^\pi(s)$ over a state distribution $\mu$, $V_i^\pi(\mu) \DefinedAs \mathbb{E}_{s\,\sim\,\mu} [ V_i^\pi(s) ]$.

In a MPG, at any state $s\in \calS$, there exists a global function -- the potential function $\Phi^{\pi}(s)$: $\Pi\times \calS\to \mathbb{R}$ -- that captures the incentive of all players to vary their policies at state $s$,
\[
V_i^{\pi_i,\, \pi_{-i}} (s) \, - \, V_i^{\pi_i',\, \pi_{-i}} (s)
\; = \;
\Phi^{\pi_i,\, \pi_{-i}}(s) \, - \, \Phi^{\pi_i',\, \pi_{-i}}(s)
\]
for any policies $\pi_i, \pi_i' \in \Pi_i$ and $\pi_{-i} \in \Pi_{-i}$.
Let $\Phi^{\pi}(\mu) \DefinedAs  \mathbb{E}_{s\,\sim\,\mu} [\Phi^{\pi}(s)]$ be the expected potential function over a state distribution $\mu$.
Thus, 
$V_i^{\pi_i,\, \pi_{-i}} (\mu) - V_i^{\pi_i',\, \pi_{-i}} (\mu)
=
\Phi^{\pi_i,\, \pi_{-i}}(\mu) - \Phi^{\pi_i',\, \pi_{-i}}(\mu)$. There always exists a constant $C_\Phi>0$ such that $|\Phi^\pi(\mu)-\Phi^{\pi'}(\mu)| \leq C_\Phi$ for any $\pi, \pi', \mu$; see a trivial upper bound in \pref{lem: bounded phimax} in \pref{ap: auxiliary lemmas}. An important subclass of Markov potential games is given by Markov cooperative games (MCG) in which all players share the same reward function $r = r_i$ for all $i=1,\ldots,N$. 

We also denote by $Q_i^{\pi}:\calS\times\calA\to\mathbb{R}$ the action-value function under policy $\pi$, starting from an initial state-action pair $(s^{(0)}, a^{(0)}) = (s, a)$:
\[
Q_i^{\pi} (s,a) 
  \DefinedAs  
\mathbb{E}^\pi \sbr{\,\sum_{t\,=\,0}^\infty \gamma^t r_i(s^{(t)},a^{(t)}) \,\bigg|\, s^{(0)}=s, a^{(0)}=a} \!\!.
\]
The value function can be equivalently expressed as $V_i^\pi(s) = \sum_{a'\,\in\,\calA} \pi(a'\,\vert\,s)Q_i^{\pi} (s,a')$. For each player $i$, by averaging out $\pi_{-i}$, we can define the averaged action-value function $\bar Q_i^{\pi_i,\,\pi_{-i}}$:  $\calS\times \calA_i \to \mathbb{R}$, 
\[
\bar Q_i^{\pi_i,\,\pi_{-i}} ( s, a_i ) 
  \DefinedAs  \!\!
\sum_{a_{-i} \,\in\, \calA_{-i}}  \pi_{-i}{(a_{-i}\,\vert\,s)}\, Q_i^{\pi_i,\,\pi_{-i}} (s, a_i, a_{-i})
\]
where $\calA_{-i}$ is the set of actions of all but the $i$th player. We use the shorthand $\bar Q_i^{\pi}$ for $\bar Q_i^{\pi_i,\,\pi_{-i}}$ when $\pi_i$ and $\pi_{-i}$ are from the same joint policy $\pi$. It is straightforward to see that $V_i^\pi, Q_i^\pi$, and $\bar Q_i^\pi$ are bounded between $0$ and $1/(1-\gamma)$.

We recall the notion of (Markov perfect stationary) Nash equilibrium~\cite{fink1964equilibrium}.  
A joint policy $\pi^\star$ is called a Nash equilibrium if for each player $i = 1,\ldots,N$,
\[
V_i^{\pi_i^\star,\,\pi_{-i}^\star}(s)
\; \geq \;
V_i^{\pi_i,\,\pi_{-i}^\star}(s),\;\text{ for all } \pi_i \,\in\, \Pi_i,\, s\,\in\, \calS,
\]
and called an $\epsilon$-Nash equilibrium if for $i = 1,\ldots,N$,  
\[
V_i^{\pi_i^\star,\,\pi_{-i}^\star}(s)
\; \geq \;
V_i^{\pi_i,\,\pi_{-i}^\star}(s) \,-\, \epsilon,
\;\text{ for all } \pi_i \,\in\, \Pi_i,\, s\,\in\, \calS.
\]
Nash equilibria for MPGs with finite states and actions always exist~\cite{fink1964equilibrium}. When the state space is infinite, we assume the existence of a Nash equilibrium; see~\citet{takahashi1962stochastic,maitra1970stochastic,maitra1971stochastic,altman1997contraction} for cases with countable or compact state spaces. 

Given policy $\pi$ and initial state $s^{(0)}$, we define the discounted state visitation distribution,
\[
d_{s^{(0)}}^\pi(s)
\;  = \;
(1-\gamma) \sum_{t\,=\,0}^\infty \gamma^t \, \text{Pr}^\pi( \, s^{(t)} = s \, \vert \, s^{(0)} \, ).
\]
For a state distribution $\mu$, define
$d_\mu^\pi (s) = \mathbb{E}_{s^{(0)}\sim\mu} [ \, d_{s^{(0)}}^\pi(s) \,]$. By definition, $d_\mu^\pi(s) \geq (1-\gamma) \mu(s)$ for \mbox{any $\mu$ and $s$.} 

It is useful to introduce a variant of the performance difference lemma~\cite{agarwal2021theory} for multiple players; for other versions,  see~\citet{zhang2019policy,daskalakis2020independent,zhang2021gradient,leonardos2021global}. 

\begin{lemma}[Performance difference]
\label{lem: performance difference}
    For the $i$th player, if we fix the policy $\pi_{-i}$ and any state distribution $\mu$, then for any two policies $\hat\pi_i$ and $\bar\pi_i$, 
    \begin{align*}
    & V_i^{\hat\pi_i,\, \pi_{-i}} (\mu) - V_i^{\bar\pi_i,\, \pi_{-i}} (\mu)
     \\
    & = \; \frac{1}{1-\gamma} \sum_{s, \, a_i} d_\mu^{\hat \pi_i, \,\pi_{-i}}(s) \cdot (\hat\pi_i - \bar\pi_i)(a_i\,\vert\,s)  
         \bar Q_i^{\bar\pi_i,\, \pi_{-i}}(s,a_i)
    \end{align*}
    where 
    $\bar Q_i^{\bar \pi_i, \pi_{-i}} \!( s, a_i ) \! = \!
    \sum_{a_{-i} } \!\!\!\!
    \pi_{-i}{(a_{-i}\vert s)} Q_i^{\bar \pi_i, \pi_{-i}} \!(s, a_i, a_{-i}). 
   $
\end{lemma}

It is common to use the distribution mismatch coefficient to measure the exploration difficulty in policy optimization~\cite{agarwal2021theory}. We next define a distribution mismatch coefficient for MPGs~\cite{leonardos2021global} in \pref{def: distribution mismatch coeff}, and its minimax variant in \pref{def: minimax coeff}.

\begin{definition}[Distribution mismatch coefficient]\label{def: distribution mismatch coeff}
	For any distribution $\mu \in \Delta(\calS)$ and policy $\pi \in \Pi$, the distribution mismatch coefficient $\kappa_\mu$ is the maximum distribution mismatch of $\pi$ relative to $\mu$, $
	\kappa_\mu 
	\DefinedAs 
	\sup_{\pi \, \in\, \Pi } 
	\,
	\left\Vert d_{\mu}^\pi /\mu \right\Vert_\infty $,  
	where the division ${ d_{\mu}^\pi }/{\mu}$ is evaluated in a componentwise manner.
\end{definition}

\begin{definition}[Minimax distribution mismatch coefficient]\label{def: minimax coeff} For any distribution $\mu \in \Delta(\calS)$, the minimax distribution mismatch coefficient $\tilde{\kappa}_\mu$ is the minimax value of the distribution mismatch of $\pi$ relative to $\nu$, $
	\tilde{\kappa}_\mu
	\DefinedAs 
	\inf_{\nu\,\in\, \Delta(\calS)}\sup_{\pi \, \in\, \Pi } 
	\,
	\left\Vert d_{\mu}^\pi /\nu \right\Vert_\infty $, where the division $d_\mu^\pi/\nu$ is evaluated in a componentwise manner. 
\end{definition}

\noindent\textbf{Other notation}.
We denote by $\|\cdot\|$ the $\ell_2$-norm of a vector or the spectral norm of a matrix. 
The inner product of a function $f$: $\calS\times\calA\to\mathbb{R}$ with $p \in\Delta(\calA)$ at fixed $s\in \calS$ is given by $\langle f(s,\cdot), p(\cdot)\rangle_\calA \DefinedAs \sum_{a\,\in\,\calA} f(s,a)p(a)$. The $\ell_2$-norm projection operator onto a convex set $\Omega$ is defined as $\calP_{\Omega}(x)\DefinedAs\argmin_{x'\in\Omega}\|x'-x\|$. 
For functions $f$ and $g$, we write $f(n) = O(g(n))$ if there exists $N<\infty$ and $C<\infty$ such that $f(n) \leq C g(n)$ for $n\geq N$, and write $f(n) = \tilde{O}(g(n))$ if $\log g(n)$ appears in $O(\cdot)$. 
We use ``$\lesssim$'' and ``$\gtrsim$'' to denote ``$\leq$'' and ``$\geq$'' up to a constant.

\section{Independent Learning Setting}

\begin{algorithm*}
	\caption{ Independent policy gradient ascent }
	\label{alg: PMA}
	\begin{algorithmic}[1]
		\STATE
		\textbf{Parameters:} $\eta>0$. \\
		\textbf{Initialization}: Let $\pi_{i}^{(1)} (a_i\,\vert\,s)={1}/{A}$ for $s\in\calS$, $a_i\in\calA_i$ and $i = 1,\ldots,N$. 
		\FOR{step $t=1,\ldots,T$} 
		\FOR{player $i=1,\ldots,N$ (in parallel)} 
		\STATE Define player $i$'s policy on $s\in\calS$, 
		\begin{equation}\label{eq: policy gradient ascent step}
		\pi_i^{(t+1)} (\,\cdot\,\vert\,s) 
		\; \DefinedAs \;
		\argmax_{\pi_i(\cdot\,|\,s) \,\in\,\Delta(\calA_i)}
		\left\{  
		\big\langle 
		\pi_i(\cdot\,|\,s),\, \bar{Q}_i^{(t)}(s, \cdot)
		\big\rangle_{\calA_i}
		\, - \, 
		\frac{1}{2\eta}
		\big\|
		\pi_i(\cdot\,|\,s)-\pi_i^{(t)}(\cdot\,|\,s)
		\big\|^2 
		\right\}
		\end{equation}
		where $\bar{Q}_i^{(t)}(s,\, a_i)$ is a shorthand for $\bar{Q}_i^{\pi_i^{(t)},\,\pi_{-i}^{(t)}}(s, a_i)$ (defined in \pref{sec: problem setup}). 
		\ENDFOR
		\ENDFOR
	\end{algorithmic}
\end{algorithm*}

We examine an independent learning setting~\cite{zhang2021multi,daskalakis2020independent,ozdaglar2021independent} for Markov potential games in which all players repeatedly execute their own policy and update rules individually. 
At each time $t$, all players propose their own polices $\pi_i^{(t)}$: $\calS \to \Delta(\calA_i)$ with the player index $i=1,\ldots,N$, while a game oracle can either evaluate each player's policy or generate a set of sample trajectories for each player. In repeating such protocol for $T$ times, each player behaves myopically in optimizing its own policy.

To evaluate the learning performance, we introduce a notion of regret,
\begin{equation*}
	\text{Nash-Regret}(T)
	\! \DefinedAs \!
	\frac{1}{T}\sum_{t \,=\, 1}^T \! \max_i \! \left(\!\!
	\max_{\pi_i'} V_i^{\pi_i', \, \pi_{-i}^{(t)}}\!(\rho)  \!-\! V_i^{\pi^{(t)}}\!(\rho)
	\!\!\right)  
\end{equation*}
which averages the worst player's local gaps in $T$ iterations: $\max_{\pi_i'} V_i^{\pi_i', \, \pi_{-i}^{(t)}}(\rho) -  V_i^{\pi^{(t)}}(\rho)$ for $t = 1,\ldots,T$, where $\max_{\pi_i'} V_i^{\pi_i', \, \pi_{-i}^{(t)}}(\rho)$ is the $i$th player best response given $\pi_{-i}^{(t)}$. In $\text{Nash-Regret}(T)$, we compare the learnt joint policy $\pi^{(t)}$ with the best policy that the $i$th player can take by fixing $\pi_{-i}^{(t)}$. We notice that Nash-Regret is closely related to the notion of dynamic regret~\cite{zinkevich2003online} in which the regret comparator changes over time. This is a suitable notion because the environment is non-stationary from the perspective of an independent learner~\cite{matignon2012independent,zhang2021multi}.

To obtain an $\epsilon$-Nash equilibrium $\pi^{(t^\star)}$ with a tolerance $\epsilon>0$, our goal is to show the following average performance,
\[
\text{Nash-Regret}(T) 
\;=\; 
{\epsilon }.
\]
The existence of such $t^\star$ is straightforward, 
\[
t^\star
\; \DefinedAs \;
\argmin_{
	1 \; \leq \; 
	t 
	\; \leq \; T}
\;
\max_i \, 
\left(\,
\max_{\pi_i'} V_i^{\pi_i', \, \pi_{-i}^{(t)}}(\rho) -  V_i^{\pi^{(t)}}(\rho)
\,\right).   
\]
Since each summand above is non-negative, $V_i^{\pi^{(t^\star)}}(\rho) \geq V_i^{\pi_i', \, \pi_{-i}^{(t^\star)}}(\rho) - \epsilon$ for any $\pi_i'$ and $i=1,\ldots,N$, which implies that $\pi^{(t^\star)}$ is an $\epsilon$-Nash equilibrium.

For an independent learning setting without uncertainty in gradient evaluation, we introduce a policy gradient method for Markov potential/cooperative games in~\pref{sec: gradient plays}. In~\pref{sec: function approximation}, we utilize a sample-based approach with function approximation to address the scenario in which true gradient is not available and, in~\pref{sec: BBW}, we provide the game-agnostic convergence analysis.

\section{Independent Policy Gradient  Methods}
\label{sec: gradient plays}

In this section, we assume that we have access to exact gradient and examine a gradient-based method for learning a Nash equilibrium in Markov potential/cooperative games.

\subsection{Policy gradient  for Markov potential games}
\label{sec: Independent policy gradient ascent}

A natural independent learning scheme for MPGs is to let every player independently perform policy gradient ascent~\cite{leonardos2021global,zhang2021gradient}. In this approach, the $i$th player updates its policy according the gradient of the value function with respect to the policy parameters,
\begin{align}\label{eq: policy gradient}
    &\pi_i^{(t+1)}(\cdot\,|\,s)\leftarrow \calP_{\Delta(\calA_i)} \left(\pi_i^{(t)}(\cdot\,|\,s) + \eta \frac{\partial V_i^{\pi} (\rho )}{ \partial \pi_i(a_i\,\vert\,s)}\bigg\vert_{\pi=\pi^{(t)}}\right)  \nonumber \\
    & \frac{\partial V_i^{\pi} (\rho )}{ \partial \pi_i(a_i\,\vert\,s)}
    \; = \;
    \frac{1}{1-\gamma} d_\rho^{\pi} (s) \bar Q_i^{\pi}(s,a_i)
\end{align}
where the calculation for the gradient in~\pref{eq: policy gradient} can be found in \citet{agarwal2021theory, leonardos2021global,zhang2021gradient}. 

Update rule~\pref{eq: policy gradient} may suffer from a slow learning rate for some states. Since the gradient with respect to $\pi_i(a_i\,\vert\,s)$ scales with $d^{\pi}_\rho(s)$ -- which may be small if the current policy $\pi$ has small visitation frequency to $s$ -- the corresponding states may experience slow learning progress. To address this issue, we propose the following update rule (equivalent to~\pref{eq: policy gradient ascent step} in~\pref{alg: PMA}): 
\begin{align}\label{eq: our policy gradient}
    &
    \!\!\!\!
    \pi_i^{(t+1)}(\cdot\,|\,s)
    \; \leftarrow \;
    \calP_{\Delta(\calA_i)} 
    \left(\,
    \pi_i^{(t)}(\cdot\,|\,s) 
    \, + \, 
    \eta \bar{Q}_i^{\pi^{(t)}}(s,\cdot)
    \,\right)
\end{align}
which essentially removes the ${d^{\pi}_\rho(s)}/(1-\gamma)$ factor in standard policy gradient~\pref{eq: policy gradient} and alleviates the slow-learning issue. 
Interestingly, update rule~\pref{eq: our policy gradient} for the single-player MDP has also been studied in \citet{xiao2022convergence}, concurrently. However, since the optimal value is not unique, the analysis of~\citet{xiao2022convergence} does not apply to our multi-player case for which many Nash policies exist and the set that contains them is non-convex~\cite{leonardos2021global}. We also note that regularized variants of~\pref{eq: our policy gradient} for the single-player MDP appeared in \citet{lan2021policy,zhan2021policy}.

Furthermore, in contrast to~\pref{eq: policy gradient}, our update rule~\pref{eq: our policy gradient} is invariant to the initial state distribution $\rho$. This allows us to establish performance guarantees simultaneously for all $\rho$ in a similar way as typically done for natural policy gradient (NPG) and other policy mirror descent algorithms for single-player MDPs~\cite{agarwal2021theory,lan2021policy, zhan2021policy}.

\pref{thm: convergence PMA} establishes performance guarantees for \pref{alg: PMA}; see~\pref{ap: convergence PMA} for proof.
\begin{theorem}[Nash-Regret bound for Markov potential games]
\label{thm: convergence PMA}
	For MPG~\pref{eq: Markov potential game} with an initial state distribution $\rho$, if all players independently perform the policy update in~\pref{alg: PMA} then, for two different choices of stepsize $\eta$, \mbox{we have}
	\begin{align*}
	\scalebox{1}{ 
	$\displaystyle
    \text{\normalfont Nash-Regret}(T)
	\; \lesssim \; \displaystyle 
	\calR(\eta)
	$
	}
	\end{align*}
	\[
	\calR(\eta) =
	\begin{cases}
		\displaystyle   \frac{ \sqrt{\coeff_\rho \,A N}\,(C_\Phi)^{\frac{1}{4}}}{(1-\gamma)^{\frac{9}{4}}\, T^{\frac{1}{4}}}, 
		\;\;\;\; \;\;\;\;
		\eta = \frac{(1-\gamma)^{\frac{5}{2}}\sqrt{C_\Phi}}{NA\sqrt{T}}
		\\[0.4cm]
		\displaystyle  \frac{ \min(\kappa_\rho, S)^2\sqrt{ A N C_\Phi}}{(1-\gamma)^{3} \sqrt{T}}, \,
		\eta = \frac{(1-\gamma)^4}{8 \min(\kappa_\rho, S)^3 NA}.
	\end{cases}
	\]
\end{theorem}

Depending on the stepsize $\eta$, \pref{thm: convergence PMA} provides two rates for the average Nash regret: $T^{-{1}/{4}}$ and $T^{-{1}/{2}}$. The technicalities behind these choices will be explained later and, to obtain an $\epsilon$-Nash equilibrium, our two bounds suggest respective iteration complexities,
\[
\dfrac{\coeff_\rho^2\, A^2 N^2\, C_\Phi}{(1-\gamma)^9\, \epsilon^4}
\;
\text{ and }
\;
\dfrac{\min(\kappa_\rho, S)^4 \, AN\, C_\Phi}{(1-\gamma)^6\, \epsilon^2}.
\]

Compared with the iteration complexity guarantees in \citet{leonardos2021global,zhang2021gradient}, our bounds in \pref{thm: convergence PMA} improve the dependence on the distribution mismatch coefficient $\kappa_\rho$ and the state space size $S$. Since our minimax distribution mismatch coefficient $\tilde{\kappa}_\rho $ satisfies
\[
\tilde{\kappa}_\rho 
\; \leq \;
\min(\kappa_\rho,S)
\; \leq \;
\kappa_\rho, 
\]
our $\tilde{\kappa}_\rho$-dependence or $\min(\kappa_\rho,S)$-dependence are less restrictive than the explicit $S$-dependence in \citet{leonardos2021global,zhang2021gradient}. Importantly, this permits our bounds to work for systems with large number of states, and makes~\pref{alg: PMA} suitable for sample-based scenario with function approximation (see~\pref{sec: function approximation}). With polynomial dependence on the number of players $N$ instead of exponential,~\pref{alg: PMA} overcomes the \emph{curse of multiagents}~\cite{jin2021v,song2021can}. In terms of problem parameters $(\gamma,A,N,C_\Phi)$, our iteration complexity either improves or becomes slightly worse.
\begin{remark}[Infinite state space]
When the state space is infinite, explicit $S$-dependence disappears in our iteration complexities. Implicit $S$-dependence only exists in the distribution mismatch coefficient $\kappa_\rho$ or $\tilde{\kappa}_\rho$. However, it is easy to bound $\kappa_\rho$ by devising an initial state distribution without introducing constraints on the MDP dynamics. For instance, in MPGs with agent-independent transitions (in which every state is a potential game and transitions do not depend on actions~\citep{leonardos2021global}), if we select $\rho$ to be the stationary state distribution $d_\rho^\pi$ then $\kappa_\rho = 1$  regardless of the state-space size $S$. 
\end{remark}

\begin{remark}[Our key techniques]
A key step of the analysis is to quantify the policy improvement regarding the potential function $\Phi$ in each iteration. Similar to the standard descent lemma in optimization \cite{Arora08}, applying the projected policy gradient algorithm to a smooth $\Phi$ yields the following ascent property (cf.\ Eq.~(9) in~\citet{leonardos2021global} and Lemmas~11 and~12 in~\citet{zhang2021gradient}),  
\begin{align*}
    \scalebox{0.96}{$\displaystyle \Phi^{\pi^{(t+1)}}(\mu) \!-\!
    \Phi^{\pi^{(t)}}(\mu) \,\gtrsim\, \frac{1}{\beta} \sum_{i=1}^N \sum_{s}\norm{\pi_i^{(t+1)}(\cdot|s) \!-\! \pi_i^{(t)}(\cdot|s)}^2$ }
\end{align*}
where $\beta>0$ is related to the smoothness constant (or the second-order derivative) of the potential function. However, since the search direction in our policy update is not the standard search direction utilized in policy gradient, this ascent analysis does not apply to our algorithm.

To obtain such improvement bound, it is crucial to analyze the \emph{joint policy improvement}. Let us consider two players $i$ and $j$: player $i$ changes its policy from $\pi_i$ to $\pi_i'$ to maximize its own reward based on the \emph{current} policy profile $(\pi_i,\pi_j)$ and player $j$ changes its policy from $\pi_j$ to $\pi_j'$ in its own interest. What is the overall progress after they independently change their policies from $(\pi_i, \pi_j)$ to $(\pi_i', \pi_j')$? One method of capturing the joint policy improvement exploits the smoothness of the potential function, which is useful in the standard policy gradient ascent method \cite{leonardos2021global, zhang2021gradient}. In our analysis, we connect the joint policy improvement with the individual policy improvement via the performance difference lemma. In particular, as shown in~\pref{lem: MPG policy improvement},~\pref{lem: decomposition lemma} and \pref{lem: second-order PDL} provide an effective means for analyzing the joint policy improvement. The proposed approach could be of independent interests for analyzing other Markov games. 

In \pref{lem: MPG policy improvement}, we obtain two different joint policy improvement bounds by dealing with the cross terms in two different ways (see the proofs for details). Hence, we establish two different Nash-Regret bounds in \pref{thm: convergence PMA}: one has better dependence on $T$ while the other has better dependence on $\kappa_\rho$. Even though, it is an open issue how to achieve the best of the two, we next show that this is indeed possible for a special case: Markov cooperative games. 

\end{remark}

\subsection{Faster rates for Markov cooperative games}

When all players use the same reward function, i.e., $r = r_i$ for all $i=1,\ldots,N$, MPG~\pref{eq: Markov potential game} reduces to a Markov cooperative game. In this case, $V_i^\pi=V^\pi $ and $Q_i^\pi=Q^\pi$ for all $i = 1,\ldots,N$ and \pref{alg: PMA} works immediately. Thus, we continue to use $\text{Nash-Regret}_i(T)$ that is defined through $V_i^{\pi_i', \, \pi_{-i}^{(t)}}=V^{\pi_i', \, \pi_{-i}^{(t)}}$ and $V_i^{\pi^{(t)}}= V^{\pi^{(t)}}$. 

\pref{thm: convergence PMA cooperative} provides a Nash-Regret bound for Markov cooperative games; see~\pref{ap: convergence PMA cooperative} for proof.

\begin{theorem}[Nash-Regret bound for Markov cooperative games]
\label{thm: convergence PMA cooperative}
	For MPG~\pref{eq: Markov potential game} with identical rewards and an initial state distribution $\rho$, if all players independently perform the policy update in~\pref{alg: PMA} with stepsize $\eta= (1-\gamma)/(2NA)$ then,
	\begin{align*}
		\scalebox{1.0}{$ \displaystyle
			\text{\normalfont Nash-Regret}(T)
			\; \lesssim \;
			\displaystyle 
			\frac{ \sqrt{\coeff_\rho\, AN} } 
			{(1-\gamma)^{2}\,\sqrt{T}}
			$. }
	\end{align*}
\end{theorem}

For Markov cooperative games,~\pref{thm: convergence PMA cooperative} achieves the best of the two bounds in \pref{thm: convergence PMA} and an $\epsilon$-Nash equilibrium is achieved with the following iteration complexity,
\[
\dfrac{\coeff_\rho\, AN}{(1-\gamma)^4\, \epsilon^2}.
\]
This iteration complexity improves the ones provided in~\citet{leonardos2021global,zhang2021gradient} in several aspects. In particular, we have introduced the minimax distribution mismatch coefficient $\tilde{\kappa}_\rho$, which is upper bounded by $\kappa_\rho$. Since we take this $\tilde{\kappa}_\rho$, our bound improves the $\kappa_\rho$-dependence in~\citet{leonardos2021global,zhang2021gradient} from $\kappa_\rho^2$ to $\kappa_\rho$. We note that if we view the Markov cooperative game as an MPG, then the value function $V^\pi$ serves as a potential function $\Phi$ which is bounded between $0$ and $1/(1-\gamma)$. Thus, our $(1-\gamma)$-dependence matches the one in \citet{zhang2021gradient} and improves the one in \citet{leonardos2021global} by $(1-\gamma)^2$.

\section{Independent Policy Gradient with Function Approximation}
\label{sec: function approximation}

We next remove the exact gradient requirement and apply~\pref{alg: PMA} to the linear function approximation setting. In what follows, we assume that the averaged action value function is linear in a given feature map.

\begin{assumption}[Linear averaged $Q$]\label{as: linear Q}
	In MPG~\pref{eq: Markov potential game}, for each player $i$, there is a feature map $\phi_i$ : $\calS\times\calA_i \to \mathbb{R}^d$, such that for any $(s,a_i)\in\calS\times\calA_i$ and any policy $\pi \in \Pi$,
	\[
	\bar Q_i^\pi (s,a_i) 
	\; = \; 
	\langle\, \phi_i(s,a_i), \, w_i^\pi \,\rangle
	,\; \text{ for some } w_i^\pi \, \in \, \mathbb{R}^d.
	\]
	Moreover, $\norm{ \phi_i }\leq 1$ for all $s,a_i$, and $\norm{w_i^\pi} \leq W$ for all $\pi$.  
\end{assumption}

Without loss of generality, we can assume $W\leq  {\sqrt{d}}/(1-\gamma)$; see Lemma~8 in~\citet{wei2021learning}. \pref{as: linear Q} is a multi-agent generalization of the standard linear $Q$ assumption~\cite{abbasi2019politex} for single-player MDPs. It is different from the multi-agent linear MDP assumption~\cite{xie2020learning,dubey2021provably} in which both transition and reward functions are linear in given feature maps. In contrast, \pref{as: linear Q} qualifies each player to estimate its averaged action value function without observing other players' actions. A special case of \pref{as: linear Q} is the tabular case in which the sizes of state/action spaces are finite, and where we can select $\phi_i$ to be an indicator function. Since the feature map $\phi_i$ is locally-defined coordination between players is avoided~\cite{zhao2021provably}.

\begin{remark}[Function approximation]
	Since RL with function approximation is statistically hard in general, e.g.,  see~\citet{weisz2021exponential,wang2021exponential} for hardness results, assuming regularity of underlying MDPs is necessary for the application of function approximation to multi-agent RL in which either the value function~\cite{xie2020learning,dubey2021provably,jin2021power,huang2022towards} or the policy~\cite{zhao2021provably} is approximated. Because of restrictive function approximation power, the main challenge is the entanglement of policy improvement (or optimization) and policy evaluation (or approximation) errors. In \pref{thm: sample PMA potential} and \pref{thm: sample PMA cooperative}, we show that optimization and approximation errors are decoupled under~\pref{as: linear Q} so that we can control them, separately. Our analysis can be generalized to some neural networks, e.g., overparametrized neural networks~\cite{liu2019neural}, a rich function class that allows splitting optimization and approximation errors, which we leave for future work.
\end{remark}

We formally present our algorithm in \pref{alg: PMA fa} (see it in \pref{ap: algorithms}). At each step $t$, there are two phases. In Phase 1, the players begin with the initial state $\bar s^{(0)}\sim \rho$ and simultaneously execute their current policies $\{\pi_i^{(t)} \}_{i\,=\,1}^N$ to interact with the environment for $K$ rounds. In each round $k$, we terminate the interaction at step $H = \max_{i} (h_i+h_i')$, where $h_i$ and $h_i'$ are sampled from a geometric distribution $\textsc{Geometric}(1-\gamma)$, independently; the state at $h_i$ naturally follows $\bar s^{(h_i)}\sim d_{\rho}^{\pi^{(t)}}$. By collecting rewards from step $h_i$ to $h_i+h_i'-1$, as shown in \pref{eq: samples}, we can justify $\E[R^{(k)}_i] = \bar{Q}^{(t)}_i (\bar s^{(h_i)}, \bar a_i^{(h_i)})$ where $\bar{Q}_i^{(t)}(\cdot,\cdot)\DefinedAs\bar{Q}_i^{\pi^{(t)}}(\cdot,\cdot)$ and $\bar a_i^{(h_i)} \sim \pi_i^{(t)}(\cdot~|~\bar s^{(h_i)})$, in \pref{ap: unbiased estimate}. In the end of round $k$, we collect a sample tuple: $(s_i^{(k)}, a_i^{(k)}, R_i^{(k)})$ in \pref{eq:  samples} for each player~$i$. 

After each player collects $K$ samples, in Phase 2, they use these samples to estimate $\bar{Q}_i^{(t)}(\cdot,\cdot)$, which is required for policy updates. By \pref{as: linear Q}, 
\begin{align*}
	\bar{Q}_i^{(t)}(s,a_i)
	\; = \;
	\langle \phi_i(s,a_i), w_i^{(t)} \rangle,\; \forall (s,a_i) \,\in\,\calS\times\calA_i
\end{align*}
where $w_i^{(t)}$ represents $w_i^{\pi^{(t)}}$. Our goal is to obtain a solution $\hat{w}_i^{(t)}\approx w_i^{(t)}$ using samples, and estimate $\bar{Q}_i^{(t)}(s,a_i)$ via 
\begin{equation}\label{eq: Q estimate}
	\hat{Q}_i^{(t)}(s, a_i) 
	\; \DefinedAs \;
	\langle \phi_i(s,a_i), \hat w_i^{(t)} \rangle
	,\;
	\forall (s,a_i) \,\in\,\calS\times\calA_i.
\end{equation}

To obtain $\hat{w}_i^{(t)}$, the standard approach is to solve linear regression~\pref{eq: linear regression} since $\E[R^{(k)}_i]=\bar{Q}^{(t)}_i (s_i^{(k)}, a_i^{(k)})=\langle\phi_i(s_i^{(k)}, a_i^{(k)}), w_i^{(t)}\rangle$.  
We measure the estimation quality of $\hat{w}_i^{(t)}$ via the expected regression loss,
\begin{align*}
	& L_i^{(t)} ( w_i )
	= 
	\mathbb{E}_{(s,a_i)\,\sim\,\nu_i^{(t)}} 
	\left[ 
	\left(
	\bar Q_i^{(t)}(s,a_i) - \langle \phi_i(s,a_i), w_i\rangle
	\right)^2
	\right]
\end{align*}
where $\nu_i^{(t)}(s,a_i) \DefinedAs d_\rho^{(t)}(s)\circ \pi^{(t)}_i(a_i\,|\,s)$ and $L_i^{(t)}(w_i^{(t)}) = 0$ by \pref{as: linear Q}. We make the following assumption for the expected regression loss of $\hat{w}_i^{(t)}$.

\begin{assumption}[Bounded statistical error]\label{as: bounded error}
	Fix a state distribution $\rho$. For any sequence of iterates $\hat w_i^{(1)},\ldots, \hat w_i^{(T)}$ for $i = 1,\ldots,N$ that are generated by \pref{alg: PMA fa}, there exists an $\epsilon_{\normalfont\text{stat}}<\infty$ such that 
	\[
	\mathbb{E} 
	\big[\,
	L_i^{(t)} (\,\hat w_i^{(t)}\,)
	\,\big] 
	\;\leq\; 
	\epsilon_{\normalfont\text{stat}}
	\] 
	for all $i$ and $t$, where the expectation is on randomness in generating $\hat w_i^{(t)}$.
\end{assumption}

The bound for $\epsilon_{\text{stat}}$ can be established using standard linear regression analysis \cite{audibert09} and it is given by $\epsilon_{\text{stat}}=O\big(\frac{dW^2}{K(1-\gamma)^2}\big)$. This bound can be achieved by applying the stochastic projected gradient descent method~\cite{hsu2012random,cohen2017projected} to the regression problem.

After obtaining $\hat{Q}^{(t)}_i(\cdot,\cdot)$, we update the polices in \pref{eq: sample-based policy gradient ascent step} which is different from the update in \pref{alg: PMA} in two aspects: (i) the gradient direction $\hat{Q}^{(t)}_i(\cdot,\cdot)$ is the estimated version of $\bar{Q}^{(t)}_i(\cdot,\cdot)$; and (ii) the Euclidean projection set becomes $\Delta_\xi (\calA_i) \DefinedAs \{\, (1-\xi)\,\pi_i(\cdot\,\vert\,s) + \xi\,\text{Unif}_{\calA_i}, \forall \pi_i(\cdot\,\vert\,s)\,\}$ that introduces $\xi$-greedy policies for exploration~\citep{leonardos2021global,zhang2021gradient}, where $\xi \in (0,1)$. 

\pref{thm: sample PMA potential} establishes performance guarantees for \pref{alg: PMA fa}; see \pref{ap: convergence PMA sample-based} for proof.
 
\begin{theorem}[Nash-Regret bound for Markov potential games with function approximation]
\label{thm: sample PMA potential}
    Let \pref{as: linear Q} hold for MPG~\pref{eq: Markov potential game} with an initial state distribution $\rho$. If all players independently run \pref{alg: PMA fa} with 
	$\xi=\min\Big(\left(\frac{\kappa_\rho^2NA\epsilon_{\normalfont\text{stat}}}{(1-\gamma)^2W^2}\right)^{\frac{1}{3}}, \frac{1}{2}\Big)$ and \pref{as: bounded error} holds, then
	\begin{align*}
	& \mathbb{E}
	\left[\,
	\text{\normalfont Nash-Regret}(T)  
	\,\right] 
	 \; \lesssim \; \calR(\eta) 
	\, + \,
	\left(\frac{\kappa_\rho^2\, WAN\epsilon_{\normalfont\text{stat}}}{(1-\gamma)^5}\right)^{\frac{1}{3}}
	\end{align*}
	\[
	\calR(\eta) \; = \; 
	\begin{cases}
		\displaystyle
		\frac{ \sqrt{\kappa_\rho\, W N }\,(A C_\Phi)^{\frac{1}{4}} } 
		{(1-\gamma)^{\frac{7}{4}}\, T^{\frac{1}{4}}}, \;\;\;\; \eta\,=\,\frac{(1-\gamma)^{\frac{3}{2}}\sqrt{C_\Phi}}{WN\sqrt{AT}} 
		\\[0.4cm]
		\displaystyle
		\frac{ \kappa_\rho^2\, \sqrt{ A N \, C_\Phi } } 
		{(1-\gamma)^3\, \sqrt T}, \;\;\;\; \;\;\;\; \;\;\;\; \; \eta\,=\, \frac{(1-\gamma)^4}{16\, \kappa_\rho^3 \, NA}.
	\end{cases}
	\]
\end{theorem}

\pref{thm: sample PMA potential} shows the additive effect of the function approximation error $\epsilon_{\normalfont\text{stat}}$ on the Nash regret of \pref{alg: PMA fa}. When $\epsilon_{\normalfont\text{stat}}=0$, \pref{thm: sample PMA potential} matches the rates in \pref{thm: convergence PMA} in the exact gradient case. As in~\pref{alg: PMA}, even though update rule~\pref{eq: sample-based policy gradient ascent step} iterates over all $s\in\calS$, we do not need to assume a finite state space $\calS$. In fact, \pref{eq: sample-based policy gradient ascent step} only ``defines'' a function  $\pi_i^{(t)}(\cdot~|~s)$ instead of ``calculating'' it. This is commonly used in policy optimization with function approximation, e.g., \citet{cai2020provably, luo2021policy}. To execute this algorithm,  $\pi_i^{(t)}(\cdot~|~s)$ only needs to be evaluated if necessary, e.g., when the state $s$ is visited in Phase 1 of \pref{alg: PMA fa}.

When we apply stochastic projected gradient updates to \pref{eq:  linear regression}, \pref{alg: PMA fa} becomes a sample-based algorithm and existing stochastic projected gradient results directly apply. Depending on the stepsize choice, an $\epsilon$-Nash equilibrium is achieved with sample complexities (see~\pref{cor: sample complexity potential} in \pref{ap: sample complexity}),
$$TK \; = \; O\left(\frac{1}{\epsilon^{7}}\right) 
\;\text{ and }\;
O\left(\frac{1}{\epsilon^{5}}\right),
\;\text{ respectively.}
$$ 
Compared with the sample complexity guarantees for the tabular MPG case \cite{leonardos2021global,zhang2021gradient}, our sample complexity guarantees hold for MPGs with potentially infinitely large state spaces. When we specialize \pref{as: linear Q} to the tabular case, our second sample complexity improves the sample complexity in \citet{leonardos2021global,zhang2021gradient} from $O({1}/{\epsilon^6})$ to $O({1}/{\epsilon^5})$. 

As before, we get improved performance guarantees when we apply \pref{alg: PMA fa} to Markov cooperative games.
\begin{theorem}[Nash-Regret bound for Markov cooperative games with function approximation]
\label{thm: sample PMA cooperative}
    Let \pref{as: linear Q} hold for MPG~\pref{eq: Markov potential game} with identical rewards and an initial state distribution $\rho>0$. If all players independently perform the policy update in~\pref{alg: PMA fa} with stepsize $\eta=(1-\gamma)/(2NA)$ and exploration rate 
	$\xi=\min\Big(\left(\frac{\kappa_\rho^2 NA\epsilon_{\normalfont\text{stat}}}{(1-\gamma)^2W^2}\right)^{\frac{1}{3}}, \frac{1}{2}\Big)$, with  \pref{as: bounded error}, 
    \begin{align*}
	&\mathbb{E}
	\left[\,
	\text{\normalfont Nash-Regret}(T)  
	\,\right]   
	 \; \lesssim  \;
	\calR(\eta)
	\, + \,
	\left(\frac{\kappa_\rho^2\, WAN\epsilon_{\normalfont\text{stat}}}{(1-\gamma)^5}\right)^{\frac{1}{3}}
	\end{align*} 
	where 
	$
	\calR(\eta)
	 = 
	\frac{ \sqrt{\kappa_\rho A N} } 
	{(1-\gamma)^2 \sqrt T}.
	$
\end{theorem}

We prove~\pref{thm: sample PMA cooperative} in \pref{ap: convergence PMA cooperative sample-based} and show sample complexity $TK = O({1}/{\epsilon^5})$ in~\pref{cor: sample complexity cooperative} of \pref{ap: sample complexity}.

\section{Game-Agnostic Convergence}
\label{sec: BBW}

In \pref{sec: gradient plays} and \pref{sec: function approximation}, we have shown that our independent policy gradient method converges (in best-iterate sense) to a Nash equilibrium of MPGs. For the same algorithm in two-player case, however, \cite{bailey2019fast} showed that players' policies can diverge for zero-sum matrix games (a single-state case of zero-sum Markov games).  
A natural question arises: 
\begin{center}
	\emph{Does there exist a simple gradient-based algorithm that provably converges to a Nash equilibrium in both potential/cooperative and zero-sum games?} 
\end{center}

Unfortunately, classical MWU and optimistic MWU updates do not converge to a Nash equilibrium in zero-sum and coordination games simultaneously~\cite{cheung2020chaosvideo}. Recently, this question was partially answered by~\citet{leonardos2021exploration,leonardos2022exploration} in which the authors established last-iterate convergence of $Q$-learning dynamics to a quantal response equilibrium for both zero-sum and potential/cooperative matrix games. In this work, we provide an affirmative answer to this question for general Markov games that cover matrix games. Specifically, we next show that optimistic gradient descent/ascent with a smoothed critic (see \pref{alg: OPMA smooth Q} in \pref{ap: algorithms}) -- an algorithm that converges to a Nash equilibrium in  two-player zero-sum Markov games~\cite{Wei2021LastiterateCO}  -- also converges to a Nash equilibrium in Markov cooperative games.

We now setup notation for tabular two-player Markov cooperative games with $N=2$, $r = r_1 = r_2$, $A = |\calA_1| = |\calA_2|$, and $S = |\calS|$. For convenience, we use $x_s \in \mathbb{R}^A$ and $y_s \in \mathbb{R}^A$ to denote policies $\pi_1 (\cdot\,\vert\,s)$ and $ \pi_2(\cdot\,\vert\,s)$ taken at state $s \in \calS$, and $Q_s^\pi \in \mathbb{R}^{A\times A}$ to denote $Q^{\pi}(s,a_1,a_2)$ with $a_1\in\calA_1$ and $a_2\in \calA_2$. We describe our policy update \pref{eq: optimistic gradient ascents} in \pref{alg: OPMA smooth Q}: the next iterate $(x_s^{(t+1)},y_s^{(t+1)})$ is obtained from two steps of policy gradient ascent with an intermediate iterate $(\bar x_s^{(t+1)},\bar y_s^{(t+1)})$. Motivated by \citet{Wei2021LastiterateCO}, we introduce a critic $\calQ_s^{(t)}$ to learn the value function at each state $s$ using the learning rate $\alpha^{(t)}$. When the critic is ideal, i.e., $\calQ_s^{(t)}  = Q_s^{(t)}$, where $Q_s^{(t)}$ is a matrix form of $Q^{(t)}(s,a_1,a_2)$ for $a_1\in\calA_1$ and $a_2\in\calA_2$, we can view \pref{alg: OPMA smooth Q} as a two-player case of \pref{alg: PMA}.

In \pref{thm: bbw asymptotic}, we establish asymptotic last-iterate convergence of \pref{alg: OPMA smooth Q} in Markov cooperative games; see \pref{ap: BBW asymptotic convergence} for proof.

\begin{theorem}[Last-iterate convergence for two-player Markov cooperative games]\label{thm: bbw asymptotic}
	For MPG~\pref{eq: Markov potential game} with two players and identical rewards, if both players run \pref{alg: OPMA smooth Q} with $0<\eta< (1-\gamma)/(32\sqrt{A})$ and a non-increasing $\{\alpha^{(t)}\}_{t\,=\,1}^\infty$ that satisfies $0< \alpha^{(t)}<1/6$ and $\sum_{t\,=\,t'}^\infty\alpha^{(t)} = \infty$ for any $t'\geq 0$, then the policy pair $(x^{(t)}, y^{(t)})$ converges to a Nash equilibrium when $t\rightarrow \infty$.  
\end{theorem}

Last-iterate convergence in \pref{thm: bbw asymptotic} is measured by the local gaps     
$\max_{x'} (V^{x', y^{(t)}}(\rho) - V^{x^{(t)}, y^{(t)}}(\rho))$ and $\max_{y'} (V^{ x^{(t)},y'}(\rho) - V^{x^{(t)}, y^{(t)}}(\rho))$, i.e., a policy pair $(x^{(t)},y^{(t)})$ constitutes an approximate Nash policy for large $t$. The condition on algorithm parameters $\eta$ and $\alpha^{(t)}$ in \pref{thm: bbw asymptotic} is mild in sense that it is straightforward to take a pair of such parameters that ensures last-iterate convergence in zero-sum Markov games \cite{Wei2021LastiterateCO}. Hence, \pref{alg: OPMA smooth Q} enjoys last-iterate convergence in both two-player Markov cooperative and zero-sum competitive games. Compared with the result~\cite{fox2021independent}, our proof of \pref{thm: bbw asymptotic} utilizes gap convergence instead of point-wise policy convergence that is restricted to isolated fixed points of the algorithm dynamics. Moreover, our algorithm works for both cooperative and competitive Markov games. 

In the following \pref{thm: convergence OPMA cooperative}, we further strengthen our result of \pref{thm: bbw asymptotic} and show the sublinear Nash-Regret bounds for \pref{alg: OPMA smooth Q} in both two-player Markov cooperative and zero-sum competitive games; see \pref{ap: convergence OPMA cooperative} for proof.

\begin{theorem}[Nash-Regret bound for two-player Markov cooperative/competitive games]
	\label{thm: convergence OPMA cooperative}
	{\normalfont (i)} For MPG~\pref{eq: Markov potential game} with two players and identical rewards ($r_1 = r_2 = r$), if both players independently run \pref{alg: OPMA smooth Q} with $\alpha^{(t)} = \frac{1}{6 \sqrt[3]{t}}$ and $\eta=\frac{(1-\gamma)^2}{32\sqrt{SA}}$, then
	\[
		\begin{array}{rcl}
			&& \!\!\!\!  \!\!\!\!  \!\!\!\!  \!\!
			\displaystyle \frac{1}{T}\sum_{t\,=\,1}^T   \max_{x', y'} 
			\left( V^{x',y^{(t)}}(\rho) +V^{x^{(t)},y'}(\rho) -2 V^{x^{(t)},y^{(t)}}(\rho)\right)
			\\[0.2cm]
			&& \!\!\!\!  \!\!\!\!  \!\!\!\!  \!\!
			\lesssim \; \displaystyle \frac{(\,S^3A\,)^{\frac{1}{4}}}{(1-\gamma)^{\frac{7}{2}}\,T^{\frac{1}{6}}}
		\end{array}
	\]
	{\normalfont (ii)} For a two-player zero-sum Markov game ($r_1 = -r_2 = r$), if both players independently run \pref{alg: OPMA smooth Q} with the same choice of $\alpha^{(t)}$ and $\eta$, then
	\[
	\frac{1}{T}\sum_{t\,=\,1}^T   \max_{x', y'} 
	\left( V^{x',y^{(t)}}(\rho) - V^{x^{(t)},y'}(\rho) \right)
	\lesssim 
	\displaystyle
	\frac{(\,S^3A\,)^{\frac{1}{2}}}{(1-\gamma)^{\frac{15}{4}}\,T^{\frac{1}{6}}}.
	\]
\end{theorem}

For two-player Markov cooperative/competitive games, \pref{thm: convergence OPMA cooperative} establishes the same rate $T^{-1/6}$ for the Nash regret and the average duality gap, respectively. Alternatively, independent players in \pref{alg: OPMA smooth Q} can find an $\epsilon$-Nash equilibrium after $O(1/\epsilon^6)$ iterations, no matter which  types of games are being  played. To the best of our knowledge, \pref{thm: convergence OPMA cooperative} appears to be the first game-agnostic convergence for Markov cooperative/competitive games with finite-time performance guarantees. We leave the extension to more general Markov games for future work.

\section{Experimental Results }
\label{sec: experiments}

\begin{figure}[tbh]
	\begin{center}
		\begin{tabular}{cc}
			{\rotatebox{90}{ \;\;\;\; \;\;\;\; \;\;\;\; \;\;\;\; accuracy}} 
			\!\!\!\!\!\!
			& {\includegraphics[scale=0.47]{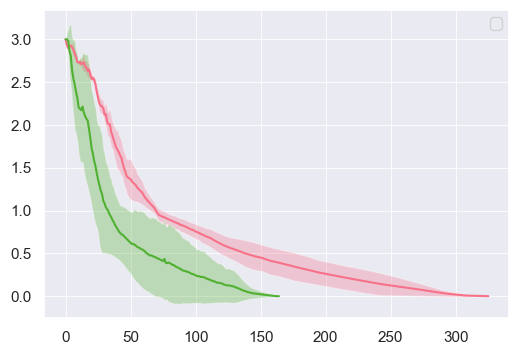}}
			\\[-0.1cm]
			{} & { \;\;\;\; iteration}
			\\[-0.2cm]
		\end{tabular}
	\end{center}
	\caption{ Learning curves for our independent policy gradient \mbox{(\textbf{\color{applegreen}---})} with stepsize $\eta=0.002$ and the projected stochastic gradient ascent (\textbf{\color{awesome}---}) with $\eta=0.0001$~\cite{leonardos2021global}.
		The accuracy measures the absolute distance of each iterate to the converged Nash policy, i.e., $\frac{1}{N}\sum_{i\,=\,1}^N \Vert{\pi_i^{(t)}-\pi_i^{\text{\normalfont Nash}}}\Vert_1$.
		Each solid line is the mean of trajectories over three random seeds and each shaded region displays the confidence interval.
	}
	\label{fig: policy distance unif}
\end{figure}

To demonstrate the merits and the effectiveness of our approach, we examine an MDP in which every state defines a congestion game. This example is borrowed from~\citet{bistritz2020cooperative} and it includes MPG as a special case.

\pref{fig: policy distance unif} shows that our independent policy gradient with a large stepsize ({\color{applegreen}green curve}) quickly converges to a Nash equilibrium. We note that stepsize $\eta \geq 0.001$ does not provide convergence of the projected stochastic gradient ascent \cite{leonardos2021global}. In contrast, our approach allows large stepsizes for a broad range of initial distributions; see~\pref{ap: experiments} for additional details.

\section{Concluding Remarks }
\label{sec: conclusion}

We have proposed new independent policy gradient algorithms for learning a Nash equilibrium of Markov potential games when the size of state space and/or the number of players are large. In the exact gradient case, we show that our algorithm finds an $\epsilon$-Nash equilibrium with $O(1/\epsilon^2)$ iteration complexity. Such iteration complexity does not explicitly depend on the state space size. In the sample-based case, our algorithm works in the function approximation setting, and we prove $O(1/\epsilon^5)$ sample complexity in a potentially infinitely large state space. This appears to be the first result for learning MPGs with function approximation. Moreover, we identify a class of independent policy gradient algorithms that enjoys last-iterate convergence and sublinear Nash regret for both zero-sum Markov games and Markov cooperative games (a special case of MPGs). This finding sheds light on an open question in the literature on the existence of such an algorithm. 

Future directions include extending techniques that offer faster rates for the single-agent policy gradient methods~\citep{lan2021policy,zhan2021policy,xiao2022convergence} to independent multi-agent learning and applying independent policy gradient for other large-scale Markov games.

\section*{Acknowledgements}
\label{sec: ack}
The work of D.\ Ding and M.\ R.\ Jovanovi\'{c} is supported in part by the National Science Foundation under awards ECCS-1708906 and 1809833.
The work of C.-Y. Wei is supported by NSF Award IIS-1943607. 
The work of K.\ Zhang is supported in part by the Simons-Berkeley Research Fellowship.
Part of this work was done while K.\ Zhang was visiting Simons Institute for the Theory of Computing.

%

\newpage
\bibliography{dd-bib}
\bibliographystyle{icml2022} %
\newpage


\newpage
\onecolumn
\appendix

~\\
\centerline{{\fontsize{13.5}{13.5}\selectfont \textbf{Supplementary Materials for }}}

\vspace{6pt}
\centerline{\fontsize{13.5}{13.5}\selectfont \textbf{
	``Independent Policy Gradient for Large-Scale Markov Potential Games:}}

\vspace{6pt}
\centerline{\fontsize{13.5}{13.5}\selectfont \textbf{
	Sharper Rates, Function Approximation, and Game-Agnostic Convergence''}}
\vspace{10pt}

\section{Algorithms in Section~\ref{sec: function approximation} and Section~\ref{sec: BBW}}\label{ap: algorithms}

\begin{algorithm*}[tbh]
	\caption{ Independent policy gradient with linear function approximation }
	\label{alg: PMA fa}
	\begin{algorithmic}[1]
		\STATE \textbf{Parameters:} $K$, $W$, and $\eta>0$.
		\STATE \textbf{Initialization}: Let $\pi_{i}^{(1)} (a_i\,\vert\,s)={1}/{A}$ for $s\in\calS$, $a_i\in\calA_i$ and $i = 1,\ldots,N$.
		\FOR{step $t=1,\ldots,T$} 
		\STATE \texttt{// Phase 1 (data collection)} 
		\FOR{round $k = 1,\ldots,K$}
		\STATE For each $i\in[N]$, sample $h_{i}\sim \textsc{Geometric}(1-\gamma)$ and $h_i'\sim \textsc{Geometric}(1-\gamma)$. 
		\STATE Draw an initial state $\bar{s}^{(0)}\sim \rho$. 
		\STATE Continuing from $\bar{s}^{(0)}$, let all players interact with each other using $\{\pi^{(t)}_i\}_{i\,=\,1}^N$ for $H=\max_i (h_i + h_i')$ steps,
		which generates a state-joint-action-reward trajectory $\bar{s}^{(0)},\, \bar{a}^{(0)}, \bar{r}^{(0)},\, \bar{s}^{(1)},\, \bar{a}^{(1)},\,\bar{r}^{(1)},\, \ldots, \bar{s}^{(H)},\, \bar{a}^{(H)},\,\bar{r}^{(H)}$. 
		\STATE Define for every player $i\in [N]$,
		\begin{align}\label{eq: samples}
			s_i^{(k)} \;=\; \bar{s}^{(h_i)}, \qquad a_i^{(k)} \;=\; \bar{a}_i^{(h_i)}, \qquad R_i^{(k)} \;=\; \sum_{h\,=\,h_i}^{h_i+h_i'-1} \bar{r}_i^{(h)}\,.
		\end{align}
		\ENDFOR 
		\STATE \texttt{// Phase 2 (policy update)} 
		\FOR{player $i=1,\ldots,N$ (in parallel)}

		\STATE Compute $\hat w_i^{(t)}$ as
		\begin{equation}\label{eq:  linear regression}
			\hat w_i^{(t)}
			\; \approx \;
			\argmin_{ \norm{w_i} \,\leq\, W} 
			\;
			\sum_{k \,=\, 1}^K \,
			\left(\, 
			R_i^{(k)} \, -\, \big\langle \phi_i(s_i^{(k)},a_i^{(k)}), w_i \big\rangle
			\,\right)^2.
		\end{equation}
		\STATE Define $\hat{Q}_i^{(t)}(s, \cdot) \DefinedAs \big\langle \phi_i(s,\cdot), \hat w_i^{(t)} \big\rangle$ and player $i$'s policy for $s\in\calS$,
		\begin{equation}\label{eq: sample-based policy gradient ascent step}
			\pi_i^{(t+1)} (\,\cdot\,\vert\,s) 
			\; = \;
			\argmax_{\pi_i(\cdot\,|\,s) \,\in\,\Delta_{\xi}(\calA_i)}\;
			\left\{\,  
			\big\langle 
			\pi_i(\cdot\,|\,s),\, \hat{Q}_i^{(t)}(s, \cdot)
			\big\rangle_{\calA_i}
			\, - \, 
			\frac{1}{2\eta}\,
			\big\|
			\pi_i(\cdot\,|\,s)-\pi_i^{(t)}(\cdot\,|\,s)
			\big\|^2 
			\,\right\}.
		\end{equation}
		\ENDFOR
		\ENDFOR
	\end{algorithmic}
\end{algorithm*}

\begin{algorithm*}[tbh]
	\caption{ Independent optimistic policy gradient ascent}
	\label{alg: OPMA smooth Q}
	\begin{algorithmic}[1]
		\STATE \textbf{Parameters}: $0<\eta\leq \frac{1-\gamma}{32\sqrt{A}}$ and a non-increasing sequence $\{\alpha^{(t)}\}_{t=1}^\infty$ that satisfies 
		\begin{align*}
			0
			\; < \; 
			\alpha^{(t)}
			\; \leq \;
			\frac{1}{6} \ \ \ \text{for all\ } t \qquad \qquad \text{and}\qquad \qquad \sum_{t\,=\,t'}^\infty \alpha^{(t)}
			\; = \;
			\infty \ \  \text{\ for any }t'.
		\end{align*}

		\STATE \textbf{Initialization}: Let $x_{s}^{(1)}=\bar x_{s}^{(1)}=y_{s}^{(1)}=\bar y_{s}^{(1)}={1}/{A}$ and $\calV_s^{(0)} = 0$ for all $s\in\calS$.
		\FOR{step $t=1,2,\ldots$} 
		\STATE Define $\calQ_s^{(t)} \in \mathbb{R}^{A\times A}$ for all $s\in\calS$,
		\[
		\calQ_s^{(t)}(a_1,a_2)
		\; = \;
		r(s,a_1,a_2) 
		\,+\,
		\gamma\, \mathbb{E}_{s'\,\sim\,\mathbb{P}
			(\cdot\,\vert\,s,a_1,a_2) }
		\left[ \calV_{s'}^{(t-1)}
		\right].
		\]
		\STATE Define two players' policies for $s\in \calS$, 
		\begin{equation}\label{eq: optimistic gradient ascents}
			\begin{array}{rcl}
				\bar x_s^{(t+1)} 
				& = & \displaystyle
				\argmax_{x_s \,\in\,\Delta(\calA_1)}\;
				\left\{\,  
				x_s^\top \calQ_s^{(t)} y_s^{(t)}
				\, - \, 
				\frac{1}{2\eta}
				\big\|
				x_s-\bar x_s^{(t)}
				\big\|^2 
				\,\right\}
				\\[0.4cm]
				x_s^{(t+1)}  
				& = & \displaystyle
				\argmax_{x_s \,\in\,\Delta(\calA_1)}\;
				\left\{\,  
				x_s^\top \calQ_s^{(t)} y_s^{(t)}
				\, - \, 
				\frac{1}{2\eta}
				\big\|
				x_s-\bar x_s^{(t+1)}
				\big\|^2 
				\,\right\}
				\\[0.4cm]
				\bar y_s^{(t+1)} 
				& = & \displaystyle
				\argmax_{y_s \,\in\,\Delta(\calA_2)}\;
				\left\{\,  
				(x_s^{(t)})^\top \calQ_s^{(t)} y_s
				\, - \, 
				\frac{1}{2\eta}
				\big\|
				y_s-\bar y_s^{(t)}
				\big\|^2 
				\,\right\}
				\\[0.4cm]
				y_s^{(t+1)}  
				& = & \displaystyle
				\argmax_{y_s \,\in\,\Delta(\calA_2)}\;
				\left\{\,  
				(x_s^{(t)})^\top \calQ_s^{(t)} y_s
				\, - \, 
				\frac{1}{2\eta}
				\big\|
				y_s-\bar y_s^{(t+1)}
				\big\|^2 
				\,\right\}
			\end{array}
		\end{equation}
		\[
		\calV_s^{(t)}
		\; = \;
		(1-\alpha^{(t)}) \calV_s^{(t-1)}
		\, + \, 
		\alpha^{(t)} (x_s^{(t)})^\top \calQ_s^{(t)} y_s^{(t)}.
		\]
		\ENDFOR
	\end{algorithmic}
\end{algorithm*}
\onecolumn

\section{Proofs for Section~\ref{sec: gradient plays}}
\label{ap: gradient plays}

In this section, we provide proofs of \pref{thm: convergence PMA} and \pref{thm: convergence PMA cooperative} in \pref{ap: convergence PMA} and \pref{ap: convergence PMA cooperative}, respectively.

\subsection{\pfref{thm: convergence PMA}}\label{ap: convergence PMA}

We first seek to decompose the difference of a potential function $\Phi^\pi(\mu)$ at two different policies for any state distribution $\mu$. 

Let $\Psi^\pi$ : $\Pi \to \mathbb{R}$ be any multivariate function mapping a policy $\pi\in\Pi$ to a real number. In~\pref{lem: decomposition lemma}, we show that the difference $\Psi^{\pi'} - \Psi^\pi$ at any two policies $\pi,\pi'$ equals to a sum of several partial differences. For $i,j\in\{1, \ldots, N\}$ with $i<j$, we denote by ``$i\sim j$'' the set of indices $\{k \,\vert\, i<k<j\}$, ``$<i$'' the set of indices $\{ k\,\vert\, k =1,\ldots,i-1 \}$, and ``$>j$'' the set of indices $\{ k\,\vert\, k =j+1,\ldots,N \}$. We use the shorthand $\pi_I\DefinedAs \{\pi_k\}_{k\, \in\, I}$ to represent the joint policy for all players $k\in I$. For example, when $I = i\sim j$, $\pi_{I} = \{\pi_k\}_{k \,=\,i+1}^{j-1}$ is a joint policy for players from $i+1$ to $j-1$; $\pi_{<i,\, i\sim j}$, $\pi_{<i}$, and $\pi_{>j}$ can be introduced similarly. 

\begin{lemma}[Multivariate function difference]
\label{lem: decomposition lemma}
    For any function $\Psi^\pi$: $\Pi\to\mathbb{R}$, and any two policies $\pi, \pi'\in\Pi$,
    \begin{equation}\label{eq: decomposition}
    \begin{array}{rcl}
        \Psi^{\pi'} \,-\, \Psi^\pi
        & = & \displaystyle 
        \sum_{i \, = \, 1}^N  \left(\Psi^{\pi'_i,\, \pi_{-i}} \,-\, \Psi^\pi \right)
        \\[0.2cm]
        & & \displaystyle 
        +\, \sum_{i \, = \, 1}^N \sum_{j \, = \, i+1}^N \Big(\Psi^{\pi_{<i, i\sim j},\, \pi'_{>j}, \,\pi_i',\, \pi_j'} 
        \, - \, 
        \Psi^{\pi_{<i, i\sim j},\, \pi'_{>j}, \,\pi_i,\, \pi_j'}
        \\[0.2cm]
        && \;\;\;\; \;\;\;\; \;\;\;\; \;\;\;\; \;\;\;\; \;\;\;\; 
        -\, \Psi^{\pi_{<i, i\sim j},\, \pi'_{>j}, \,\pi_i',\, \pi_j} 
        \, + \, 
        \Psi^{\pi_{<i, i\sim j},\, \pi'_{>j},\, \pi_i,\, \pi_j} \Big).
    \end{array}
    \end{equation}
\end{lemma}

\begin{proof}[\pfref{lem: decomposition lemma}]
    We prove~\pref{eq: decomposition} by induction on the number of players $N$. In the basic step: $N=2$, the right-hand side of~\pref{eq: decomposition} becomes 
   \[
   \left(\Psi^{\pi_1',\, \pi_2} - \Psi^{\pi_1,\, \pi_2}\right) 
   \, + \,
   \left(\Psi^{\pi_1,\, \pi_2'} - \Psi^{\pi_1,\, \pi_2}\right)  
   \, + \,
   \left(\Psi^{\pi_1',\, \pi_2'} - \Psi^{\pi_1,\, \pi_2'} - \Psi^{\pi_1',\, \pi_2} + \Psi^{\pi_1,\, \pi_2}\right)
   \]
    which equals to the left-hand side: $\Psi^{\pi_1', \pi_2'} - \Psi^{\pi_1, \pi_2}$. 
    
    Assume the equality~\pref{eq: decomposition} holds for $N$ players. We next consider the induction step for $N+1$ players . By subtracting and adding $ \Psi^{\pi_{\leq N},\, \pi'_{N+1}}$,  
    \begin{equation}\label{eq: decompose two differences}
    \Psi^{\pi'} \,-\, \Psi^\pi
    \; = \;
    \underbrace{\left( \Psi^{\pi'_{\leq N},\, \pi'_{N+1}} - \Psi^{\pi_{\leq N},\, \pi'_{N+1}}\right)}_{\textbf{Diff}_{\leq N}}
    \, + \, 
    \underbrace{\left(\Psi^{\pi_{\leq N},\, \pi'_{N+1}} - \Psi^{\pi_{\leq N},\, \pi_{N+1}}\right)}_{\textbf{Diff}_{N+1}}.
    \end{equation}
    In~\pref{eq: decompose two differences}, we use the shorthand $\pi'_{\leq N}$ and $\pi_{\leq N}$ for $\{\pi'_k\}_{k\,=\,1}^N$ and $\{\pi_k\}_{k\,=\,1}^N$, respectively. 
    We note that ${\textbf{Diff}_{\leq N}}$ or ${\textbf{Diff}_{N+1}}$ can be viewed as a function for $N$ players if we fix the $(N+1)$th policy. By the induction assumption, for the first term ${\textbf{Diff}_{\leq N}}$,   
    \[
    \begin{array}{rcl}
        {\textbf{Diff}_{\leq N}}
        & = & \displaystyle 
        \sum_{i \, = \, 1}^N  \left( \Psi^{\pi'_i, \, \pi_{<i, i\sim N+1},  \, \pi_{N+1}'} 
        \,-\,
        \Psi^{\pi_{\leq N}, \, \pi'_{N+1}} \right) 
        \\[0.2cm]
        && \displaystyle
        + \, \sum_{i \, = \, 1}^N \sum_{j \, = \, i+1}^N \Big( \Psi^{\pi_{<i, i\sim j}, \, \pi'_{>j}, \, \pi_i', \, \pi_j',\,  \pi'_{N+1}}  
        \,-\,
        \Psi^{\pi_{<i, i\sim j}, \, \pi'_{>j}, \, \pi_i, \, \pi_j',\, \, \pi'_{N+1} } 
        \\[0.2cm]
        && \;\;\;\; \;\;\;\; \;\;\;\; \;\;\;\; \;\;\;\; \;\;\;\; 
        -\, \Psi^{\pi_{<i, i\sim j}, \, \pi'_{>j}, \, \pi_i', \, \pi_j, \, \, \pi'_{N+1}} 
        \,+\,
        \Psi^{\pi_{<i, i\sim j}, \, \pi'_{>j}, \, \pi_i, \, \pi_j,\, \pi'_{N+1}} 
        \Big) 
        \\[0.2cm]
        & = & \displaystyle
         \sum_{i \, = \, 1}^N  \left(\Psi^{\pi'_i,\, \pi_{<i, i\sim N+1},\, \pi_{N+1}} 
         \,-\,
         \Psi^{\pi_{\leq N},\, \pi_{N+1}}\right) 
         \\[0.2cm]
        &&  \displaystyle 
        +\, \sum_{i \, = \, 1}^N  \left(\Psi^{\pi'_i,\, \pi_{<i, i\sim N+1},\, \pi_{N+1}'} 
        \,-\,
        \Psi^{\pi_{\leq N},\, \pi'_{N+1}}  -\Psi^{\pi'_i,\, \pi_{<i, i\sim N+1},\, \pi_{N+1}} \,+\,
        \Psi^{\pi_{\leq N},\, \pi_{N+1}} \right) 
        \\[0.2cm]
        && \displaystyle
        + \, \sum_{i \, = \, 1}^N \sum_{j \, = \, i+1}^N \Big( \Psi^{\pi_{<i, i\sim j}, \, \pi'_{>j}, \, \pi_i', \, \pi_j',\,  \pi'_{N+1}}  
        \,-\, \Psi^{\pi_{<i, i\sim j}, \, \pi'_{>j}, \, \pi_i, \, \pi_j',\, \, \pi'_{N+1} } 
        \\[0.2cm]
        && \;\;\;\; \;\;\;\; \;\;\;\; \;\;\;\; \;\;\;\; \;\;\;\; 
        -\, \Psi^{\pi_{<i, i\sim j}, \, \pi'_{>j}, \, \pi_i', \, \pi_j, \, \, \pi'_{N+1}} 
        \,+\, \Psi^{\pi_{<i, i\sim j}, \, \pi'_{>j}, \, \pi_i, \, \pi_j,\, \pi'_{N+1}} 
        \Big)
    \end{array}
    \]
    where we use $\pi_{>j}'$ to represent $\{\pi_k'\}_{k\,=\,j+1}^N$.  
    
    Adding ${\textbf{Diff}_{N+1}}$ to the last equivalent expression of ${\textbf{Diff}_{\leq N}}$ above yields
    \[
    \begin{array}{rcl}
        {\textbf{Diff}_{\leq N}} + {\textbf{Diff}_{N+1}}
        & = & \displaystyle
         \sum_{i \, = \, 1}^{N+1}  \left(\Psi^{\pi'_i,\, \pi_{-i}} \,-\, \Psi^{\pi}\right) 
         \\[0.2cm]
        &&  \displaystyle 
        +\, \sum_{i \, = \, 1}^N \sum_{j \, = \, N+1}^{N+1} \Big(\Psi^{ \pi_{<i, i\sim j},\, \pi_{>j}',\,\pi'_i,\, \pi_{j}'} 
        \,-\,
        \Psi^{\pi_{<i, i\sim j},\, \pi'_{>j},\, \pi_i,\,\pi_{j}'}  
        \\[0.2cm]
        && \;\;\;\; \;\;\;\; \;\;\;\; \;\;\;\; \;\;\;\; \;\;\;\; 
        -\,\Psi^{\pi_{<i, i\sim j},\, \pi_{>j}',\, \pi_{i}',\,\pi_j} 
        \,+\,
        \Psi^{\pi_{<i, i\sim j},\, \pi_{>j}',\,\pi_i,\,\pi_j} \Big) 
        \\[0.2cm]
        && \displaystyle
        + \, \sum_{i \, = \, 1}^N \sum_{j \, = \, i+1}^N \Big( \Psi^{\pi_{<i, i\sim j}, \, \pi'_{>j}, \, \pi_i', \, \pi_j',\,  \pi'_{N+1}}  
        \,-\, 
        \Psi^{\pi_{<i, i\sim j}, \, \pi'_{>j}, \, \pi_i, \, \pi_j',\, \, \pi'_{N+1} } 
        \\[0.2cm]
        && \;\;\;\; \;\;\;\; \;\;\;\; \;\;\;\; \;\;\;\; \;\;\;\; 
        -\, \Psi^{\pi_{<i, i\sim j}, \, \pi'_{>j}, \, \pi_i', \, \pi_j, \, \, \pi'_{N+1}} 
        \,+\, 
        \Psi^{\pi_{<i, i\sim j}, \, \pi'_{>j}, \, \pi_i, \, \pi_j,\, \pi'_{N+1}} 
        \Big)
        \\[0.2cm]
        & = & \displaystyle
         \sum_{i \, = \, 1}^{N+1}  \left(\Psi^{\pi'_i,\, \pi_{-i}} \,-\, \Psi^{\pi}\right) 
         \\[0.2cm]
        && \displaystyle
        + \, \sum_{i \, = \, 1}^{N+1} \sum_{j \, = \, i+1}^{N+1} \Big( \Psi^{\pi_{<i, i\sim j}, \, \pi'_{>j}, \, \pi_i', \, \pi_j'}  \,-\, \Psi^{\pi_{<i, i\sim j}, \, \pi'_{>j}, \, \pi_i, \, \pi_j' } 
        \\[0.2cm]
        && \;\;\;\; \;\;\;\; \;\;\;\; \;\;\;\; \;\;\;\; \;\;\;\; 
        -\, \Psi^{\pi_{<i, i\sim j}, \, \pi'_{>j}, \, \pi_i', \, \pi_j} \,+\, \Psi^{\pi_{<i, i\sim j}, \, \pi'_{>j}, \, \pi_i, \, \pi_j} 
        \Big)
    \end{array}
    \]
    where the first equality has a slight abuse of the notation: $\pi_{>j}'$ represents $\{\pi_k'\}_{k\,=\,j+1}^{N+1}$ in the first double sum and $\pi_{>j}'$ represents $\{\pi_k'\}_{k\,=\,j+1}^N$ in the second double sum. Therefore,~\pref{eq: decomposition} holds for $N+1$ players. The proof is completed by induction.
\end{proof}

We apply~\pref{lem: decomposition lemma} to the potential function $\Phi^\pi(\mu)$ at two consecutive policies $\pi^{(t+1)}$ and $\pi^{(t)}$ in~\pref{alg: PMA}, where $\mu$ is an initial state distribution. We use the shorthand $\Phi^{(t)}(\mu)$ for $\Phi^{\pi^{(t)}}(\mu)$, the value of potential function at policy $\pi^{(t)}$.

\begin{lemma}[Policy improvement: Markov potential games]
\label{lem: MPG policy improvement}
    For MPG~\pref{eq: Markov potential game} with any state distribution $\mu$, the potential function $\Phi^\pi(\mu)$ at two consecutive policies $\pi^{(t+1)}$ and $\pi^{(t)}$ in~\pref{alg: PMA} satisfies 
    \[
    \begin{array}{rcl}
    \text{\normalfont(i)} \;\;
    \Phi^{(t+1)}(\mu) - \Phi^{(t)}(\mu)
    & \geq & \displaystyle
    \frac{1}{2\eta(1-\gamma)}\sum_{i \, = \, 1}^N \sum_{s}d_\mu^{\pi_i^{(t+1)}, \pi^{(t)}_{-i}}(s)\left\|\pi_i^{(t+1)}(\cdot|s)-\pi_i^{(t)}(\cdot|s)\right\|^2 - \frac{4 \eta^2 A^2 N^2}{(1-\gamma)^5}
    \\[0.2cm]
    \text{\normalfont(ii)} \;\;
    \Phi^{(t+1)}(\mu) - \Phi^{(t)}(\mu)
    & \geq & \displaystyle
    \frac{1}{2\eta(1-\gamma)}\sum_{i \, = \, 1}^N \sum_{s}d_\mu^{\pi_i^{(t+1)}, \pi^{(t)}_{-i}}(s)\left(1- \frac{4\eta \kappa_\mu^3AN}{(1-\gamma)^4}\right) \left\|\pi_i^{(t+1)}(\cdot|s)-\pi_i^{(t)}(\cdot|s)\right\|^2
    \end{array}
    \]
    where $\eta$ is the stepsize, $N$ is the number of players, $A$ is the size of one player's action space, and $\kappa_\mu$ is the distribution mismatch coefficient relative to $\mu$ (see $\kappa_\mu$ in \pref{def: distribution mismatch coeff}).
\end{lemma}

\begin{proof}[\pfref{lem: MPG policy improvement}]
    We let $\pi'=\pi^{(t+1)}$ and $\pi = \pi^{(t)}$ for brevity.
    By~\pref{lem: decomposition lemma} with $\Psi^\pi = \Phi^\pi(\mu)$, it is equivalent to analyze
    \begin{equation}\label{eq: DiffAB}
    \Phi^{(t+1)}(\mu) \,-\, \Phi^{(t)}(\mu)
    \; =\;
    \textbf{Diff}_\alpha \,+\, \textbf{Diff}_\beta
    \end{equation}
where
\[
    \textbf{Diff}_\alpha
    \; = \;    
    \sum_{i \, = \, 1}^N  \left(\Phi^{\pi'_i,\, \pi_{-i}}(\mu) - \Phi^\pi (\mu
    ) \right)
\]
\[
    \begin{array}{rcl}
        \textbf{Diff}_\beta
        & = & \displaystyle 
        \sum_{i \, = \, 1}^N \sum_{j \, = \, i+1}^N \Big(\Phi^{\pi_{<i, i\sim j},\, \pi'_{>j}, \,\pi_i',\, \pi_j'}(\mu) \, - \, \Phi^{\pi_{<i, i\sim j},\, \pi'_{>j}, \,\pi_i,\, \pi_j'}(\mu)
        \\[0.2cm]
        && \;\;\;\; \;\;\;\; \;\;\;\; \;\;\;\; \;\;\;\; \;\;\;\; 
        -\, \Phi^{\pi_{<i, i\sim j},\, \pi'_{>j}, \,\pi_i',\, \pi_j}(\mu) \, + \, \Phi^{\pi_{<i, i\sim j},\, \pi'_{>j},\, \pi_i,\, \pi_j}(\mu) \Big).
    \end{array}
    \]

\noindent\textbf{Bounding} $\textbf{Diff}_\alpha$. By the property of the potential function $\Phi^\pi(\mu)$, 
\begin{equation}\label{eq: potential diff calculation}
\begin{array}{rcl}
    \Phi^{\pi_i', \, \pi_{-i}}(\mu) \,-\, \Phi^\pi(\mu) 
    & = & 
    V_i^{\pi_i',\, \pi_{-i}}(\mu) \,-\, V_i^\pi(\mu) 
    \\[0.2cm]
    & = & \displaystyle
    \frac{1}{1-\gamma}\sum_{s, \, a_i} d_\mu^{\pi_i',\, \pi_{-i}}(s) \left(\pi_i'(a_i\,|\,s) - \pi_i(a_i\,|\,s)\right) \bar Q_i^{\pi_i,\,\pi_{-i}} (s,a_i)
\end{array}
\end{equation}
where the second equality is due to~\pref{lem: performance difference} using $\hat\pi_i = \pi_i'$ and $\bar \pi_i = \pi_i$. 
The optimality of $\pi_i'=\pi_i^{(t+1)}$ in line~4 of~\pref{alg: PMA} leads to 
\begin{equation}\label{eq: optimality condition}
\big\langle\pi_i'(\cdot\,\vert\,s), \bar Q_i^{\pi_i,\pi_{-i}}(s,\cdot)\big\rangle_{\calA_i}  \,-\, \frac{1}{2\eta}\big\|\pi_i'(\cdot\,\vert\,s) - \pi_i(\cdot\,\vert\,s)\big\|^2 
\; \geq \;
\big\langle\pi_i(\cdot\,\vert\,s), \bar Q_i^{\pi_i,\pi_{-i}}(s,\cdot)\big\rangle_{\calA_i}.
\end{equation}
Combining \pref{eq: potential diff calculation} and \pref{eq: optimality condition}, we get
\[
    \Phi^{\pi_i', \, \pi_{-i}}(\mu) \,-\, \Phi^\pi(\mu) 
    \; \geq \;
    \frac{1}{2\eta(1-\gamma)}\sum_{s}d_\mu^{\pi_i',\, \pi_{-i}}(s)\left\|\pi_i'(\cdot\,|\,s)-\pi_i(\cdot\,|\,s)\right\|^2. 
\]
Therefore,
\begin{equation}\label{eq: DiffA}
    \textbf{Diff}_\alpha
    \; \geq \;
    \frac{1}{2\eta(1-\gamma)}\sum_{i\,=\,1}^N\sum_{s}d_\mu^{\pi_i^{(t+1)},\, \pi_{-i}^{(t)}}(s)\left\|\pi_i^{(t+1)}(\cdot\,|\,s)-\pi_i^{(t)}(\cdot\,|\,s)\right\|^2.
\end{equation}

\noindent\textbf{Bounding} $\textbf{Diff}_\beta$. For simplicity, we denote $\tilde{\pi}_{-ij}$ as the joint policy of players $N\backslash \{i,j\}$ where players $<i$ and $i\sim j$ use $\pi$ and players $>j$ use $\pi'$. For each summand in $\textbf{Diff}_\beta$,
\[
\begin{array}{rcl}
     && \!\!\!\! \!\!\!\! \!\!\!\! 
     \Phi^{\tilde{\pi}_{-ij}, \,\pi_i',\, \pi_j'}(\mu) 
     \, - \, \Phi^{\tilde{\pi}_{-ij}, \,\pi_i,\, \pi_j'}(\mu)
     \,-\, 
     \Phi^{\tilde{\pi}_{-ij}, \,\pi_i',\, \pi_j}(\mu) 
     \, + \, \Phi^{\tilde{\pi}_{-ij},\, \pi_i,\, \pi_j}(\mu) 
     \\[0.4cm]
     &  \overset{(a)}{=} & V_i^{\tilde{\pi}_{-ij}, \,\pi_i',\, \pi_j'}(\mu) 
     \, - \, V_i^{\tilde{\pi}_{-ij}, \,\pi_i,\, \pi_j'}(\mu)
    \,-\, 
    V_i^{\tilde{\pi}_{-ij}, \,\pi_i',\, \pi_j}(\mu) 
    \, + \, V_i^{\tilde{\pi}_{-ij},\, \pi_i,\, \pi_j}(\mu) 
        \\[0.4cm]
     &  \overset{(b)}{=}  & \displaystyle
     \frac{1}{1-\gamma} \sum_{s,\, a_i} d_\mu^{\tilde{\pi}_{-ij}, \,\pi_i',\, \pi_j'}(s) \left( \pi_i'(a_i\,|\,s) - \pi_i(a_i\,|\,s) \right) \bar Q_i^{\tilde{\pi}_{-ij}, \,\pi_i,\, \pi_j'}(s,a_i)
     \\[0.4cm]
     &    & \displaystyle -\,
     \frac{1}{1-\gamma} \sum_{s,\, a_i} d_\mu^{\tilde{\pi}_{-ij}, \,\pi_i',\, \pi_j}(s) \left( \pi_i'(a_i\,|\,s) - \pi_i(a_i\,|\,s) \right) \bar Q_i^{\tilde{\pi}_{-ij}, \,\pi_i,\, \pi_j}(s,a_i)
     \\[0.4cm]
     &  =  & \displaystyle
     \frac{1}{1-\gamma} \sum_{s,\, a_i} d_\mu^{\tilde{\pi}_{-ij}, \,\pi_i',\, \pi_j'}(s) \left( \pi_i'(a_i\,|\,s) - \pi_i(a_i\,|\,s) \right) \left( \bar Q_i^{\tilde{\pi}_{-ij}, \,\pi_i,\, \pi_j'}(s,a_i) -\bar Q_i^{\tilde{\pi}_{-ij}, \,\pi_i,\, \pi_j}(s,a_i)\right)
     \\[0.4cm]
     &    & \displaystyle +\,
     \frac{1}{1-\gamma} \sum_{s,\, a_i} 
     \left(
     d_\mu^{\tilde{\pi}_{-ij}, \,\pi_i',\, \pi_j'}(s) - 
     d_\mu^{\tilde{\pi}_{-ij}, \,\pi_i',\, \pi_j}(s)
     \right)
     \left( \pi_i'(a_i\,|\,s) - \pi_i(a_i\,|\,s) \right) \bar Q_i^{\tilde{\pi}_{-ij}, \,\pi_i,\, \pi_j}(s,a_i)
     \\[0.4cm]
     &  \geq  & \displaystyle
     -\, \frac{1}{1-\gamma} \sum_{s} d_\mu^{\tilde{\pi}_{-ij}, \,\pi_i',\, \pi_j'}(s) \norm{ \pi_i'(\cdot\,|\,s) - \pi_i(\cdot\,|\,s) }_1 \norm{ \bar Q_i^{\tilde{\pi}_{-ij}, \,\pi_i,\, \pi_j'}(s,\cdot) -\bar Q_i^{\tilde{\pi}_{-ij}, \,\pi_i,\, \pi_j}(s,\cdot)}_\infty
     \\[0.4cm] 
     & & \displaystyle -\, \frac{1}{1-\gamma}\sum_{s} \left|d_\mu^{\tilde{\pi}_{-ij}, \,\pi_i',\, \pi_j'}(s) - 
     d_\mu^{\tilde{\pi}_{-ij}, \,\pi_i',\, \pi_j}(s)
     \right| \norm{ \pi_i'(\cdot\,|\,s) - \pi_i(\cdot\,|\,s) }_1 
      \norm{ \bar Q_i^{\tilde{\pi}_{-ij}, \,\pi_i,\, \pi_j}(s,\cdot)}_\infty  \\[0.4cm]
     &  \overset{(c)}{\geq}  & \displaystyle
     -\,
     \frac{1}{(1-\gamma)^3} \left(\max_s  \norm{ \pi_i'(\cdot\,|\,s) - \pi_i(\cdot\,|\,s) }_1\right) \left(\max_s  \norm{ \pi_j'(\cdot\,|\,s) - \pi_j(\cdot\,|\,s) }_1\right) \\[0.4cm]  & & \displaystyle -\, \frac{1}{(1-\gamma)^2}\left(\max_s  \norm{ \pi_j'(\cdot\,|\,s) - \pi_j(\cdot\,|\,s) }_1\right) \left(\max_s  \norm{ \pi_i'(\cdot\,|\,s) - \pi_i(\cdot\,|\,s) }_1\right) \\[0.4cm]
     &  \overset{(d)}{\geq}  & \displaystyle
     - \,\frac{8 \eta^2 A^2}{(1-\gamma)^5}
\end{array}
\]
where $(a)$ is due to the property of the potential function, $(b)$ is due to~\pref{lem: performance difference}; for $(c)$, we use~\pref{lem: Q diff for multiagent},~\pref{lem: sum of occupancy diff}, and the fact that $\sum_{s} d_\mu^{\tilde{\pi}_{-ij},\pi_i',\pi_j'}(s)=1$ and $\norm{ \bar Q_i^{\tilde{\pi}_{-ij}, \,\pi_i,\, \pi_j}(s,\cdot)}_\infty\leq \frac{1}{1-\gamma}$; 
The last inequality $(d)$ follows a direct result from the optimality of $\pi_i'=\pi_i^{(t+1)}$ given by~\pref{eq: optimality condition} and $\|\cdot\|\leq \sqrt{A}\|\cdot\|_\infty$ and $\|\cdot\|_1\leq \sqrt{A}\|\cdot\|$: 
\begin{align*} 
      \norm{ \pi_i'(\cdot\,\vert\,s) - \pi_i(\cdot\,\vert\,s) }^2 
       &\leq\; 2\eta \big\langle\pi_i^{(t+1)}(\cdot\,\vert\,s) - \pi_i^{(t)}(\cdot\,\vert\,s), \bar Q_i^{\pi_i,\pi_{-i}}(s,\cdot)\big\rangle_{\calA_i} 
      \\
      & \leq\; 2\eta \norm{ \pi_j^{(t+1)}(\cdot\,\vert\,s) - \pi_j^{(t)}(\cdot\,\vert\,s) }\norm{\bar Q_i^{\pi_i,\pi_{-i}}(s,\cdot)} \\
      &\implies\; \norm{ \pi_j^{(t+1)}(\cdot\,\vert\,s) - \pi_j^{(t)}(\cdot\,\vert\,s) }
    \;\leq\; 
    2\eta \norm{\bar Q_i^{\pi_i,\pi_{-i}}(s,\cdot)} 
    \;\leq\;
    \frac{2\eta \sqrt A}{1-\gamma} \\
    &\implies\; \norm{ \pi_j^{(t+1)}(\cdot\,\vert\,s) - \pi_j^{(t)}(\cdot\,\vert\,s) }_1
    \;\leq\;
    \frac{2\eta A}{1-\gamma}. 
\end{align*}

Therefore,
\begin{equation}\label{eq: DiffB}
    \textbf{Diff}_\beta 
    \; \geq \; 
    - \frac{N(N-1)}{2}\times \frac{8\eta^2 A^2}{(1-\gamma)^5} 
    \; \geq \;
    -\,\frac{ 4 \eta^2 A^2 N^2}{(1-\gamma)^5}.
\end{equation}

We now complete the proof of (i) by combining~\pref{eq: DiffAB},~\pref{eq: DiffA}, and~\pref{eq: DiffB}. 

Alternatively, by \pref{lem: second-order PDL}, we can bound each summand of $\textbf{Diff}_\beta$ by 
\begin{align*}
    &\Phi^{\tilde{\pi}_{-ij}, \,\pi_i',\, \pi_j'}(\mu) 
    \, - \, \Phi^{\tilde{\pi}_{-ij}, \,\pi_i,\, \pi_j'}(\mu)
    \,-\,
    \Phi^{\tilde{\pi}_{-ij}, \,\pi_i',\, \pi_j}(\mu) 
    \, + \, \Phi^{\tilde{\pi}_{-ij},\, \pi_i,\, \pi_j}(\mu) 
     \\
     &= \;
     V_i^{\tilde{\pi}_{-ij}, \,\pi_i',\, \pi_j'}(\mu) 
     \, - \, V_i^{\tilde{\pi}_{-ij}, \,\pi_i,\, \pi_j'}(\mu)
      \, -\,
      V_i^{\tilde{\pi}_{-ij}, \,\pi_i',\, \pi_j}(\mu) 
      \, + \, V_i^{\tilde{\pi}_{-ij},\, \pi_i,\, \pi_j}(\mu) 
        \\
    &\geq \;-\, \frac{2\kappa_\mu^2  A}{(1-\gamma)^4}\sum_s d_\mu^{\tilde{\pi}_{-ij}, \pi_i,\pi_j}(s)\left(\norm{\pi_i(\cdot\,|\,s) - \pi_i'(\cdot\,|\,s)}^2 + \norm{\pi_j(\cdot\,|\,s) - \pi_j'(\cdot\,|\,s)}^2\right). 
\end{align*}
Thus, 
\begin{align*}
    \textbf{Diff}_\beta 
    &\;\geq\; -\frac{2\kappa_\mu^2  A}{(1-\gamma)^4} \sum_{i\,=\,1}^N \sum_{j\,=\,i+1}^N \sum_s d_\mu^{\tilde{\pi}_{-ij}, \pi_i,\pi_j}(s)\left(\norm{\pi_i(\cdot\,|\,s) - \pi_i'(\cdot\,|\,s)}^2 + \norm{\pi_j(\cdot\,|\,s) - \pi_j'(\cdot\,|\,s)}^2\right) \\
    &\;\geq\; - \frac{2\kappa_\mu^3  NA}{(1-\gamma)^5} \sum_{i\,=\,1}^N \sum_s d_\mu^{\pi_i^{(t+1)},\, \pi_{-i}^{(t)}}(s)\norm{\pi_i^{(t)}(\cdot\,|\,s) - \pi_i^{(t+1)}(\cdot\,|\,s)}^2 . \tag{since $\frac{d_\mu^{\pi}(s)}{d_\mu^{\pi'}(s)}\leq \frac{\kappa_\mu}{1-\gamma}$ for any $\pi, \pi',s$} 
\end{align*}
Combining the inequality above with ~\pref{eq: DiffAB} and~\pref{eq: DiffA} finishes the proof of (ii). 
\end{proof}

\begin{lemma}\label{lem: Q diff for multiagent}
Suppose $i<j$ for $i,j = 1,\ldots,N$. Let $\tilde{\pi}_{-ij}$ be the policy for all players but $i,j$ and $\pi_i$ be the policy for player $i$. For any two policies for player $j$: $\pi_j$ and $\pi_j'$, we have 
\[
    \max_s\norm{ \bar Q_i^{\tilde{\pi}_{-ij}, \,\pi_i,\, \pi_j'}(s,\cdot) -\bar Q_i^{\tilde{\pi}_{-ij}, \,\pi_i,\, \pi_j}(s,\cdot)}_\infty
    \;\leq\;
    \frac{1}{(1-\gamma)^2}
        \max_{s} \norm{ \pi_j'(\cdot\,\vert\,s) - \pi_j(\cdot\,\vert\,s) }_1.
\]
\end{lemma}
\begin{proof}[\pfref{lem: Q diff for multiagent}]
    We note that $\bar Q_i^{\tilde{\pi}_{-ij}, \,\pi_i,\, \pi_j'}(s,\cdot)$ and $\bar Q_i^{\tilde{\pi}_{-ij}, \,\pi_i,\, \pi_j}(s,\cdot)$ are averaged action value functions for player $i$ using policy $\pi_i$, but they have different underlying averaged MDPs because of different policies executed by player $j$. Hence, we can directly apply \pref{lem: Q diff lemma}. Specifically, let $(r,p)$ be the averaged reward and transition functions for player $i$ induced by $(\tilde{\pi}_{-ij}, \pi_j)$, and $(\tilde{r}, \tilde{p})$ be those induced by $(\tilde{\pi}_{-ij}, \pi_j')$. Then, 
    \[
    \begin{array}{rcl}
        & & \!\!\!\!  \!\!\!\!  \!\!
        |r(s,a_i)-\tilde{r}(s,a_i)| 
        \\[0.2cm]
        &=& \displaystyle
        \left|\sum_{a_j, a_{-ij}} r(s,a_i, a_j, a_{-ij}) \pi_j(a_j\,|\,s)\tilde{\pi}_{-ij}(a_{-ij}\,|\,s) - \sum_{a_j, a_{-ij}} r(s,a_i, a_j, a_{-ij}) \pi_j'(a_j\,|\,s)\tilde{\pi}_{-ij}(a_{-ij}\,|\,s) \right| 
        \\[0.2cm]
        &\leq& \norm{\pi_j(\cdot\,|\,s) - \pi_j'(\cdot\,|\,s)}_1
    \end{array}
    \]
    and 
    \[
    \begin{array}{rcl}
        && \!\!\!\!  \!\!\!\!  \!\!
        \norm{p(\cdot|s,a_i)-\tilde{p}(\cdot|s,a_i)}_1 
        \\[0.2cm]
        &=& \displaystyle
        \sum_{s'} \left|
        \sum_{a_j,\, a_{-ij}} p(s'\,|\,s,a_i, a_j, a_{-ij}) 
        \left( 
        \pi_j(a_j\,|\,s) - \pi_j'(a_j\,|\,s) 
        \right)
        \tilde{\pi}_{-ij}(a_{-ij}\,|\,s) \right|
        \\[0.4cm]
        &\leq& \displaystyle
        \sum_{s'} 
        \sum_{a_j,\, a_{-ij}} p(s'\,|\,s,a_i, a_j, a_{-ij}) 
        \tilde{\pi}_{-ij}(a_{-ij}\,|\,s) 
        \left|
        \pi_j(a_j\,|\,s) - \pi_j'(a_j\,|\,s) 
        \right| 
        \\[0.4cm]
        &\leq& \norm{\pi_j(\cdot|s) - \pi_j'(\cdot|s)}_1.  
    \end{array}
    \]
    Application of two inequalities above to~\pref{lem: Q diff lemma} competes the proof. 
\end{proof}

\begin{proof}[\pfref{thm: convergence PMA}]
    By the optimality of $\pi_i^{(t+1)}$ in line~4 of~\pref{alg: PMA},  
    \[
        \big\langle
        \pi_i'(\cdot\,|\,s) - \pi_i^{(t+1)}(\cdot\,|\,s)
        , \, 
        \eta \bar Q_i^{(t)}(s,\cdot) - \pi^{(t+1)}_i(\cdot\,|\,s) + \pi^{(t)}_i(\cdot\,|\,s) 
        \big\rangle_{\calA_i}
        \; \leq \;
        0,\;
         \text{ for any } \pi_i' \,\in\, \Pi_i.     
    \]
    Hence, if $\eta\leq \frac{1-\gamma}{\sqrt{A}}$, then for any $\pi_i'\in\Pi_i$,
    \[
    \begin{array}{rcl}
         & & \!\!\!\! \!\!\!\! \!\! 
         \displaystyle
         \big\langle
        \pi_i'(\cdot\,|\,s) - \pi_i^{(t)}(\cdot\,|\,s)
        ,\,
        \bar Q_i^{(t)}(s,\cdot)
        \big\rangle_{\calA_i}
        \\[0.2cm]
        & = & \displaystyle
        \big\langle
        \pi_i'(\cdot\,|\,s) - \pi_i^{(t+1)}(\cdot\,|\,s)
        ,\,
        \bar Q_i^{(t)}(s,\cdot)
        \big\rangle_{\calA_i}
        \,+\,
        \big\langle
        \pi_i^{(t+1)}(\cdot\,|\,s) - \pi_i^{(t)}(\cdot\,|\,s)
        ,\,
        \bar Q_i^{(t)}(s,\cdot)
        \big\rangle_{\calA_i}
        \\[0.2cm]
        & \leq & \displaystyle
        \frac{1}{\eta} \big\langle\pi_i'(\cdot\,|\,s) - \pi_i^{(t+1)}(\cdot\,|\,s),\, \pi^{(t+1)}_i(\cdot\,|\,s) - \pi^{(t)}_i(\cdot\,|\,s) \big\rangle_{\calA_i} \,+\,
        \big\langle
        \pi_i^{(t+1)}(\cdot\,|\,s) - \pi_i^{(t)}(\cdot\,|\,s)
        ,\,
        \bar Q_i^{(t)}(s,\cdot)
        \big\rangle_{\calA_i}
        \\[0.4cm]
        & \overset{(a)}{\leq} & \displaystyle
        \frac{2}{\eta} \norm{\pi^{(t+1)}_i(\cdot\,|\,s) - \pi^{(t)}_i(\cdot\,|\,s) }
        \,+\,
        \norm{
        \pi_i^{(t+1)}(\cdot\,|\,s) - \pi_i^{(t)}(\cdot\,|\,s)
        }
        \norm{
        \bar Q_i^{(t)}(s,\cdot)
        }
        \\[0.4cm]
        & \overset{(b)}{\leq} & \displaystyle  \frac{3}{\eta}\left\|\pi_i^{(t+1)}(\cdot|s) - \pi_i^{(t)}(\cdot|s)\right\|
    \end{array}
    \]
    where in $(a)$ we apply the Cauchy-Schwarz inequality and that $\|p-p'\|\leq \|p-p'\|_1\leq 2$ for any two distributions $p$ and $p'$; $(b)$ is because of~$\Vert{\bar Q_i^{(t)}(s,\cdot)}\Vert \leq \frac{\sqrt{A}}{1-\gamma}$ and $\eta \leq \frac{1-\gamma}{\sqrt{A}}$. Therefore, for any initial distribution $\rho$, 
    \begin{equation}\label{eq: summary bound}
        \begin{array}{rcl}
             & & \!\!\!\! \!\!\!\! \!\!
             \displaystyle
             \sum_{t \,=\, 1}^T\max_{i} \left(\max_{\pi_i'} V_i^{\pi_i', \, \pi_{-i}^{(t)}}(\rho) -  V_i^{\pi^{(t)}}(\rho) \right) \\[0.4cm]         
             & \overset{(a)}{=} & \displaystyle
             \frac{1}{1-\gamma}\sum_{t\, = \, 1}^T \max_{\pi_i'} \sum_{s,\, a_i} d_\rho^{ \pi_i', \pi_{-i}^{(t)}}(s)
             \left(\pi_i'(a_i\,|\,s) - \pi_i^{(t)}(a_i\,|\,s)\right)
             \bar Q_i^{{(t)}}(s,a_i)
             \\[0.4cm]
             & \overset{(b)}{\leq} & \displaystyle
             \frac{3}{\eta(1-\gamma)} \sum_{t \,=\, 1}^T \sum_{s} d_\rho^{\pi_i', \pi_{-i}^{(t)}}(s) \left\|\pi_i^{(t+1)}(\cdot\,|\,s) - \pi_i^{(t)}(\cdot\,|\,s)\right\| 
             \\[0.4cm]
             & \overset{(c)}{\lesssim} & \displaystyle
             \frac{\sqrt{\sup_{\pi\,\in\,\Pi}\|d^{\pi}_\rho/\nu\|_\infty}}{\eta(1-\gamma)^{\frac{3}{2}}} \sum_{t \,=\, 1}^T \sum_{s} \sqrt{ d_\rho^{\pi_i', \pi_{-i}^{(t)}}(s) \times d_{\nu}^{\pi_i^{(t+1)}, \pi_{-i}^{(t)}}(s) } \left\|\pi_i^{(t+1)}(\cdot\,|\,s) - \pi_i^{(t)}(\cdot\,|\,s)\right\| 
             \\[0.5cm]
             & \overset{(d)}{\leq} & \displaystyle
             \frac{\sqrt{\sup_{\pi\,\in\,\Pi}\|d^{\pi}_\rho/\nu\|_\infty}}{\eta(1-\gamma)^{\frac{3}{2}}} \sqrt{
             \sum_{t \,=\, 1}^T \sum_{s} d_\rho^{\pi_i', \pi_{-i}^{(t)}}(s)
             }
             \times\,
             \sqrt{
             \sum_{t \,=\, 1}^T \sum_{s} d_{\nu}^{\pi_i^{(t+1)}, \pi_{-i}^{(t)}}(s) \left\|\pi_i^{(t+1)}(\cdot\,|\,s) - \pi_i^{(t)}(\cdot\,|\,s)\right\|^2
             }
             \\[0.5cm]
             & \overset{(e)}{\leq} & \displaystyle
             \frac{\sqrt{\sup_{\pi\,\in\,\Pi}\|d^{\pi}_\rho/\nu\|_\infty}}{\eta(1-\gamma)^{\frac{3}{2}}} \sqrt{
             \sum_{t \,=\, 1}^T \sum_{s} d_{\rho}^{\pi_i', \pi_{-i}^{(t)}}(s)
             }
             \times\,
             \sqrt{
             \sum_{t \,=\, 1}^T \sum_{i\,=\,1}^N \sum_{s} d_{\nu}^{\pi_i^{(t+1)}, \pi_{-i}^{(t)}}(s) \left\|\pi_i^{(t+1)}(\cdot\,|\,s) - \pi_i^{(t)}(\cdot\,|\,s)\right\|^2
             }
             \end{array}
    \end{equation}
    where $(a)$ is due to~\pref{lem: performance difference} and we slightly abuse the notation $i$ to represent $\argmax_i$, in $(b)$ we slightly abuse the notation $\pi_i'$ to represent $\argmax_{\pi_i'}$, in $(c)$ we choose an arbitrary $\nu\in\Delta(\calS)$ and use the following inequality: 
    \[
    \frac{d_\rho^{\pi_i',\,\pi_{-i}^{(t)}}(s)}{d_\nu^{\pi_i^{(t+1)},\,\pi_{-i}^{(t)}}(s)}
    \; \leq \; \frac{d_\rho^{\pi_i',\, \pi_{-i}^{(t)}}(s)}{(1-\gamma)\nu(s)}
    \; \leq \; \frac{\sup_\pi\|d^{\pi\,\in\,\Pi}_\rho/\nu\|_\infty}{1-\gamma}.
    \]
    We apply the Cauchy–Schwarz inequality in $(d)$, and finally we replace $i$ ( $\argmax_i$ in $(a)$) in the last square root term in $(e)$ by the sum over all players.
    
    If we proceed \pref{eq: summary bound} with  $\nu=\argmin_{\nu\,\in\,\Delta(\calS)}\max_{\pi\,\in\,\Pi}\|d^{\pi}_\rho/\nu\|_\infty$, then, 
    \[
        \begin{array}{rcl}
             & & \!\!\!\! \!\!\!\! \!\!
             \displaystyle \sum_{t \,=\, 1}^T\max_i \left(\max_{\pi_i'} V_i^{\pi_i', \, \pi_{-i}^{(t)}}(\rho) -  V_i^{\pi^{(t)}}(\rho) \right) \\
             & \overset{(a)}{\leq} & \displaystyle  \frac{\sqrt{\coeff_\rho}}{\eta(1-\gamma)^{\frac{3}{2}}} \sqrt{T} \times \sqrt{2\eta (1-\gamma)\left(\Phi^{(T+1)}(\nu) - \Phi^{(1)}(\nu)\right) + \frac{4 \eta^3 A^2 N^2 }{(1-\gamma)^4  }T}
             \\[0.4cm]
             & \overset{(b)}{\lesssim} & \displaystyle  
             \sqrt{\frac{\coeff_\rho TC_\Phi}{\eta(1-\gamma)^2} }
             \, + \,  
              \sqrt{ \frac{\coeff_\rho \eta T^2 A^2 N^2 }{(1-\gamma)^7  }}
        \end{array}
    \]  
    where in $(a)$ we apply the first bound (i) in~\pref{lem: MPG policy improvement} (with $\mu=\nu$) and use \pref{def: minimax coeff}: $\coeff_\rho=\min_{\nu\,\in\,\Delta(\calS)}\max_{\pi\,\in\,\Pi}\|d^{\pi}_\rho/\nu\|_\infty$, and in $(b)$ we use $|\Phi^{\pi}(\nu) - \Phi^{\pi'}(\nu)| \leq C_\Phi$ for any $\pi, \pi'$, and further simplify the bound in $(b)$.
    We complete the proof for the first bound by taking stepsize $\eta = \frac{(1-\gamma)^{2.5}\sqrt{C_\Phi}}{NA\sqrt{T}}$ (by the upper bound of $C_\Phi$ given in \pref{lem: bounded phimax}, the condition $\eta\leq \frac{1-\gamma}{\sqrt{A}}$ is satisfied). 
    
    If we proceed \pref{eq: summary bound} with the second bound (ii) in \pref{lem: MPG policy improvement} with the choice of $\eta\leq \frac{(1-\gamma)^4}{8\kappa_{\nu}^3 NA}$, then, 
    \[
        \begin{array}{rcl}
             & & \!\!\!\! \!\!\!\! \!\!
             \displaystyle \sum_{t \,=\, 1}^T\max_{i} \left(\max_{\pi_i'} V_i^{\pi_i', \, \pi_{-i}^{(t)}}(\rho) -  V_i^{\pi^{(t)}}(\rho) \right) 
             \\
             & \leq & \displaystyle  \frac{\sqrt{\sup_{\pi\,\in\,\Pi}\|d_\rho^\pi/\nu\|_\infty}}{\eta(1-\gamma)^{\frac{3}{2}}} \sqrt{T} \times \sqrt{4\eta (1-\gamma)\left(\Phi^{(T+1)}(\nu) - \Phi^{(1)}(\nu)\right) }
             \\[0.4cm]
             & \lesssim & \displaystyle  
             \sqrt{\frac{\sup_{\pi\,\in\,\Pi}\|d_\rho^\pi/\nu\|_\infty TC_\Phi}{\eta(1-\gamma)^2} }. 
        \end{array}
    \] 
    
    We next discuss two special choices of $\nu$ for proving our bound. First, if $\nu=\rho$, then $\eta\leq \frac{(1-\gamma)^4}{8\kappa_{\rho}^3 NA}$. By letting $\eta= \frac{(1-\gamma)^4}{8\kappa_{\rho}^3 NA}$, the last square root term can be bounded by $O\left(\sqrt{\frac{\kappa_\rho^4 NA TC_\Phi}{(1-\gamma)^6} }\right)$. Second, if $\nu=\frac{1}{S}\mathbf{1}$, the uniform distribution over $\calS$, then $\kappa_\nu\leq \frac{1}{S}$, which allows a valid choice $\eta=\frac{(1-\gamma)^4}{8S^3NA}\leq \frac{(1-\gamma)^4}{8\kappa_\nu^3NA}$. Hence, we can bound the last square root term by $O\left(\sqrt{\frac{S^4 NA TC_\Phi}{(1-\gamma)^6} }\right)$. Since $\nu$ is arbitrary,  combining these two special choices completes the proof. 
\end{proof}

\subsection{Proof of~\pref{thm: convergence PMA cooperative}}\label{ap: convergence PMA cooperative}

We first establish policy improvement regarding the $Q$-function at two consecutive policies $\pi^{(t+1)}$ and $\pi^{(t)}$ in~\pref{alg: PMA}.

\begin{lemma}[Policy improvement: Markov cooperative games]\label{lem: MCG policy improvement}
   For MPG~\pref{eq: Markov potential game} with identical rewards and an initial state distribution $\rho>0$, if all players independently perform the policy update in~\pref{alg: PMA} with stepsize $\eta\leq \frac{1-\gamma}{2N}$, then for any $t$ and any $s$, \[
    \E_{a\sim \pi^{(t+1)}(\cdot\,|\,s)} \left[Q^{(t)}(s,a)\right] 
    \, - \, 
    \E_{a\sim \pi^{(t)}(\cdot\,|\,s)} \left[Q^{(t)}(s,a)\right]
    \; \geq \;
    \frac{1}{4\eta}\sum_{i\,=\,1}^N \norm{\pi_i^{(t+1)}(\cdot|s) - \pi_i^{(t)}(\cdot|s)}^2 
    \]
    where $\eta$ is the stepsize and $N$ is the number of players.
\end{lemma}
\begin{proof}[\pfref{lem: MCG policy improvement}]
     Fixing the time $t$ and the state $s$, we apply \pref{lem: decomposition lemma} to 
     \begin{align*}
         \Psi^{\pi} \;=\; \E_{a\,\sim\, \pi(\cdot\,|\,s)}\left[Q^{(t)}(s,a)\right]
     \end{align*}
     where $Q^{(t)} \DefinedAs Q^{\pi^{(t)}}$ (recall that $\pi$ is a joint policy of all players). By \pref{lem: decomposition lemma}, for any two policies $\pi'$ and $\pi$, 
     \begin{equation}\label{eq: Q difference decomp}
     \begin{array}{rcl}
        && \!\!\!\!  \!\!\!\!  \!\!
        \displaystyle
        \E_{a \,\sim\, \pi'(\cdot\,|\,s)} \left[Q^{(t)}(s,a)\right] 
        \, - \, 
        \E_{a \,\sim\, \pi(\cdot\,|\,s)} \left[Q^{(t)}(s,a)\right] 
        \\[0.2cm]
        & = & \displaystyle
        \sum_{i \, = \, 1}^N 
        \left(
        \E_{a_i \,\sim\, \pi_i'(\cdot\,|\,s),\, a_{-i} \,\sim\, \pi_{-i}(\cdot\,|\,s)}
        \left[Q^{(t)}(s,a)\right] 
        - 
        \E_{a\,\sim\, \pi(\cdot\,|\,s)}\left[Q^{(t)}(s,a)\right]
        \right) 
        \\[0.2cm]
        && \displaystyle 
        +\, \sum_{i\,=\,1}^N \sum_{j\,=\,i+1}^N \Bigg( 
        \E_{a_i\,\sim\,\pi'_i(\cdot\,|\,s),\, a_j\,\sim\,\pi_j'(\cdot\,|\,s),\, a_{-ij}\,\sim\,\tilde{\pi}_{-ij}(\cdot\,|\,s)} \left[Q^{(t)}(s,a)\right] 
        \\[0.2cm]
        && \;\;\;\; \;\;\;\; \;\;\;\; \;\;\;\; \;\;\;\; \;\;\;\; \;\;\;\; 
        \displaystyle
        -\, \E_{a_i \,\sim\, \pi_i(\cdot\,|\,s),\, a_j\,\sim\, \pi_j'(\cdot\,|\,s),\, a_{-ij}\,\sim\, \tilde{\pi}_{-ij}(\cdot\,|\,s)} \left[Q^{(t)}(s,a)\right] 
        \\[0.2cm]
        && \;\;\;\; \;\;\;\; \;\;\;\; \;\;\;\; \;\;\;\; \;\;\;\; \;\;\;\; 
        \displaystyle
        -\, \E_{a_i \,\sim\, \pi'_i(\cdot\,|\,s),\, a_j\,\sim\, \pi_j(\cdot\,|\,s),\, a_{-ij}\,\sim\, \tilde{\pi}_{-ij}(\cdot\,|\,s)} \left[Q^{(t)}(s,a)\right]
        \\[0.2cm]
        && \;\;\;\; \;\;\;\; \;\;\;\; \;\;\;\; \;\;\;\; \;\;\;\; \;\;\;\; 
        \displaystyle
        +\, \E_{a_i \,\sim\, \pi_i(\cdot\,|\,s),\, a_j \,\sim\, \pi_j(\cdot\,|\,s),\, a_{-ij}\,\sim \,\tilde{\pi}_{-ij}(\cdot\,|\,s)} \left[Q^{(t)}(s,a)\right]\Bigg)
        \end{array}
     \end{equation}
     where $\tilde{\pi}_{-ij}$ is a joint policy of players $N\backslash\{i,j\}$ in which players $<i$ and $i\sim j$ use $\pi$, and players $>j$ use $\pi'$. Particularly, we choose $\pi' = \pi^{(t+1)}$ and $\pi = \pi^{(t)}$. Thus, we can reduce~\pref{eq: Q difference decomp} into
     \[
     \begin{array}{rcl}
        && \!\!\!\!  \!\!\!\!  \!\!
        \displaystyle
        \E_{a \,\sim\, \pi'(\cdot\,|\,s)} \left[Q^{(t)}(s,a)\right] 
        \, - \, 
        \E_{a \,\sim\, \pi(\cdot\,|\,s)} \left[Q^{(t)}(s,a)\right] 
        \\[0.2cm]
        & = & \displaystyle
        \sum_{i \,=\, 1}^N \sum_{a_i}\left(\pi'_i(a_i \,|\, s) - \pi_i(a_i \,|\, s) \right)\bar{Q}_i^{(t)}(s,a_i) 
        \\[0.2cm]
        && \displaystyle 
        +\, \sum_{i \,=\, 1}^N \sum_{j \,=\, i+1}^N \sum_{a_i,\, a_j}\left( \pi'_i(a_i \,|\, s) - \pi_i(a_i \,|\, s) \right)\left( \pi'_j(a_j \,|\, s) - \pi_j(a_j \,|\, s) \right)\E_{a_{-ij}\,\sim\, \tilde{\pi}_{-ij}(\cdot\,|\,s)}\left[Q^{(t)}(s,a)\right]
        \\[0.2cm]
        & \overset{(a)}{\geq} & \displaystyle 
        \sum_{i \,= \,1}^N \frac{1}{2\eta} \norm{\pi_i'(\cdot\,|\,s) - \pi_i(\cdot\,|\,s)}^2 - \frac{1}{1-\gamma}\sum_{i\,=\,1}^N \sum_{j\,=\,i+1}^N \sum_{a_i,\, a_j}  \left| \pi'_i(a_i\,|\,s) - \pi_i(a_i\,|\,s) \right|\left| \pi'_j(a_j\,|\,s) - \pi_j(a_j\,|\,s) \right| 
        \\[0.2cm]
        & \overset{(b)}{\geq} & \displaystyle
        \sum_{i\,=\,1}^N \frac{1}{2\eta} \norm{\pi_i'(\cdot\,|\,s) - \pi_i(\cdot\,|\,s)}^2 - \frac{A}{2(1-\gamma)}\sum_{i\,=\,1}^N \sum_{j\,=\,i+1}^N \left( \norm{\pi_i'(\cdot\,|\,s) - \pi_i(\cdot\,|\,s)}^2 + \norm{ \pi'_j(\cdot\,|\,s) - \pi_j(\cdot\,|\,s) }^2 \right) 
        \\[0.2cm]
        &=& \displaystyle
        \sum_{i=1}^N \frac{1}{2\eta} \norm{\pi_i'(\cdot\,|\,s) - \pi_i(\cdot\,|\,s)}^2 - \frac{(N-1)A}{2(1-\gamma)}\sum_{i=1}^N  \norm{\pi_i'(\cdot\,|\,s) - \pi_i(\cdot\,|\,s)}^2
        \\[0.2cm]
         & \overset{(c)}{\geq} & \displaystyle 
         \sum_{i\,=\,1}^N \frac{1}{4\eta}\norm{\pi_i'(\cdot\,|\,s) - \pi_i(\cdot\,|\,s)}^2
    \end{array}
     \]
     where $(a)$ is due to the optimality  condition~\pref{eq: optimality condition} and $Q^{(t)}(s,a)\leq\frac{1}{1-\gamma}$, $(b)$ is due to $\langle x, y \rangle \leq \frac{\norm{x}^2+\norm{y}^2}{2}$, and $(c)$ follows the choice of $\eta\leq \frac{1-\gamma}{2NA}$. 
\end{proof}

\begin{proof}[Proof of~\pref{thm: convergence PMA cooperative}]
    By \pref{lem: performance difference} and \pref{lem: MCG policy improvement}, we have for any $\nu\in\Delta(\calS)$, 
    \begin{equation}\label{eq: policy improvement V}
    \begin{array}{rcl}
        V^{(t+1)}(\nu) 
        \,-\, 
        V^{(t)}(\nu) 
        &=& \displaystyle \frac{1}{1-\gamma}
        \sum_{s,\, a} d_\nu^{\pi^{(t+1)}}(s)\left(\pi^{(t+1)}(a\,|\,s) - \pi^{(t)}(a\,|\,s)\right)Q^{(t)}(s,a) 
        \\[0.2cm]
        &\geq& \displaystyle
        \frac{1}{4\eta(1-\gamma)}\sum_{i\,=\,1}^N \sum_s d_\nu^{\pi^{(t+1)}}(s) \norm{\pi_i^{(t+1)}(\cdot\,|\,s) - \pi_i^{(t)}(\cdot\,|\,s)}^2. 
    \end{array}
    \end{equation}
    By the same argument as the proof of \pref{thm: convergence PMA},
    \[
        \begin{array}{rcl}
             & & \!\!\!\! \!\!\!\! \!\!
             \displaystyle
             \sum_{t \,=\, 1}^T\max_{i} \left(\max_{\pi_i'} V^{\pi_i', \, \pi_{-i}^{(t)}}(\rho) -  V^{\pi^{(t)}}(\rho) \right) \\[0.4cm] 
             & \overset{(a)}{\leq} & \displaystyle
             \frac{3}{\eta(1-\gamma)} \sum_{t \,=\, 1}^T  \sum_{s} d_\rho^{\pi_i', \pi_{-i}^{(t)}}(s) \left\|\pi_i^{(t+1)}(\cdot\,|\,s) - \pi_i^{(t)}(\cdot\,|\,s)\right\| 
             \\[0.4cm]
             & \overset{(b)}{\lesssim} & \displaystyle
             \frac{\sqrt{\coeff_\rho}}{\eta(1-\gamma)^{\frac{3}{2}}} \sum_{t \,=\, 1}^T  \sum_{s} \sqrt{d_\rho^{\pi_i', \pi_{-i}^{(t)}}(s) \times d_\nu^{\pi^{(t+1)}}(s)} \left\|\pi_i^{(t+1)}(\cdot\,|\,s) - \pi_i^{(t)}(\cdot\,|\,s)\right\| 
             \\[0.5cm] 
             & \leq & \displaystyle
             \frac{\sqrt{\coeff_\rho}}{\eta(1-\gamma)^{\frac{3}{2}}} \sqrt{
             \sum_{t \,=\, 1}^T \sum_{s} d_\rho^{\pi_i', \pi_{-i}^{(t)}}(s)
             }
             \times\,
             \sqrt{
             \sum_{t \,=\, 1}^T \sum_{s} d_\nu^{\pi^{(t+1)}}(s) \left\|\pi_i^{(t+1)}(\cdot\,|\,s) - \pi_i^{(t)}(\cdot\,|\,s)\right\|^2
             }
             \\[0.5cm] 
             & \overset{(c)}{\leq} & \displaystyle
             \frac{\sqrt{\coeff_\rho}}{\eta(1-\gamma)^{\frac{3}{2}}} \sqrt{
             \sum_{t \,=\, 1}^T \sum_{s} d_\rho^{\pi_i', \pi_{-i}^{(t)}}(s)
             }
             \times\,
             \sqrt{
             \sum_{t \,=\, 1}^T\sum_{i\,=\,1}^N \sum_{s} d_\nu^{\pi^{(t+1)}}(s) \left\|\pi_i^{(t+1)}(\cdot\,|\,s) - \pi_i^{(t)}(\cdot\,|\,s)\right\|^2
             }
             \\[0.6cm]
             & \overset{(d)}{\leq} & \displaystyle  \frac{\sqrt{\coeff_\rho}}{\eta(1-\gamma)^{\frac{3}{2}}} \sqrt{T} \times \sqrt{4\eta (1-\gamma) \left(V^{(T+1)}(\nu) - V^{(1)}(\nu)\right) }
        \end{array}
    \]
    where in $(a)$ we slightly abuse the notation $i$ to represent $\argmax_i$ as in \pref{eq: summary bound}, in $(b)$ we take $\nu=\argmin_{\nu\in\Delta(\calS)}\max_{\pi\,\in\,\Pi}\|d_\rho^\pi/\nu\|_\infty$ and use the definition of $\coeff_\rho$ from \pref{def: minimax coeff}, and we replace $i$ ($\argmax_i$ in $(a)$) in the last square root term in $(c)$ by the sum over all players, and we apply \pref{eq: policy improvement V} in $(d)$.
    
    Finally, we complete the proof by taking stepsize $\eta = \frac{1-\gamma}{2NA}$ and using $V^{(T+1)}(\nu) - V^{(1)}(\nu)\leq \frac{1}{1-\gamma}$.
\end{proof}

\section{Proofs for Section~\ref{sec: function approximation}}
\label{ap: approximation}

In this section, we provide proofs of \pref{thm: sample PMA potential} and \pref{thm: sample PMA cooperative} in \pref{ap: convergence PMA sample-based} and \pref{ap: convergence PMA cooperative sample-based}, respectively.

\subsection{Unbiased estimate}\label{ap: unbiased estimate}

We consider the $k$th sampling in the data collection phase of \pref{alg: PMA fa}. By the sampling model in lines 6-8 of \pref{alg: PMA fa}, it is straightforward to see that $\bar s^{(h_i)} \sim d_\rho^{\pi^{(t)}}$ for player $i$. Then, we take $\bar a_i^{(h_i)}\sim\pi_i^{(k)}(\cdot\,\vert\,s^{(h_i)})$ at step $h_i$ for player $i$. Each player $i$ begins with such $(\bar s^{(h_i)},\bar a_i^{(h_i)})$ while all players execute the policy $\{\pi_i^{(t)}\}_{i\,=\,1}^N$ with the termination probability $1-\gamma$. Once terminated, we add all rewards collected in $R_i^{(k)}$. We next show that $\mathbb{E} [\, R_i^{(k)} \,] = \bar Q_i^{\pi}(\bar s^{(h_i)},\bar a_i^{(h_i)})$,
\[
\begin{array}{rcl}
     \mathbb{E} [\, R_i^{(k)} \,] 
     & = & \displaystyle
     \mathbb{E} \left[\, \sum_{h\, = \,h_i}^{h_i+h_i'-1} \bar r_i^{(h)} \,\bigg\vert\, \bar s^{(h_i)}, \bar a_i^{(h_i)}, \bar a_{-i}^{(h_i)}\sim \pi_{-i}^{(t)}(\cdot\,\vert\,\bar s^{(h_i)}), h_i'\sim \textsc{Geometric}(1-\gamma) \,\right]
     \\[0.5cm]
     & \overset{(a)}{=} & \displaystyle
     \mathbb{E} \left[\, \sum_{h\, = \,0}^{h_i'-1} \bar r_i^{(h+h_i)} \,\bigg\vert\, \bar s^{(h_i)}, \bar a_i^{(h_i)}, \bar a_{-i}^{(h_i)}\sim \pi_{-i}^{(t)}(\cdot\,\vert\,\bar s^{(h_i)}), h_i'\sim \textsc{Geometric}(1-\gamma) \,\right]
     \\[0.5cm]
     & = & \displaystyle
     \mathbb{E} \left[\, \sum_{h\, = \,0}^\infty \one_{\{0\,\leq\, h \,\leq h_i'-1\}} \bar r_i^{(h+h_i)} \,\bigg\vert\, \bar s^{(h_i)}, \bar a_i^{(h_i)}, \bar a_{-i}^{(h_i)}\sim \pi_{-i}^{(t)}(\cdot\,\vert\,\bar s^{(h_i)}), h_i'\sim \textsc{Geometric}(1-\gamma) \,\right]
     \\[0.4cm]
     & \overset{(b)}{=} & \displaystyle
     \sum_{h\, = \,0}^\infty \mathbb{E} \left[\, \mathbb{E}_{h_i'} \big[\one_{\{0\,\leq\, h \,\leq h_i'-1\}} \big] \, \bar r_i^{(h+h_i)} \,\bigg\vert\, \bar s^{(h_i)}, \bar a_i^{(h_i)}, \bar a_{-i}^{(h_i)}\sim \pi_{-i}^{(t)}(\cdot\,\vert\,\bar s^{(h_i)}) \,\right]
     \\[0.4cm]
     & \overset{(c)}{=} & \displaystyle
     \sum_{h\, = \,0}^\infty \mathbb{E} \left[\, \gamma^h \, \bar r_i^{(h+h_i)} \,\bigg\vert\, \bar s^{(h_i)}, \bar a_i^{(h_i)}, \bar a_{-i}^{(h_i)}\sim \pi_{-i}^{(t)}(\cdot\,\vert\,\bar s^{(h_i)}) \,\right]
     \\[0.4cm]
     & = & \displaystyle \mathbb{E}_{\bar a_{-i}^{(h_i)}\,\sim\, \pi_{-i}^{(t)}(\cdot\,\vert\,\bar s^{(h_i)})} \mathbb{E} \left[\,
     \sum_{h\, = \,0}^\infty  \gamma^h \, \bar r_i^{(h+h_i)} \,\bigg\vert\, \bar s^{(h_i)}, \bar a_i^{(h_i)}, \bar a_{-i}^{(h_i)}\sim \pi_{-i}^{(t)}(\cdot\,\vert\,\bar s^{(h_i)}) \,\right]
     \\[0.4cm]
     & = & \displaystyle \mathbb{E}_{\bar a_{-i}^{(h_i)}\,\sim\, \pi_{-i}^{(t)}(\cdot\,\vert\,\bar s^{(h_i)})}\left[ Q_i^{(t)}(\bar s^{(h_i)},\bar a_i^{(h_i)},\bar a_{-i}^{(h_i)})\right]
     \\[0.2cm]
     & = & \displaystyle \bar Q_i^{(t)}(\bar s^{(h_i)},\bar a_i^{(h_i)})
\end{array}
\]
where in $(a)$ we change the range of index $h$ while using the same initial state and action, $(b)$ is due to the tower property, $(c)$ follows that $\mathbb{E}_{h_i'} \big[\one_{\{0\,\leq\, h \,\leq h_i'-1\}}\big] =1-(1-(1-p)^{h}) = \gamma^{h}$, where $p=1-\gamma$, and we also apply the monotone convergence and dominated convergence theorems for swapping the sum and the expectation. 

\subsection{Proof of \pref{thm: sample PMA potential}}\label{ap: convergence PMA sample-based}

We apply~\pref{lem: decomposition lemma} to the potential function $\Phi^\pi(\rho)$ at two consecutive policies $\pi^{(t+1)}$ and $\pi^{(t)}$ in~\pref{alg: PMA fa}, where $\rho$ is the initial state distribution. We use the shorthand $\Phi^{(t)}(\rho)$ for $\Phi^{\pi^{(t)}}(\rho)$, the value of potential function at policy $\pi^{(t)}$. The proof extends~\pref{lem: MPG policy improvement} by accounting for the statistical error in \pref{as: bounded error}.

\begin{lemma}[Policy improvement: Markov potential games]
\label{lem: MPG policy improvement sample-based}
    Let \pref{as: linear Q} hold.
    In~\pref{alg: PMA fa}, the potential function $\Phi^\pi(\rho)$ at two consecutive policies $\pi^{(t+1)}$ and $\pi^{(t)}$ satisfies
    \[
    \begin{array}{rcl}
         \text{\normalfont (i)} \;\; \displaystyle
         \Phi^{(t+1)}(\rho) - \Phi^{(t)}(\rho)
         & \geq & \displaystyle
        \frac{1}{4\eta(1-\gamma)}\sum_{i \,=\, 1}^N \sum_{s}d_\rho^{\pi_i^{(t+1)}, \pi^{(t)}_{-i}}(s)\left\|\pi_i^{(t+1)}(\cdot\,|\,s)-\pi_i^{(t)}(\cdot\,|\,s)\right\|^2  -\, \frac{4 \eta^2 A W^2 N^2}{(1-\gamma)^3} 
        \\[0.2cm]
        && \displaystyle
        - 
        \frac{\eta \kappa_\rho A }{(1-\gamma)^2\xi} \sum_{i \,=\, 1}^N { L_i^{(t)} (\hat w_i^{(t)})}
        \\
        \text{\normalfont (ii)} \;\; \displaystyle
         \Phi^{(t+1)}(\rho) - \Phi^{(t)}(\rho)
         & \geq & \displaystyle
        \frac{1}{4\eta(1-\gamma)}\sum_{i \,=\, 1}^N \sum_{s}d_\rho^{\pi_i^{(t+1)}, \pi^{(t)}_{-i}}(s)
        \left(1-{\frac{4\eta\kappa_\rho^3  NA}{(1-\gamma)^4}}\right)
        \left\|\pi_i^{(t+1)}(\cdot\,|\,s)-\pi_i^{(t)}(\cdot\,|\,s)\right\|^2 
        \\[0.2cm]
        && \displaystyle
        -\, 
        \frac{\eta \kappa_\rho A }{(1-\gamma)^2\xi} \sum_{i \,=\, 1}^N { L_i^{(t)} (\hat w_i^{(t)})}
    \end{array}
    \]
    where $\eta$ is the stepsize, $N$ is the number of players, $A$ is the size of one player's action space, $W$ is the 2-norm bound of $\hat w_i^{(t)}$, and $\kappa_\rho$ is the distribution mismatch coefficient relative to $\rho$ (see $\kappa_\rho$ in \pref{def: distribution mismatch coeff}).
\end{lemma}

\begin{proof}[\pfref{lem: MPG policy improvement sample-based}]
    We let $\pi'=\pi^{(t+1)}$ and $\pi = \pi^{(t)}$ for brevity. We first express
    $\Phi^{(t+1)}(\rho) - \Phi^{(t)}(\rho) = \textbf{Diff}_\alpha + \textbf{Diff}_\beta$, where $\textbf{Diff}_\alpha$ and $\textbf{Diff}_\beta$ are given as those in~\pref{eq: DiffAB}.

\noindent\textbf{Bounding} $\textbf{Diff}_\alpha$. By the property of the potential function $\Phi^\pi(\rho)$ and~\pref{lem: performance difference},
\[
\begin{array}{rcl}
    \Phi^{\pi_i', \, \pi_{-i}}(\rho) 
    \,-\,
    \Phi^\pi(\rho) 
    & = & 
    V_i^{\pi_i',\, \pi_{-i}}(\rho) \,-\,
    V_i^\pi(\rho) 
    \\[0.2cm]
    & = & \displaystyle
    \frac{1}{1-\gamma}\sum_{s, \, a_i} d_\rho^{\pi_i',\, \pi_{-i}}(s) \left(\pi_i'(a_i\,|\,s) - \pi_i(a_i\,|\,s)\right) \bar Q_i^{\pi_i,\,\pi_{-i}} (s,a_i). 
\end{array}
\]
The optimality of $\pi_i'=\pi_i^{(t+1)}$ in line~14 of~\pref{alg: PMA fa} leads to 
\begin{equation}\label{eq: optimality condition sample-based}
        \big\langle\pi_i'(\cdot\,\vert\,s), \hat{Q}_i^{(t)}(s,\cdot)\big\rangle_{\calA_i}  
        \, - \, \frac{1}{2\eta}\big\|\pi_i'(\cdot\,\vert\,s) - \pi_i(\cdot\,\vert\,s)\big\|^2 
        \; \geq \;
        \big\langle\pi_i(\cdot\,\vert\,s), \hat{Q}_i^{(t)}(s,\cdot)\big\rangle_{\calA_i}.
\end{equation}
Hence, 
\[
    \begin{array}{rcl}
         \Phi^{\pi_i', \, \pi_{-i}}(\rho) 
         \, - \,
         \Phi^\pi(\rho) 
         & \geq & \displaystyle
        \frac{1}{2\eta(1-\gamma)}\sum_{s}d_\rho^{\pi_i',\, \pi_{-i}}(s)\left\|\pi_i'(\cdot\,|\,s)-\pi_i(\cdot\,|\,s)\right\|^2
        \\[0.2cm]
        && \displaystyle
        +\, \frac{1}{1-\gamma} \sum_{s} d_\rho^{\pi_i',\, \pi_{-i}}(s) \big\langle\pi_i'(\cdot\,\vert\,s)-\pi_i(\cdot\,\vert\,s), \bar Q_i^{\pi_i,\pi_{-i}}(s,\cdot)-\hat{Q}_i^{(t)}(s,\cdot)\big\rangle_{\calA_i}.
    \end{array}
\]
Therefore,
\[
    \begin{array}{rcl}
        \textbf{Diff}_\alpha
        & \geq & 
        \displaystyle
        \frac{1}{2\eta(1-\gamma)}\sum_{i\,=\,1}^N\sum_{s}d_\rho^{\pi_i^{(t+1)},\, \pi_{-i}^{(t)}}(s)\left\|\pi_i^{(t+1)}(\cdot\,|\,s)-\pi_i^{(t)}(\cdot\,|\,s)\right\|^2
        \\[0.2cm]
        && \displaystyle
        +\, \frac{1}{1-\gamma} \sum_{i\,=\,1}^N \sum_{s} d_\rho^{\pi_i^{(t+1)},\, \pi_{-i}^{(t)}}(s) \big\langle \big(\pi_i^{(t+1)}-\pi_i^{(t)}\big)(\cdot\,\vert\,s), \bar Q_i^{(t)}(s,\cdot)-\hat{ Q}_i^{(t)}(s,\cdot)\big\rangle_{\calA_i}.
    \end{array} 
\]
However, 
\[
\begin{array}{rcl}
    && \displaystyle
    \!\!\!\! \!\!\!\! \!\!
    \sum_{s} d_\rho^{\pi_i^{(t+1)},\, \pi_{-i}^{(t)}}(s) \big\langle \big(\pi_i^{(t+1)}-\pi_i^{(t)}\big)(\cdot\,\vert\,s), \bar Q_i^{(t)}(s,\cdot)-\hat{ Q}_i^{(t)}(s,\cdot)\big\rangle_{\calA_i}
    \\[0.2cm]
    & \overset{(a)}{\geq} & \displaystyle
    - \, \sum_{s} d_\rho^{\pi_i^{(t+1)},\, \pi_{-i}^{(t)}}(s)  
    \left( 
    \frac{1}{2\eta'}
    \norm{
    \big(\pi_i^{(t+1)}-\pi_i^{(t)}\big)(\cdot\,\vert\,s)}^2
    \, + \, 
    \frac{\eta'}{2}
    \norm{
    \bar Q_i^{(t)}(s,\cdot)-\hat{ Q}_i^{(t)}(s,\cdot)}^2
    \right)
    \\[0.2cm]
    & \overset{(b)}{=} & \displaystyle
    -\, \frac{1}{4\eta}\sum_{s} d_\rho^{\pi_i^{(t+1)},\, \pi_{-i}^{(t)}}(s)
    \norm{
    \big(\pi_i^{(t+1)}-\pi_i^{(t)}\big)(\cdot\,\vert\,s)}^2
    \, - \,
    \eta \sum_{s} d_\rho^{\pi_i^{(t+1)},\, \pi_{-i}^{(t)}}(s)
    \norm{
    \bar Q_i^{(t)}(s,\cdot)-\hat{ Q}_i^{(t)}(s,\cdot)}^2
\end{array}
\]
where $(a)$ follows the inequality $\langle x, y\rangle \leq \frac{\norm{x}^2}{2 \eta'} +\frac{\eta'\norm{y}^2}{2}$ for $\eta'>0$, and we choose $\eta' = 2\eta $ in $(b)$.

Therefore, 
\begin{equation}\label{eq: DiffA sample-based}
    \begin{array}{rcl}
        \textbf{Diff}_\alpha
        & \geq & 
        \displaystyle
        \frac{1}{4\eta(1-\gamma)}\sum_{i\,=\,1}^N\sum_{s}d_\rho^{\pi_i^{(t+1)},\, \pi_{-i}^{(t)}}(s)\left\|\pi_i^{(t+1)}(\cdot\,|\,s)-\pi_i^{(t)}(\cdot\,|\,s)\right\|^2
        \\[0.2cm]
        && \displaystyle
    -\, \frac{\eta}{1-\gamma} \sum_{i\,=\,1}^N \sum_{s} d_\rho^{\pi_i^{(t+1)},\, \pi_{-i}^{(t)}}(s) 
    \norm{
    \bar Q_i^{(t)}(s,\cdot)-\hat{ Q}_i^{(t)}(s,\cdot)
    }^2.
    \end{array} 
\end{equation}

\noindent\textbf{Bounding} $\textbf{Diff}_\beta$. For simplicity, we denote $\tilde{\pi}_{-ij}$ as the joint policy of players $N\backslash \{i,j\}$ where players $<i$ and $i\sim j$ use $\pi$ and players $>j$ use $\pi'$. As done in the proof of~\pref{lem: MPG policy improvement}, we can bound each summand in $\text{Diff}_\beta$ except for the last step from $(c)$ to $(d)$,
\[
\begin{array}{rcl}
     && \!\!\!\! \!\!\!\! \!\!\!\! 
     \Phi^{\tilde{\pi}_{-ij}, \,\pi_i',\, \pi_j'}(\rho) 
     \, - \,
     \Phi^{\tilde{\pi}_{-ij}, \,\pi_i,\, \pi_j'}(\rho)
    \,-\, 
    \Phi^{\tilde{\pi}_{-ij}, \,\pi_i',\, \pi_j}(\rho) 
    \, + \, \Phi^{\tilde{\pi}_{-ij},\, \pi_i,\, \pi_j}(\rho) 
     \\[0.4cm]    
     &  \overset{(c)}{\geq}  & \displaystyle
     -\,
     \frac{1}{(1-\gamma)^3} \left(\max_s  \norm{ \pi_i'(\cdot\,|\,s) - \pi_i(\cdot\,|\,s) }_1\right) \left(\max_s  \norm{ \pi_j'(\cdot\,|\,s) - \pi_j(\cdot\,|\,s) }_1\right) \\[0.4cm]  & & \displaystyle -\, \frac{1}{(1-\gamma)^2}\left(\max_s  \norm{ \pi_j'(\cdot\,|\,s) - \pi_j(\cdot\,|\,s) }_1\right) \left(\max_s  \norm{ \pi_i'(\cdot\,|\,s) - \pi_i(\cdot\,|\,s) }_1\right) \\[0.4cm]
     &  \overset{(d)}{\geq}  & \displaystyle
     - \,\frac{8 \eta^2 A W^2}{(1-\gamma)^3}
\end{array}
\]
where $(d)$ follows a direct result from the optimality of $\pi_j^{(t+1)}$ given by~\pref{eq: optimality condition sample-based},
    \[
    \norm{ \pi_j^{(t+1)}(\cdot\,\vert\,s) - \pi_j^{(t)}(\cdot\,\vert\,s) }
    \;\leq\; 2\eta\norm{ \hat{ Q}_i^{(t)}(s,\cdot)}
    \;\leq\;
    2\eta \,W
    \]
and that $\|\cdot\|_1\leq \sqrt{A}\|\cdot\|$. 
Therefore,
\begin{equation}\label{eq: DiffB sample-based}
    \textbf{Diff}_\beta 
    \; \geq \; 
    -\,\frac{ 4 \eta^2 A W^2 N^2}{(1-\gamma)^3}.
\end{equation}

We now complete the proof of (i) by combining~\pref{eq: DiffA sample-based} and~\pref{eq: DiffB sample-based} and we also employ that 
\[
\begin{array}{rcl}
    && \displaystyle
    \!\!\!\! \!\!\!\! \!\!
    \sum_{s} d_\rho^{\pi_i^{(t+1)},\, \pi_{-i}^{(t)}}(s)  
    \norm{
    \bar Q_i^{(t)}(s,\cdot)-\hat{Q}_i^{(t)}(s,\cdot)
    }^2
    \\[0.2cm]
    & \overset{(a)}{\leq} & \displaystyle
    \frac{\kappa_\rho}{1-\gamma} \sum_{s} d_\rho^{\pi_i^{(t)},\, \pi_{-i}^{(t)}}(s) 
    \norm{
    \bar Q_i^{(t)}(s,\cdot)-\hat{Q}_i^{(t)}(s,\cdot)
    }^2
    \\[0.2cm]
    & \overset{(b)}{\leq } & \displaystyle
    \frac{\kappa_\rho A}{(1-\gamma)\xi} {  L_i^{(t)} (\hat w_i^{(t)})}
\end{array}
\]
where $(a)$ follows the definition of $\kappa_\rho$ and $(b)$ is the definition of $L_i^{(t)} (\hat w_i^{(t)})$: 
\begin{align*}
    L_i^{(t)} (\hat w_i^{(t)})
    \; \DefinedAs \;
    \mathbb{E}_{s \sim\,d_\rho^{(t)}, a_i\sim \pi_i^{(t)}(\cdot|s)} \left[
    \big(
    \bar Q_i^{(t)}(s,a_i) -  \hat{ Q}_i^{(t)}(s,a_i)
    \big)^2
    \right] 
    \; \geq \;
    \frac{\xi}{A} \E_{s\sim d_\rho^{(t)}} \sum_{a_i} \big(
    \bar Q_i^{(t)}(s,a_i) -  \hat{ Q}_i^{(t)}(s,a_i)
    \big)^2. 
\end{align*}

Alternatively, as done in \pref{lem: MPG policy improvement}, we can apply \pref{lem: second-order PDL} to each summand of $\textbf{Diff}_\beta$ and show that  
\[
    \textbf{Diff}_\beta 
    \;\geq\;
    -\, \frac{2\kappa_\rho^3  NA}{(1-\gamma)^5} \sum_{i\,=\,1}^N \sum_s d_\rho^{\pi_i^{(t+1)},\, \pi_{-i}^{(t)}}(s)\norm{\pi_i^{(t)}(\cdot\,|\,s) - \pi_i^{(t+1)}(\cdot\,|\,s)}^2 . 
\]
Combining the inequality above with~\pref{eq: DiffA sample-based} finishes the proof of (ii). 
\end{proof}

\begin{proof}[\pfref{thm: sample PMA potential}]
    By the optimality of $\pi_i^{(t+1)}$ in line~14 of~\pref{alg: PMA fa}, 
    \[
        \Big\langle
        (1-\xi)\pi_i'(\cdot\,|\,s) + \frac{\xi}{A}\one - \pi_i^{(t+1)}(\cdot\,|\,s)
        , \, 
        \eta \hat { Q}_i^{(t)}(s,\cdot) - \pi^{(t+1)}_i(\cdot\,|\,s) + \pi^{(t)}_i(\cdot\,|\,s) 
        \Big\rangle_{\calA_i}
        \; \leq \;
        0,\;
         \text{ for any } \pi_i' \,\in\, \Pi_i.     
    \]
    which leads to 
    \begin{align}
        &\Big\langle
        \pi_i'(\cdot\,|\,s) - \pi_i^{(t+1)}(\cdot\,|\,s)
        ,\,
        \eta\hat{Q}_i^{(t)}(s,\cdot)
        \Big\rangle_{\calA_i}   \nonumber \\
        &\leq\; \Big\langle
        \pi_i'(\cdot\,|\,s) - \pi_i^{(t+1)}(\cdot\,|\,s)
        ,\,
        \pi_i^{(t+1)}(\cdot|s)-\pi_i^{(t)}(\cdot|s)
        \Big\rangle_{\calA_i} 
        \\
        & \qquad+\, \frac{\xi}{1-\xi} \Big\langle \pi^{(t+1)}(\cdot|s) - \frac{1}{A}\one, \eta\hat{Q}_i^{(t)}(s,\cdot)  - \pi^{(t+1)}_i(\cdot\,|\,s) + \pi^{(t)}_i(\cdot\,|\,s)   \Big\rangle\nonumber  \\
        &\lesssim\; \norm{\pi_i^{(t+1)}(\cdot|s)-\pi_i^{(t)}(\cdot|s)} \,+\, \eta\xi W   \label{eq: temp eq}
    \end{align}
    where the last inequality is because of $\Vert{\hat{ Q}_i^{(t)}(s,\cdot)}\Vert \leq W$ and $\xi\leq \frac{1}{2}$.  
    Hence, if $\eta\leq \frac{1}{W}$, then for any $\pi_i'\in\Pi_i$,
    \[
    \begin{array}{rcl}
        & & \!\!\!\! \!\!\!\! \!\! 
        \displaystyle
        \big\langle
        \pi_i'(\cdot\,|\,s) - \pi_i^{(t)}(\cdot\,|\,s)
        ,\,
        \bar Q_i^{(t)}(s,\cdot)
        \big\rangle_{\calA_i}
        \\[0.2cm]
        & = & \displaystyle
        \big\langle
        \pi_i'(\cdot\,|\,s) - \pi_i^{(t+1)}(\cdot\,|\,s)
        ,\,
        \hat{Q}_i^{(t)}(s,\cdot)
        \big\rangle_{\calA_i}
        \,+\,
        \big\langle
        \pi_i^{(t+1)}(\cdot\,|\,s) - \pi_i^{(t)}(\cdot\,|\,s)
        ,\,
        \hat{ Q}_i^{(t)}(s,\cdot)
        \big\rangle_{\calA_i}
        \\[0.2cm]
        &  & \displaystyle
        +\, \big\langle
        \pi_i'(\cdot\,|\,s) - \pi_i^{(t)}(\cdot\,|\,s)
        ,\,
        \bar Q_i^{(t)}(s,\cdot) - \hat{Q}_i^{(t)}(s,\cdot)
        \big\rangle_{\calA_i}
        \\[0.2cm]
        & \overset{(a)}{\lesssim} & \displaystyle
        \frac{1}{\eta} \norm{\pi^{(t+1)}_i(\cdot\,|\,s) - \pi^{(t)}_i(\cdot\,|\,s) } \,+\,
        \xi W
        \,+\,
        \norm{
        \pi_i^{(t+1)}(\cdot\,|\,s) - \pi_i^{(t)}(\cdot\,|\,s)
        }
        \norm{
        \hat{Q}_i^{(t)}(s,\cdot)
        }
        \\[0.2cm]
        &  & \displaystyle
        +\, \big\langle
        \pi_i'(\cdot\,|\,s) - \pi_i^{(t)}(\cdot\,|\,s)
        ,\,
        \bar Q_i^{(t)}(s,\cdot) - \hat{Q}_i^{(t)}(s,\cdot)
        \big\rangle_{\calA_i}
        \\[0.2cm]
        & \overset{(b)}{\lesssim} & \displaystyle  \frac{1}{\eta}\left\|\pi_i^{(t+1)}(\cdot\,|\,s) - \pi_i^{(t)}(\cdot\,|\,s)\right\|  \,+\, \xi W
        \\[0.2cm]
        &  & \displaystyle
        +\, \big\langle
        \pi_i'(\cdot\,|\,s) - \pi_i^{(t)}(\cdot\,|\,s)
        ,\,
        \bar Q_i^{(t)}(s,\cdot) - \hat{Q}_i^{(t)}(s,\cdot)
        \big\rangle_{\calA_i} 
    \end{array}
    \]
    where we apply \pref{eq: temp eq} and the Cauchy-Schwarz inequality in~$(a)$, and $(b)$ is because $\Vert{\hat{ Q}_i^{(t)}(s,\cdot)}\Vert \leq W$ and $\eta \leq \frac{1}{W}$. As done in the proof of~\pref{thm: convergence PMA}, the different steps begin from $(b)$ in \pref{eq: summary bound},
    \begin{equation}\label{eq: summary bound sample-based}
        \begin{array}{rcl}
             & & \!\!\!\! \!\!\!\! \!\!
             \displaystyle
             \sum_{t \,=\, 1}^T\max_{i} \left(\max_{\pi_i'} V_i^{\pi_i', \, \pi_{-i}^{(t)}}(\rho) -  V_i^{\pi^{(t)}}(\rho) \right) \\[0.4cm]         
             & \overset{(b)}{\lesssim} & \displaystyle
             \frac{1}{\eta(1-\gamma)} \sum_{t \,=\, 1}^T \sum_{s} d_\rho^{\pi_i', \pi_{-i}^{(t)}}(s) \left\|\pi_i^{(t+1)}(\cdot\,|\,s) - \pi_i^{(t)}(\cdot\,|\,s)\right\| + \frac{\xi T W}{1-\gamma}
             \\[0.2cm]
             & & \displaystyle +\,\frac{1}{1-\gamma}
              \sum_{t \,=\, 1}^T \sum_{s} d_\rho^{\pi_i', \pi_{-i}^{(t)}}(s) \big\langle
            \pi_i'(\cdot\,|\,s) - \pi_i^{(t)}(\cdot\,|\,s)
            ,\,
            \bar Q_i^{(t)}(s,\cdot) - \hat{Q}_i^{(t)}(s,\cdot)
            \big\rangle_{\calA_i}
             \\[0.4cm]
             & \overset{(c)}{\lesssim} & \displaystyle
          \frac{\sqrt{\kappa_\rho}}{\eta(1-\gamma)^{\frac{3}{2}}} \sum_{t \,=\, 1}^T\sum_{s} \sqrt{ d_\rho^{\pi_i^{(t+1)}, \pi_{-i}^{(t)}}(s)\times d_\rho^{\pi_i', \pi_{-i}^{(t)}}(s)} \left\|\pi_i^{(t+1)}(\cdot\,|\,s) - \pi_i^{(t)}(\cdot\,|\,s)\right\| + \frac{\xi T W}{1-\gamma}
             \\[0.2cm]
             & & \displaystyle +\, \frac{\kappa_\rho}{1-\gamma} \left|
              \sum_{t \,=\, 1}^T \sum_{s} d_\rho^{\pi^{(t)}}(s) \big\langle
            \pi_i'(\cdot\,|\,s) - \pi_i^{(t)}(\cdot\,|\,s)
            ,\,
            \bar Q_i^{(t)}(s,\cdot) - \hat{Q}_i^{(t)}(s,\cdot)
            \big\rangle_{\calA_i}\right|
             \\[0.5cm]
             & \overset{(d)}{\leq} & \displaystyle
             \frac{\sqrt{\kappa_\rho}}{\eta(1-\gamma)^{\frac{3}{2}}} \sqrt{
             \sum_{t \,=\, 1}^T  \sum_{s} d_\rho^{\pi_i^{(t+1)}, \pi_{-i}^{(t)}}(s) 
             }
             \times\,
             \sqrt{
             \sum_{t \,=\, 1}^T \sum_{s} d_\rho^{\pi_i^{(t+1)}, \pi_{-i}^{(t)}}(s) \left\|\pi_i^{(t+1)}(\cdot\,|\,s) - \pi_i^{(t)}(\cdot\,|\,s)\right\|^2
             }
             \\[0.2cm]
             & & \displaystyle  +\, \frac{\xi T W}{1-\gamma} 
             \,+\, 
             \frac{\kappa_\rho }{1-\gamma} \sum_{t \,=\, 1}^T \sqrt{\frac{AL_i^{(t)} (\hat w_i^{(t)})}{\xi}}
             \\[0.5cm]
             & \overset{(e)}{\leq} & \displaystyle
             \frac{\sqrt{\kappa_\rho}}{\eta(1-\gamma)^{\frac{3}{2}}} \sqrt{
             \sum_{t \,=\, 1}^T  \sum_{s} d_\rho^{\pi_i^{(t+1)}, \pi_{-i}^{(t)}}(s) 
             }
             \times\,
             \sqrt{
             \sum_{t \,=\, 1}^T \sum_{i\,=\,1}^N\sum_{s} d_\rho^{\pi_i^{(t+1)}, \pi_{-i}^{(t)}}(s) \left\|\pi_i^{(t+1)}(\cdot\,|\,s) - \pi_i^{(t)}(\cdot\,|\,s)\right\|^2
             }
             \\[0.2cm]
             & & \displaystyle  +\, \frac{\xi T W}{1-\gamma} 
             \,+\, 
             \frac{\kappa_\rho }{1-\gamma} \sum_{t \,=\, 1}^T \sqrt{\frac{AL_i^{(t)} (\hat w_i^{(t)})}{\xi}}
        \end{array}
    \end{equation}
    where we slightly abuse the notation $\pi_i'$ in $(b)$ to represent $\argmax_{\pi_i'}$ and $i$ represents $\argmax_i$ as in \pref{eq: summary bound}, $(c)$ is due to the definition of the distribution mismatch coefficient (see it in~\pref{def: distribution mismatch coeff}): 
    \[
    \frac{d_\rho^{\pi_i',\,\pi_{-i}^{(t)}}(s)}{d_\rho^{\pi_i^{(t+1)},\,\pi_{-i}^{(t)}}(s)}
    \; \leq \; \frac{d_\rho^{\pi_i',\, \pi_{-i}^{(t)}}(s)}{(1-\gamma)\rho(s)}
    \; \leq \; \frac{\kappa_\rho}{1-\gamma},
    \]
    $(d)$ follows the Cauchy–Schwarz inequality, the inequality $\sqrt{\sum_i x_i}\leq \sum_i\sqrt{x_i}$ for any $x_i\geq 0$, the Jensen's inequality, and the definition of $L_i^{(t)} (\hat w_i^{(t)})$,
    \[
    \begin{array}{rcl}
        & & \!\!\!\! \!\!\!\! \!\! \displaystyle\left\vert\sum_{s} d_\rho^{\pi^{(t)}}(s) \big\langle
            \pi_i'(\cdot\,|\,s) - \pi_i^{(t)}(\cdot\,|\,s)
            ,\,
            \bar Q_i^{(t)}(s,\cdot) - \hat{Q}_i^{(t)}(s,\cdot)
            \big\rangle_{\calA_i} \right\vert \\[0.2cm]
        & \lesssim & \displaystyle
        \sqrt{ \sum_{s} d_\rho^{\pi^{(t)}}(s)}  \sqrt{\sum_s d_\rho^{\pi^{(t)}}(s) \sum_{a_i} \left(\bar Q_i^{(t)}(s,a_i) - \hat{ Q}_i^{(t)}(s,a_i)\right)^2}
             \\[0.2cm]
             &\leq& \displaystyle \sqrt{\frac{AL_i^{(t)} (\hat w_i^{(t)})}{\xi}}
    \end{array}
    \]
    where $L_i^{(t)} (\hat w_i^{(t)})
    \DefinedAs 
    \mathbb{E}_{s \,\sim\,d_\rho^{(t)}, a\sim \pi_i^{(t)}(\cdot|s)} \big[
    \big(
    \bar Q_i^{(t)}(s,a_i) -  \hat{Q}_i^{(t)}(s,a_i)
    \big)^2
    \big]$, and $\hat{Q}_i^{(t)}(s,a_i) = \langle \phi_i(s,a_i), \hat w_i^{(t)}\rangle$, and we replace $i$ ( $\argmax_i$ in $(b)$) in the square root term in $(e)$ by the sum over all players. 
    
    If we proceed \pref{eq: summary bound sample-based} with the first bound (i) in~\pref{lem: MPG policy improvement sample-based}, then,
    \[
        \begin{array}{rcl}
             & & \!\!\!\! \!\!\!\! \!\!
             \displaystyle \mathbb{E}\left[ 
             \sum_{t \,=\, 1}^T\max_{i} \left(\max_{\pi_i'} V_i^{\pi_i', \, \pi_{-i}^{(t)}}(\rho) -  V_i^{\pi^{(t)}}(\rho) \right) \right]\\[0.4cm]         
             & \overset{(a)}{\lesssim} & 
             \displaystyle  \frac{\sqrt{\kappa_\rho}}{\eta(1-\gamma)^{\frac{3}{2}}} \sqrt{T} 
             \sqrt{
             \eta (1-\gamma) (\Phi^{(N+1)}-\Phi^{(1)}  )
             + 
             \frac{\eta^3 A W^2 N^2 }{(1-\gamma)^2  }T
             +
             \frac{\eta^2 \kappa_\rho A }{(1-\gamma)\xi} \sum_{t \,=\, 1}^T \sum_{i \,=\, 1}^N {\mathbb{E}\left[  L_i^{(t)} (\hat w_i^{(t)})\right]}
             }
             \\[0.2cm]
             & & \displaystyle +\, \frac{\xi TW}{1-\gamma} 
             \,+\, 
             \frac{\kappa_\rho }{1-\gamma} \sum_{t \,=\, 1}^T \sqrt{\frac{A\mathbb{E}\left[L_i^{(t)} (\hat w_i^{(t)})\right]}{\xi}}
             \\[0.4cm]
             & \overset{(b)}{\lesssim} & \displaystyle  
             \sqrt{\frac{ \kappa_\rho TC_\Phi}{\eta(1-\gamma)^2} }
             \, + \,  
             T \sqrt{ \frac{\eta \kappa_\rho A W^2 N^2 }{(1-\gamma)^5  }}
             \, + \, 
             \frac{\kappa_\rho \, T }{(1-\gamma)^2} \sqrt{\frac{AN 
             \epsilon_{\text{stat}}}{\xi}}
             \, + \,
             \frac{\xi T W}{1-\gamma}
        \end{array}
    \]
    where we apply the first bound (i) in~\pref{lem: MPG policy improvement sample-based} and the telescoping sum for $(a)$, and we use the boundedness of the potential function: $|\Phi^{\pi} - \Phi^{\pi'}| \leq C_\Phi$ for any $\pi$ and $\pi'$, and further simplify the bound in $(f)$ by Assumption~\pref{as: bounded error}.
    We complete the proof of (i) by taking stepsize $\eta = \frac{(1-\gamma)^{3/2}\sqrt{C_\Phi}}{WN\sqrt{AT}}$ and exploration rate 
	$\xi\leq\left(\frac{\kappa_\rho^2 NA\epsilon_{\normalfont\text{stat}}}{(1-\gamma)^2W^2}\right)^{\frac{1}{3}}$.
    
    If we proceed \pref{eq: summary bound sample-based} with the first bound (ii) in~\pref{lem: MPG policy improvement sample-based} with the choice of $\eta \leq \frac{(1-\gamma)^4}{16 \kappa_\rho^3 NA}$, then,
    \[
        \begin{array}{rcl}
             & & \!\!\!\! \!\!\!\! \!\!
             \displaystyle \mathbb{E}\left[ 
             \sum_{t \,=\, 1}^T\max_{i} \left(\max_{\pi_i'} V_i^{\pi_i', \, \pi_{-i}^{(t)}}(\rho) -  V_i^{\pi^{(t)}}(\rho) \right) \right]\\[0.4cm]         
             & \overset{(a)}{\lesssim} & 
             \displaystyle  \frac{\sqrt{\kappa_\rho}}{\eta(1-\gamma)^{\frac{3}{2}}} \sqrt{T} 
             \sqrt{
             \eta (1-\gamma) (\Phi^{(N+1)}-\Phi^{(1)}  )
             +
             \frac{\eta^2 \kappa_\rho A }{(1-\gamma)\xi} \sum_{t \,=\, 1}^T \sum_{i \,=\, 1}^N {\mathbb{E}\left[  L_i^{(t)} (\hat w_i^{(t)})\right]}
             }
             \\[0.2cm]
             & & \displaystyle +\, \frac{\xi TW}{1-\gamma} 
             \,+\, 
             \frac{\kappa_\rho }{1-\gamma} \sum_{t \,=\, 1}^T  \sqrt{\frac{A \mathbb{E}\left[ L_i^{(t)} (\hat w_i^{(t)})\right]}{\xi}}
             \\[0.4cm]
             & \overset{(b)}{\lesssim} & \displaystyle  
             \sqrt{\frac{ \kappa_\rho T C_\Phi}{\eta(1-\gamma)^2} }
             \, + \, 
             \frac{\kappa_\rho \, T  }{(1-\gamma)^2} \sqrt{\frac{AN\epsilon_{\text{stat}}}{\xi}}
             \, +\,
             \frac{\xi TW}{1-\gamma}
        \end{array}
    \]
    which completes the proof if we choose $\eta = \frac{(1-\gamma)^4}{16\kappa_\rho^3 NA}$ and exploration rate 
	$\xi\leq\left(\frac{\kappa_\rho^2 NA\epsilon_{\normalfont\text{stat}}}{(1-\gamma)^2W^2}\right)^{\frac{1}{3}}$.
\end{proof}

\subsection{Proof of~\pref{thm: sample PMA cooperative}}\label{ap: convergence PMA cooperative sample-based}

We first establish policy improvement regarding the $Q$-function at two consecutive policies $\pi^{(t+1)}$ and $\pi^{(t)}$ in~\pref{alg: PMA fa}.

\begin{lemma}[Policy improvement: Markov cooperative games]\label{lem: MCG policy improvement sample-based}
   For MPG~\pref{eq: Markov potential game} with identical rewards and an initial state distribution $\rho>0$, if all players independently perform the policy update in~\pref{alg: PMA fa} with stepsize $\eta\leq \frac{1-\gamma}{2N}$, then for any $t$ and any $s$, 
   \[
   \begin{array}{rcl}
        \displaystyle \E_{a\sim \pi^{(t+1)}(\cdot\,|\,s)} \left[Q^{(t)}(s,a)\right] \,-\,
        \E_{a\sim \pi^{(t)}(\cdot\,|\,s)} \left[Q^{(t)}(s,a)\right]
        & \geq & \displaystyle
        \frac{1}{8\eta}\sum_{i=1}^N \norm{\pi_i^{(t+1)}(\cdot\,|\,s) 
        - \pi_i^{(t)}(\cdot\,|\,s)}^2 
        \\[0.2cm]
        && \displaystyle-\,
         \eta \sum_{i \,=\, 1}^N \norm{{Q}_i^{(t)}(s,\cdot)  - \hat{\bar Q}_i^{(t)}(s,\cdot) }^2
   \end{array}
    \]
    where $\eta$ is the stepsize and $N$ is the number of players,
\end{lemma}
\begin{proof}[\pfref{lem: MCG policy improvement sample-based}]
    As done in the proof of \pref{lem: MCG policy improvement}, we let $\Psi^{\pi} \DefinedAs \E_{a\,\sim\, \pi(\cdot\,|\,s)}\left[Q^{(t)}(s,a)\right]$ and \pref{eq: Q difference decomp} holds, where $Q^{(t)} \DefinedAs Q^{\pi^{(t)}}$. By taking $\pi'=\pi^{(t+1)}$ and $\pi = \pi^{(t)}$ for~\pref{eq: Q difference decomp},  
    \allowdisplaybreaks
     \[
     \begin{array}{rcl}
        && \!\!\!\!  \!\!\!\!  \!\!
        \displaystyle
        \E_{a \,\sim\, \pi'(\cdot\,|\,s)} \left[Q^{(t)}(s,a)\right] 
        \, - \, 
        \E_{a \,\sim\, \pi(\cdot\,|\,s)} \left[Q^{(t)}(s,a)\right] 
        \\[0.2cm]
        & = & \displaystyle
        \sum_{i \,=\, 1}^N \sum_{a_i}\left(\pi'_i(a_i \,|\, s) - \pi_i(a_i \,|\, s) \right)\hat{ Q}_i^{(t)}(s,a_i) 
        \,+\, 
        \sum_{i \,=\, 1}^N \sum_{a_i}\left(\pi'_i(a_i \,|\, s) - \pi_i(a_i \,|\, s) \right)\left( \bar{Q}_i^{(t)}(s,a_i)  - \hat{Q}_i^{(t)}(s,a_i) \right)
        \\[0.2cm]
        && \displaystyle 
        +\, \sum_{i \,=\, 1}^N \sum_{j \,=\, i+1}^N \sum_{a_i,\, a_j}\left( \pi'_i(a_i \,|\, s) - \pi_i(a_i \,|\, s) \right)\left( \pi'_j(a_j \,|\, s) - \pi_j(a_j \,|\, s) \right)\E_{a_{-ij}\,\sim\, \tilde{\pi}_{-ij}(\cdot\,|\,s)}\left[Q^{(t)}(s,a)\right]
        \\[0.2cm]
        & \overset{(a)}{\geq} & \displaystyle 
        \sum_{i \,= \,1}^N \frac{1}{2\eta} \norm{\pi_i'(\cdot\,|\,s) - \pi_i(\cdot\,|\,s)}^2 
        \,- \,
        \sum_{i \,=\, 1}^N \left( \frac{1}{2\eta'}\norm{\pi'_i(\cdot \,|\, s) - \pi_i(\cdot \,|\, s)}^2 + \frac{\eta'}{2}\norm{ \bar{Q}_i^{(t)}(s,\cdot)  - \hat{Q}_i^{(t)}(s,\cdot) }^2 \right)
        \\[0.2cm]
        && \displaystyle
        - \, \frac{1}{1-\gamma}\sum_{i\,=\,1}^N \sum_{j\,=\,i+1}^N \sum_{a_i,\, a_j}  \left| \pi'_i(a_i\,|\,s) - \pi_i(a_i\,|\,s) \right|\left| \pi'_j(a_j\,|\,s) - \pi_j(a_j\,|\,s) \right| 
        \\[0.2cm]
        & \overset{(b)}{\geq} & \displaystyle
        \sum_{i\,=\,1}^N \frac{1}{4\eta} \norm{\pi_i'(\cdot\,|\,s) - \pi_i(\cdot\,|\,s)}^2 
        \,-\,
        \sum_{i \,=\, 1}^N \eta \norm{ \bar{Q}_i^{(t)}(s,\cdot)  - \hat{Q}_i^{(t)}(s,\cdot) }^2 
        \\[0.2cm]
        && \displaystyle -\, \frac{1}{2(1-\gamma)}\sum_{i\,=\,1}^N \sum_{j\,=\,i+1}^N \left( \norm{\pi_i'(\cdot\,|\,s) - \pi_i(\cdot\,|\,s)}^2 + \norm{ \pi'_j(\cdot\,|\,s) - \pi_j(\cdot\,|\,s) }^2 \right) 
        \\[0.2cm]
        & = & \displaystyle
        \sum_{i=1}^N \frac{1}{4\eta} \norm{\pi_i'(\cdot\,|\,s)  - \pi_i(\cdot\,|\,s)}^2 
        \,-\,
        \sum_{i \,=\, 1}^N \eta \norm{ \bar{Q}_i^{(t)}(s,\cdot)  - \hat{Q}_i^{(t)}(s,\cdot) }^2  
        \,-\, \frac{N-1}{2(1-\gamma)}\sum_{i=1}^N  \norm{\pi_i'(\cdot\,|\,s) - \pi_i(\cdot\,|\,s)}^2
        \\[0.2cm]
         & \overset{(c)}{\geq} & \displaystyle 
         \sum_{i\,=\,1}^N \frac{1}{8\eta}\norm{\pi_i'(\cdot\,|\,s) - \pi_i(\cdot\,|\,s)}^2 
         \,-\,
         \sum_{i \,=\, 1}^N \eta \norm{ \bar{Q}_i^{(t)}(s,\cdot)  - \hat{Q}_i^{(t)}(s,\cdot) }^2 
    \end{array}
     \]
     where $(a)$ is due to the optimality  condition~\pref{eq: optimality condition sample-based}, the inequality $\langle x, y\rangle \leq \frac{\norm{x}^2}{2 \eta'} +\frac{\eta'\norm{y}^2}{2}$ for $\eta'>0$, and  $Q^{(t)}(s,a)\leq\frac{1}{1-\gamma}$, $(b)$ is due to $\langle x, y \rangle \leq \frac{\norm{x}^2+\norm{y}^2}{2}$ and $\eta'= 2\eta$, and $(c)$ follows the choice of $\eta\leq \frac{1-\gamma}{4N}$. 
\end{proof}

\begin{proof}[Proof of~\pref{thm: sample PMA cooperative}]
    By \pref{lem: performance difference} and \pref{lem: MCG policy improvement sample-based}, 
    \[
    \begin{array}{rcl}
        V^{(t+1)}(\rho) \, - \, V^{(t)}(\rho) 
        &=& \displaystyle \frac{1}{1-\gamma}
        \sum_{s,\, a} d_\rho^{\pi^{(t+1)}}(s)\left(\pi^{(t+1)}(a\,|\,s) - \pi^{(t)}(a\,|\,s)\right)Q^{(t)}(s,a) 
        \\[0.2cm]
        &\geq& \displaystyle
        \frac{1}{8\eta(1-\gamma)}\sum_{i\,=\,1}^N \sum_s d_\rho^{\pi^{(t+1)}}(s) \norm{\pi_i^{(t+1)}(\cdot\,|\,s) - \pi_i^{(t)}(\cdot\,|\,s)}^2
        \\[0.2cm]
        && \displaystyle 
        -\,\frac{\eta}{1-\gamma} \sum_{i\,=\,1}^N \sum_s d_\rho^{\pi^{(t+1)}}(s) \norm{ \bar Q_i^{(t)}(s,\cdot) - \hat{ Q}_i^{(t)}(s,\cdot)}^2
        \\[0.2cm]
        &\geq& \displaystyle
        \frac{1}{8\eta(1-\gamma)}\sum_{i\,=\,1}^N \sum_s d_\rho^{\pi^{(t+1)}}(s) \norm{\pi_i^{(t+1)}(\cdot\,|\,s) - \pi_i^{(t)}(\cdot\,|\,s)}^2 
        \, - \,
        \frac{\eta \kappa_\rho A}{\xi(1-\gamma)^2} \sum_{i\,=\,1}^N  {  L_i^{(t)} (\hat w_i^{(t)})}
    \end{array}
    \]
    where the last inequality is due to that
\[
\begin{array}{rcl}
    && \displaystyle
    \!\!\!\! \!\!\!\! \!\!
    \sum_{s} d_\rho^{\pi^{(t+1)}}(s)  
    \norm{
    \bar Q_i^{\pi^{(t)}}(s,\cdot)-\hat{
    Q}_i^{\pi^{(t)}}(s,\cdot)
    }^2
    \\[0.2cm]
    & \overset{(a)}{\leq} & \displaystyle
    \frac{\kappa_\rho}{1-\gamma} \sum_{s} d_\rho^{\pi^{(t)}}(s)
    \norm{
    \bar Q_i^{\pi^{(t)}}(s,\cdot)-\hat{ Q}_i^{\pi^{(t)}}(s,\cdot)
    }^2
    \\[0.2cm]
    & = & \displaystyle
    \frac{\kappa_\rho A}{(1-\gamma)\xi} \sum_{s} d_\rho^{{(t)}}(s) 
    \sum_{a_i} \frac{\xi}{A}
    \big( \bar Q_i^{{(t)}}(s,a_i)- \langle \phi_i(s,a_i), \hat w_i^{(t)} \rangle\big)^2 
    \\[0.2cm]
    & \overset{(b)}{\leq} & \displaystyle
    \frac{\kappa_\rho A}{(1-\gamma)\xi} { \,\mathbb{E}_{s\,\sim\,d_\rho^{{(t)}},\, a_i\,\sim\,\pi_i^{(t)}(\cdot\,\vert\,s)} \left[ \big( \bar Q_i^{{(t)}}(s,a_i)- \langle \phi_i(s,a_i), \hat w_i^{(t)} \rangle\big)^2\right] }
    \\[0.4cm]
    & = & \displaystyle
    \frac{\kappa_\rho A}{(1-\gamma)\xi} {  L_i^{(t)} (\hat w_i^{(t)})}
\end{array}
\]
where $(a)$ follows the definition of $\kappa_\rho$ and $(b)$ is the definition of $L_i^{(t)} (\hat w_i^{(t)})$.
    
    By the same argument as the proof of \pref{thm: sample PMA potential},
    \[
        \begin{array}{rcl}
             & & \!\!\!\! \!\!\!\! \!\!
             \displaystyle
             \sum_{t \,=\, 1}^T\max_{i} \left(\max_{\pi_i'} V^{\pi_i', \, \pi_{-i}^{(t)}}(\rho) -  V^{\pi^{(t)}}(\rho) \right) 
             \\[0.4cm] 
             & \lesssim & \displaystyle
             \frac{\sqrt{\kappa_\rho}}{\eta(1-\gamma)^{\frac{3}{2}}} \sqrt{
             \sum_{t \,=\, 1}^T \sum_{s} d_\rho^{\pi_i',\pi_{-i}^{(t)}}(s) 
             }
             \times\,
             \sqrt{
             \sum_{t \,=\, 1}^T\sum_{i \,=\, 1}^N  \sum_{s} d_\rho^{\pi^{(t+1)}}(s) \left\|\pi_i^{(t+1)}(\cdot\,|\,s) - \pi_i^{(t)}(\cdot\,|\,s)\right\|^2
             }
             \\[0.4cm]
             && \displaystyle +\, \frac{\xi T W}{1-\gamma} 
             \,+\, \frac{\kappa_\rho}{1-\gamma} \sum_{t\,=\,1}^T \sqrt{\frac{A L_i^{(t)}(\hat w_i^{(t)}) }{\xi}}.
        \end{array}
    \]
    By taking expectation and the Jensen's inequality,
    \[
        \begin{array}{rcl}
             & & \!\!\!\! \!\!\!\! \!\!
             \displaystyle
             \mathbb{E}\left[\sum_{t \,=\, 1}^T\max_{i} \left(\max_{\pi_i'} V^{\pi_i', \, \pi_{-i}^{(t)}}(\rho) -  V^{\pi^{(t)}}(\rho) \right) \right]
             \\[0.4cm]
             & \lesssim & \displaystyle  \frac{\sqrt{\kappa_\rho}}{\eta(1-\gamma)^{\frac{3}{2}}} \sqrt{T} \sqrt{8\eta(1-\gamma) (V^{(N+1)}-V^{(1)}) + \frac{8 \eta^2 \kappa_\rho A}{(1-\gamma)\xi} \sum_{t \,=\, 1}^T \sum_{i\,=\,1}^N  {  \mathbb{E}\left[L_i^{(t)} (\hat w_i^{(t)}) \right]} }
             \\[0.4cm]
             && \displaystyle +\, \frac{\xi T W}{1-\gamma} 
             \,+\, \frac{\kappa_\rho}{1-\gamma} \sum_{t\,=\,1}^T \sqrt{\frac{A\mathbb{E}\left[ L_i^{(t)}(\hat w_i^{(t)})\right] }{\xi}}
             \\[0.4cm]
             & \lesssim & \displaystyle  \sqrt{\frac{8\kappa_\rho T }{\eta(1-\gamma)^{3}}} 
             \,+\,
             {\kappa_\rho} {T} \sqrt{ \frac{8  A N }{(1-\gamma)^4\xi} {\epsilon_{\text{stat}}}}
             \,+\,
             \frac{\xi TW}{1-\gamma}.
        \end{array}
    \]
    
    We complete the proof by taking stepsize $\eta = \frac{1-\gamma}{2NA}$, exploration rate 
	$\xi\leq\left(\frac{\kappa_\rho^2 NA\epsilon_{\normalfont\text{stat}}}{(1-\gamma)^2W^2}\right)^{\frac{1}{3}}$, and using $V^{(N+1)}-V^{(1)}\leq \frac{1}{1-\gamma}$.
\end{proof}

\subsection{Sample complexity }
\label{ap: sample complexity}

We present our sample complexity guarantees for \pref{alg: PMA fa} in which the regression problem \pref{eq:  linear regression} in each iteration is approximately solved by the stochastic projected gradient descent \pref{eq: stochastic projected gradient descent}. We measure the sample complexity by the total number of trajectory samples $TK$, where $T$ is the number of iterations and $K$ is the batch size of trajectories. 

\begin{cor}[Sample complexity for Markov potential games]\label{cor: sample complexity potential}
    Assume the setting in \pref{thm: sample PMA potential} except for~\pref{as: bounded error}. 
    Suppose we compute $\hat{w}_i^{(t)} \DefinedAs \frac{1}{K} \sum_{k\,=\,1}^K \beta_{k}^{(K)} {w}_i^{(k)}$ via a stochastic projected gradient descent \pref{eq: stochastic projected gradient descent} with stepsize $\lambda^{(k)} = \frac{2}{2+k}$ and $\beta_{k}^{(K)} = \frac{1/\lambda^{(k)}}{\sum_{r\,=\,1}^K 1/\lambda^{(r)}}$. Then, if we choose stepsize $\eta=\frac{(1-\gamma)^{3/2}\sqrt{C_\Phi}}{WN\sqrt{AT}}$ and exploration rate $\xi = \min \left( \left(\frac{\kappa_\rho^2 A Nd}{(1-\gamma)^4 K}\right)^{\frac{1}{3}}, \frac{1}{2}\right)$, then,
	\begin{align*}
	& \mathbb{E}
	\left[\,
	\text{\normalfont Nash-Regret}(T)  
	\,\right] 
	\; \lesssim \;  
	\frac{ \sqrt{\kappa_\rho\, W N} (\,AC_\Phi\,)^{\frac{1}{4}}} 
	{(1-\gamma)^{\frac{7}{2}}\,T^{\frac{1}{4}}} 
	\, +\,
	\frac{ W (\,{\kappa_\rho^2\, ANd}\,)^{\frac{1}{3}} }{(1-\gamma)^{\frac{7}{3}}\, K^{\frac{1}{3}}}.
	\end{align*}
	Furthermore, if we choose stepsize $\eta= \frac{(1-\gamma)^4}{16 \kappa_\rho^3 NA}$ and exploration rate $\xi = \min \left( \left(\frac{\kappa_\rho^2 ANd}{(1-\gamma)^4 K}\right)^{\frac{1}{3}}, \frac{1}{2}\right)$, then,
	\begin{align*}
	& \mathbb{E}
	\left[
	\,\text{\normalfont Nash-Regret}(T)  
	\,\right] 
	\; \lesssim \; 
	\frac{ \kappa_\rho^2\,\sqrt{ AN C_\Phi} } 
	{(1-\gamma)^3\sqrt T} 
	\, + \,
	\frac{ W (\,{\kappa_\rho^2\, ANd}\,)^{\frac{1}{3}} }{(1-\gamma)^{\frac{7}{3}}\, K^{\frac{1}{3}}}.
	\end{align*}
	Moreover, their sample complexity guarantees are $TK = O(\frac{1}{\epsilon^7})$ or $TK = O(\frac{1}{\epsilon^5})$, respectively, for obtaining an $\epsilon$-Nash equilibrium. 
\end{cor}

\begin{proof}[Proof of \pref{cor: sample complexity potential}]

By the unbiased estimate in \pref{ap: unbiased estimate}, the stochastic gradient $\hat \nabla_i^{(t)}$ in \pref{eq: stochastic projected gradient descent} is also unbiased. We note the variance of the stochastic gradient is bounded by $\frac{1}{(1-\gamma)^2}$. By \pref{lem: stochastic projected gradient descent with weighted averaging}, if we choose $\lambda^{(k)} = \frac{2}{2+k}$ and $\beta_{k}^{(K)} = \frac{1/\lambda^{(k)}}{\sum_{r\,=\,1}^K 1/\lambda^{(r)}}$, then
\[
    \mathbb{E} \left[ L_i^{(t)}(\hat w_i^{(t)})\right] \,-\, L_i^{(t)}(w_i^{(t)})
    \;\leq\; \frac{dW^2}{(1-\gamma)^2 K}
\]
where $L_i^{(t)}(w_i^{(t)}) = 0$. by~\pref{as: linear Q}. Therefore, substitution of $\epsilon_{\text{\normalfont stat}}\leq \frac{dW^2}{(1-\gamma)^2K}$ into \pref{thm: sample PMA potential} yields desired results.

Finally, we let the upper bound on $\text{\normalfont Nash-Regret}(T)$ be $\epsilon>0$ and calculate the sample complexity $TK = O(\frac{1}{\epsilon^7})$ or $TK = O(\frac{1}{\epsilon^5})$, respectively. 
\end{proof}

\begin{cor}[Sample complexity for Markov cooperative games]\label{cor: sample complexity cooperative}
    Assume the setting in \pref{thm: sample PMA cooperative} except for Assumption~\pref{as: bounded error}. Suppose we compute $\hat{w}_i^{(t)} \DefinedAs \frac{1}{K} \sum_{k\,=\,1}^K \beta_{k}^{(K)} {w}_i^{(k)}$ via a stochastic projected gradient descent \pref{eq: stochastic projected gradient descent} with stepsize $\lambda^{(k)} = \frac{2}{2+k}$ and $\beta_{k}^{(K)} = \frac{1/\lambda^{(k)}}{\sum_{r\,=\,1}^K 1/\lambda^{(r)}}$. Then, if we choose stepsize $\eta=\frac{1-\gamma}{WNA\sqrt{T}}$ and exploration rate $\xi = \min \left( \left(\frac{\kappa_\rho^2 AN}{(1-\gamma)^4 K}\right)^{\frac{1}{3}}, \frac{1}{2}\right)$, then,
	\begin{align*}
	& \mathbb{E}
	\left[\,
	\text{\normalfont Nash-Regret}(T)  
	\,\right]   
	\; \lesssim \;  
	\frac{ \sqrt{\kappa_\rho \, AN} } 
	{(1-\gamma)^2\, \sqrt{T}} 
	\, +\,
	\frac{ W (\,{\kappa_\rho^2\, ANd}\,)^{\frac{1}{3}} }{(1-\gamma)^{\frac{7}{3}}\, K^{\frac{1}{3}}}.
	\end{align*}
	Moreover, the sample complexity guarantee is $TK = O(\frac{1}{\epsilon^5})$ for obtaining an $\epsilon$-Nash equilibrium. 
\end{cor}

\begin{proof}[Proof of \pref{cor: sample complexity cooperative}]
    The proof follows the proof steps of \pref{cor: sample complexity potential} above.
\end{proof}

\section{Proofs for Section~\ref{sec: BBW}}
\label{ap: BBW}

In this section, we prove \pref{thm: bbw asymptotic} and \pref{thm: convergence OPMA cooperative} in 
\pref{ap: BBW asymptotic convergence} and \pref{ap: convergence OPMA cooperative}, respectively.

\subsection{Proof of \pref{thm: bbw asymptotic}}
\label{ap: BBW asymptotic convergence}

It is convenient to introduce an auxiliary sequence $\{\alpha^{(t,\tau)}\}_{\tau\,=\,0}^\infty$ associated with the learning rate $\{\alpha^{(t)}\}_{t\,=\,1}^\infty$,  
    \begin{align}\label{eq: auxiliary alpha}
        \displaystyle
        \alpha^{(t,\tau)} 
        \; \DefinedAs \; 
        \begin{cases}
        \displaystyle\prod_{j\,=\,1}^t(1-\alpha^{(j)}), \quad &\text{for\ } \tau \,=\, 0 
        \\[0.2cm]
        \displaystyle\alpha^{(\tau)}\prod_{j\,=\,\tau+1}^{t}(1-\alpha^{(j)}), \quad &\text{for\ } 1 \,\leq\, \tau \,\leq\, t 
        \\[0.2cm]
        0,   \quad &\text{for\ } \tau \,>\, t.
        \end{cases}
    \end{align}
It is straightforward to verify that $\sum_{\tau\,=\,0}^{t-1} \alpha^{(t-1,\tau)} = 1$ for $t\geq 1$.

\begin{lemma}\label{lem: V decompose}
     In \pref{alg: OPMA smooth Q}, $\calV_s^{(t)} = \sum_{\tau\,=\,1}^t \alpha^{(t,\tau)} (x_s^{(\tau)})^\top \calQ_s^{(\tau)}y_s^{(\tau)}$ for all $s, t$. 
\end{lemma}
\begin{proof}[Proof of \pref{lem: V decompose}]
     We prove it by induction. When $t=0$ and $t=1$, it holds trivially by noting that $\calV_s^{(0)}=0$ and $\alpha^{(1,1)} = \alpha^{(1)}$. Assume that it holds for $0, 1, \ldots, t-1$. By the update rule for $\calV_s^{(t)} $ in \pref{alg: OPMA smooth Q}, 
     \begin{align*}
         \calV_s^{(t)} 
         & \;=\; (1-\alpha^{(t)}) \calV_s^{(t-1)} 
         \,+\, \alpha^{(t)}(x_s^{(t)})^\top \calQ_s^{(t)}y_s^{(t)} 
         \\
         & \;\overset{(a)}{=}\; (1-\alpha^{(t)})\sum_{\tau\,=\,1}^{t-1}\alpha^{(t-1,\tau)}(x_s^{(\tau)})^\top \calQ_s^{(\tau)}y_s^{(\tau)}
         \,+\, \alpha^{(t,t)}(x_s^{(t)})^\top \calQ_s^{(t)}y_s^{(t)} 
         \\
         & \;\overset{(b)}{=}\; \sum_{\tau\,=\,1}^t \alpha^{(t,\tau)} (x_s^{(\tau)})^\top \calQ_s^{(\tau)}y_s^{(\tau)} 
     \end{align*}
     where $(a)$ follows the induction hypothesis and $(b)$ is due to the definition of $\alpha^{(t,\tau)}$. 
\end{proof}

\begin{lemma}\label{lem: potential increase 1}
     In \pref{alg: OPMA smooth Q}, for every state $s$ and time $t\geq 1$,  
     \begin{align*}
         (x_s^{(t+1)})^\top \calQ_s^{(t)}y_s^{(t+1)}  - (x_s^{(t)})^\top \calQ_s^{(t)}y_s^{(t)} \;\geq\;
         \frac{15}{16\eta} \norm{z_s^{(t+1)}-\bar z_s^{(t+1)}}^2 
         \,+\,
         \frac{7}{16\eta} \norm{\bar z_s^{(t+1)}-\bar z_s^{(t)}}^2 \,-\,
         \frac{9}{16\eta} \norm{z_s^{(t)}- \bar z_s^{(t)}}^2
     \end{align*}
     where $z_s^{(t)} = (x_s^{(t)}, y_s^{(t)})$ and $\bar z_s^{(t)} = (\bar x_s^{(t)}, \bar y_s^{(t)})$. 
\end{lemma}
\begin{proof}[Proof of \pref{lem: potential increase 1}]
     We decompose the difference into three terms: 
     \begin{align*}
     &(x_s^{(t+1)})^\top \calQ_s^{(t)}y_s^{(t+1)}  - (x_s^{(t)})^\top \calQ_s^{(t)}y_s^{(t)}\\
     & =  \;
     \underbrace{(x_s^{(t+1)}-x_s^{(t)})^\top \calQ_s^{(t)} y_s^{(t)}}_{\textbf{Diff}_x} 
            \, +\, 
        \underbrace{(x_s^{(t)})^\top \calQ_s^{(t)} (y_s^{(t+1)}-y_s^{(t)}) }_{\textbf{Diff}_y}
       \, +\, \underbrace{(x_s^{(t+1)}-x_s^{(t)})^\top \calQ_s^{(t)} (y_s^{(t+1)}-y_s^{(t)}) }_{\textbf{Diff}_{xy}}.
     \end{align*}
      We next deal with $\textbf{Diff}_x$, $\textbf{Diff}_y$, and $\textbf{Diff}_{xy}$, separately. 
     
     \noindent\textbf{Bounding $\textbf{Diff}_x$}. 
     The optimality of $x_s^{(t+1)}$ implies that for any $x_s'\in \Delta(\calA_1)$,
    \[
    (x_s^{(t+1)})^\top \calQ_s^{(t)} y_s^{(t)}  
    \,-\,
    \frac{1}{2\eta} \norm{x_s^{(t+1)}-\bar x_s^{(t+1)}}^2
    \; \geq  \;
    (x_s')^\top \calQ_s^{(t)} y_s^{(t)}  \,-\,
    \frac{1}{2\eta} \norm{x_s'-\bar x_s^{(t+1)}}^2 
    \,+\,
    \frac{1}{2\eta} \norm{x_s'- x_s^{(t+1)}}^2
    \]
    which implies that, by taking $x_s' = \bar x_s^{(t+1)}$,
    \begin{equation}\label{eq: policy diffs sum11}
    (x_s^{(t+1)} - \bar x_s^{(t+1)})^\top \calQ_s^{(t)} y_s^{(t)} 
    \; \geq  \;
    \frac{1}{\eta} \norm{x_s^{(t+1)}-\bar x_s^{(t+1)}}^2.
    \end{equation}
    The optimality of $\bar x_s^{(t+1)}$ implies that for any $x_s'\in\Delta(\calA_1)$,
    \[
    (\bar x_s^{(t+1)})^\top \calQ_s^{(t)} y_s^{(t)}  
    \,-\,
    \frac{1}{2\eta} \norm{\bar x_s^{(t+1)}-\bar x_s^{(t)}}^2
    \; \geq  \;
    (x_s')^\top \calQ_s^{(t)} y_s^{(t)}  \,-\,
    \frac{1}{2\eta} \norm{x_s'-\bar x_s^{(t)}}^2 
    \,+\,
    \frac{1}{2\eta} \norm{x_s'- \bar x_s^{(t+1)}}^2
    \]
    which implies that, by taking $x_s' = x_s^{(t)}$,
    \begin{equation}\label{eq: policy diffs sum2}
    (\bar x_s^{(t+1)} -  x_s^{(t)})^\top \calQ_s^{(t)} y_s^{(t)} 
    \; \geq  \;
    \frac{1}{2\eta} \norm{\bar x_s^{(t+1)}-\bar x_s^{(t)}}^2
    \,-\,
    \frac{1}{2\eta} \norm{ x_s^{(t)}-\bar x_s^{(t)}}^2.
    \end{equation}
    Combining the two inequalities above yields 
    \begin{equation}\label{eq: optimality xhalf}
    \textbf{Diff}_x \; = \; (x_s^{(t+1)}-x_s^{(t)})^\top \calQ_s^{(t)} y_s^{(t)}
    \; \geq \;
    \frac{1}{\eta} \norm{x_s^{(t+1)}-\bar x_s^{(t+1)}}^2
     \,+\,
    \frac{1}{2\eta} \norm{\bar x_s^{(t+1)}-\bar x_s^{(t)}}^2
    \,-\,
    \frac{1}{2\eta} \norm{x_s^{(t)}-\bar x_s^{(t)}}^2.
    \end{equation}
    
    \noindent\textbf{Bounding $\textbf{Diff}_y$}. Similarly, 
    \begin{equation}\label{eq: optimality yhalf}
    \textbf{Diff}_y \; =\; (x_s^{(t)})^\top \calQ_s^{(t)} (y_s^{(t+1)}-y_s^{(t)})
    \; \geq \;
    \frac{1}{\eta} \norm{y_s^{(t+1)}-\bar y_s^{(t+1)}}^2
    \,+\,
    \frac{1}{2\eta} \norm{\bar y_s^{(t+1)}-\bar y_s^{(t)}}^2
    \,-\,
    \frac{1}{2\eta} \norm{y_s^{(t)}-\bar y_s^{(t)}}^2.
    \end{equation}
    
    \noindent\textbf{Bounding $\textbf{Diff}_{xy}$}. By the AM-GM and Cauchy-Schwarz inequalities,
    \begin{align*}
        \textbf{Diff}_{xy} & \;\geq\; - \frac{\sqrt{A}}{2(1-\gamma)}\norm{x_s^{(t+1)} - x_s^{(t)}}^2 
        \,-\, \frac{\sqrt{A}}{2(1-\gamma)}\norm{y_s^{(t+1)} - y_s^{(t)}}^2 \\
        &\; \overset{(a)}{\geq} \; - \frac{3\sqrt{A}}{2(1-\gamma)}\bigg(\norm{x_s^{(t+1)} - \bar x_s^{(t+1)}}^2 \,+\, \norm{\bar x_s^{(t+1)} - \bar x_s^{(t)}}^2 \,+\, \norm{\bar x_s^{(t)} -  x_s^{(t)}}^2 \\
        &\qquad \qquad \qquad \qquad \qquad \qquad +\, \norm{y_s^{(t+1)} - \bar y_s^{(t+1)}}^2 \,+\, \norm{\bar y_s^{(t+1)} -  \bar y_s^{(t)}}^2 \,+\, \norm{\bar y_s^{(t)} -  y_s^{(t)}}^2 \bigg) \\
        &\;\overset{(b)}{\geq}\; - \frac{1}{16\eta}\bigg(\norm{x_s^{(t+1)} - \bar x_s^{(t+1)}}^2 \,+\, \norm{\bar x_s^{(t+1)} - \bar x_s^{(t)}}^2 \,+\, \norm{\bar x_s^{(t)} -  x_s^{(t)}}^2 \\
        &\qquad \qquad \qquad \qquad \qquad \qquad +\, \norm{y_s^{(t+1)} - \bar y_s^{(t+1)}}^2 \,+\, \norm{\bar y_s^{(t+1)} -  \bar y_s^{(t)}}^2 \,+\, \norm{\bar y_s^{(t)} -  y_s^{(t)}}^2 \bigg)
    \end{align*}
    where $(a)$ follows $\norm{x+y+z}^2 \leq 3\norm{x}^2+3\norm{y}^2+3\norm{z}^2$ and $(b)$ is by $\eta\leq \frac{1-\gamma}{32\sqrt{A}}$. 
    
    Finally, we complete the proof by summing up the bounds above for $\textbf{Diff}_x$, $\textbf{Diff}_y$, and $\textbf{Diff}_{xy}$.
\end{proof}

\begin{lemma}\label{lem: potential increase lemma bbw}
     In \pref{alg: OPMA smooth Q}, for all $t$ and $s$, the following two inequalities hold:  
     \begin{align*}
         \text{\normalfont(i)}& \ \ \calV_s^{(t)} \geq \calV_s^{(t-1)}; \\
         \text{\normalfont(ii)}& \ \ (x_s^{(t+1)})^\top \calQ_s^{(t+1)}y_s^{(t+1)}  - (x_s^{(t)})^\top \calQ_s^{(t)}y_s^{(t)} \geq \frac{15}{16\eta} \norm{z_s^{(t+1)}-\bar z_s^{(t+1)}}^2 + \frac{7}{16\eta} \norm{\bar z_s^{(t+1)}-\bar z_s^{(t)}}^2 - \frac{9}{16\eta} \norm{z_s^{(t)}- \bar z_s^{(t)}}^2.
     \end{align*}
\end{lemma}
\begin{proof}[Proof of \pref{lem: potential increase lemma bbw}]
     We first note that~(ii) is a consequence of \pref{lem: potential increase 1} and~(i),
     \begin{align*}
         &(x_s^{(t+1)})^\top \calQ_s^{(t+1)}y_s^{(t+1)}  
         \,-\,
         (x_s^{(t)})^\top \calQ_s^{(t)}y_s^{(t)}  \\
         &=\; (x_s^{(t+1)})^\top \calQ_s^{(t+1)}y_s^{(t+1)}  
         \,-\,
         (x_s^{(t+1)})^\top \calQ_s^{(t)}y_s^{(t+1)} 
         \,+\,
         (x_s^{(t+1)})^\top \calQ_s^{(t)}y_s^{(t+1)}  
         \,-\,
         (x_s^{(t)})^\top \calQ_s^{(t)}y_s^{(t)} \\
         &\overset{(a)}{\geq}\; \min_{s'}\gamma \left(\calV_{s'}^{(t)} - \calV_{s'}^{(t-1)}\right) 
         \,+\,
         \frac{15}{16\eta} \norm{z_s^{(t+1)}-\bar z_s^{(t+1)}}^2 
         \,+\,
         \frac{7}{16\eta} \norm{\bar z_s^{(t+1)}-\bar z_s^{(t)}}^2 \,-\,
         \frac{9}{16\eta} \norm{z_s^{(t)}- \bar z_s^{(t)}}^2
         \\
         &\overset{(b)}{\geq}\; 
         \frac{15}{16\eta} \norm{z_s^{(t+1)}-\bar z_s^{(t+1)}}^2 
         \,+\,
         \frac{7}{16\eta} \norm{\bar z_s^{(t+1)}-\bar z_s^{(t)}}^2 \,-\,
         \frac{9}{16\eta} \norm{z_s^{(t)}- \bar z_s^{(t)}}^2
     \end{align*}
     where $(a)$ is due to \pref{lem: potential increase 1}, and the update of $\calQ_s^{(t)}$ in \pref{alg: OPMA smooth Q},
     \[
     \calQ_s^{(t+1)}(a_1,a_2) - \calQ_s^{(t)}(a_1,a_2)
     \; = \;
     \gamma 
     \mathbb{E}_{s'\,\sim\,\mathbb{P}(\cdot\,\vert\,s,a_1,a_2)}
     \left[ 
     \calV_{s'}^{(t)} 
     - 
     \calV_{s'}^{(t-1)} 
     \right]
     \]
     and $(b)$ follows (i).

     Therefore, it suffices to prove (i). We prove it by induction. 
     Define $\zeta^{(t)}_s := \norm{z_s^{(t)}-\bar z_s^{(t)}}^2$ and $\lambda^{(t)}_s := \norm{\bar z_s^{(t+1)}-\bar z_s^{(t)}}^2$. For notational simplicity, define $\calQ_s^{(0)} = \mathbf{0}_{A\times A}$, $z_s^{(0)} = \bar z_s^{(0)} = \frac{1}{A}\one = z_s^{(1)} = \bar z_s^{(1)}$. Thus,~(ii) holds for $t=0$ and~(i) holds for $t=1$. We note that for $t\geq 2$, 
     \begin{align*}
         &\calV_{s}^{(t)} - \calV_{s}^{(t-1)} \\
         &\overset{(a)}{=}\;   \alpha^{(t)}\left(x^{(t)}_{s}\calQ^{(t)}_{s}y^{(t)}_{s} - \calV_{s}^{(t-1)} \right)  \\
         &\overset{(b)}{=}\; \alpha^{(t)} \left(\sum_{\tau\,=\,0}^{t-1} \alpha^{(t-1,\tau)}\left(x^{(t)}_{s}\calQ^{(t)}_{s}y^{(t)}_{s} - x^{(\tau)}_{s}\calQ^{(\tau)}_{s}y^{(\tau)}_{s}\right) \right) \\
         &=\;  \alpha^{(t)} \left(\sum_{\tau\,=\,0}^{t-1} \alpha^{(t-1,\tau)}\sum_{i\,=\,\tau}^{t-1} \left(x^{(i+1)}_{s}\calQ^{(i+1)}_{s}y^{(i+1)}_{s} - x^{(i)}_{s}\calQ^{(i)}_{s}y^{(i)}_{s}\right) \right) \\
         &=\;  \alpha^{(t)} \left(\sum_{\tau\,=\,0}^{t-1} \alpha^{(t-1,\tau)}\sum_{i\,=\,\tau}^{t-1} \left(x^{(i+1)}_{s}\calQ^{(i+1)}_{s}y^{(i+1)}_{s} - x^{(i)}_{s}\calQ^{(i)}_{s}y^{(i)}_{s} - \frac{15}{16\eta} \zeta_{s}^{(i+1)} - \frac{7}{16\eta} \lambda_{s}^{(i)} + \frac{9}{16\eta} \zeta_{s}^{(i)}\right) \right) \\
         &\quad +\alpha^{(t)} \left(\sum_{\tau\,=\,0}^{t-1} \alpha^{(t-1, \tau)}\sum_{i\,=\,\tau}^{t-1} \left(\frac{15}{16\eta} \zeta_{s}^{(i+1)} + \frac{7}{16\eta} \lambda_{s}^{(i)} - \frac{9}{16\eta} \zeta_{s}^{(i)}\right) \right) \\
         &\overset{(c)}{\geq} \; \alpha^{(t)} \sum_{i\,=\,0}^{t-1}\left(\sum_{\tau\,=\,0}^{i} \alpha^{(t-1,\tau)}\right) \left(\frac{15}{16\eta} \zeta_{s}^{(i+1)} - \frac{9}{16\eta} \zeta_{s}^{(i)}\right)   \\
         &= \; \alpha^{(t)} \sum_{i\,=\,1}^{t} \zeta_{s}^{(i)}\left( \frac{15}{16\eta}\sum_{\tau\,=\,0}^{i-1} \alpha^{(t-1,\tau)} - \frac{9}{16\eta}\sum_{\tau\,=\,0}^{i} \alpha^{(t-1,\tau)} \right) -\alpha^{(t)}\left(\sum_{\tau\,=\,0}^{0}\alpha^{(t-1, \tau)}\right)\frac{9\eta}{16}\zeta_{s}^{(0)} \\
         &\overset{(d)}{=} \; \alpha^{(t)} \sum_{i\,=\,2}^{t} \zeta_{s}^{(i)}\left( \frac{15}{16\eta}\sum_{\tau\,=\,0}^{i-1} \alpha^{(t-1,\tau)} - \frac{9}{16\eta}\sum_{\tau\,=\,0}^{i} \alpha^{(t-1,\tau)} \right) \\
         &\overset{(e)}{\geq}\; 0
     \end{align*}
     where $(a)$ follows the update of $\calV_s^{(t)}$ in \pref{alg: OPMA smooth Q}, we apply \pref{lem: V decompose} and $\sum_{\tau\,=\,0}^{t-1} \alpha^{(t-1,\tau)} = 1$ in $(b)$, $(c)$ follows the induction hypothesis~(ii), $(d)$ is due to that  $\zeta_s^{(0)}=\zeta_s^{(1)}=0$, and we apply \pref{lem: little lemma for alpha} for $(e)$.
\end{proof}

\begin{lemma}\label{lem: convergence of several quantity}
     For every $s\in\calS$, the following quantities in \pref{alg: OPMA smooth Q} all converge to some fixed values when $t\rightarrow \infty$:
     \begin{itemize}
         \item[\normalfont(i)] $\calV^{(t)}_s$;
         \item[\normalfont(ii)] $\norm{z_s^{(t)} - \bar z_s^{(t)}}^2 + \norm{\bar z_s^{(t)} - \bar z_s^{(t-1)}}^2$ (converges to zero);
         \item[\normalfont(iii)] $(x_s^{(t)})^\top \calQ_s^{(t)}y_s^{(t)}$.
     \end{itemize}
\end{lemma} 
\begin{proof}
     \noindent\textbf{Establishing (i)}.
     By (i) in \pref{lem: potential increase lemma bbw}, $\{ \calV_s^{(t)} \}_{t\,=\,0}^\infty$ is a bounded increasing sequence. By the monotone convergence theorem, it is convergent. Therefore,~(i) holds. 
     
     \noindent\textbf{Establishing (ii)}.
     By summing up the inequality (ii) in \pref{lem: potential increase lemma bbw} over $t$ and using the fact that $z_s^{(1)} = \bar z_s^{(1)}$, 
     \begin{align*}
         \sum_{\tau\,=\,1}^t \left(\frac{6}{16\eta} \norm{z_s^{(\tau+1)}-\bar z_s^{(\tau+1)}}^2 + \frac{7}{16\eta} \norm{\bar z_s^{(\tau+1)}-\bar z_s^{(\tau)}}^2 \right) 
         \; \leq \;
         (x_s^{(t+1)})^\top \calQ_s^{(t+1)}y_s^{(t+1)}  - (x_s^{(1)})^\top \calQ_s^{(1)}y_s^{(1)} 
         \; \leq \;
         \frac{1}{1-\gamma}
     \end{align*}
     which implies that $\frac{6}{16\eta} \norm{z_s^{(\tau+1)}-\bar z_s^{(\tau+1)}}^2 + \frac{7}{16\eta} \norm{\bar z_s^{(\tau+1)}-\bar z_s^{(\tau)}}^2 $ must converge to zero when $\tau\rightarrow \infty$, which further implies~(ii). 
     
     \noindent\textbf{Establishing (iii)}.
     By (ii) in \pref{lem: potential increase lemma bbw}, 
     \begin{align*}
         &(x_s^{(t+1)})^\top \calQ_s^{(t+1)}y_s^{(t+1)} 
         \, - \,
         \frac{15}{16\eta} \norm{z_s^{(t+1)}-\bar z_s^{(t+1)}}^2 \\
         &\geq \; \left((x_s^{(t)})^\top \calQ_s^{(t)}y_s^{(t)} - \frac{15}{16\eta} \norm{z_s^{(t)}-\bar z_s^{(t)}}^2 \right) 
         \, + \,
         \frac{7}{16\eta} \norm{\bar z_s^{(t+1)}-\bar z_s^{(t)}}^2 \, + \,
         \frac{6}{16\eta} \norm{z_s^{(t)} - \bar z_s^{(t)}}^2.
     \end{align*}
     Therefore, 
     \begin{align*}
         (x_s^{(t)})^\top \calQ_s^{(t)}y_s^{(t)} 
         \,-\,
         \frac{15}{16\eta} \norm{z_s^{(t)}-\bar z_s^{(t)}}^2
     \end{align*}
     converges to a fixed value (increasing and upper bounded). In~(ii), we have shown that $\norm{z_s^{(t)}-\bar z_s^{(t)}}^2$ converges to zero. Therefore, $ (x_s^{(t)})^\top \calQ_s^{(t)}y_s^{(t)}$ must also converge. Therefore,~(iii) holds.
\end{proof}

\begin{lemma}\label{lem: converge to the same value}
     In \pref{alg: OPMA smooth Q}, for every $s\in\calS$, $\lim_{t \, \rightarrow \, \infty}V^{x^{(t)}, y^{(t)}}_s$ exists, and 
     \[
     \lim_{t \,\rightarrow \,\infty} \calV^{(t)}_s 
     \; = \;
     \lim_{t \, \rightarrow \, \infty}V^{x^{(t)}, y^{(t)}}_s.
     \]
\end{lemma} 
\begin{proof}[Proof of \pref{lem: converge to the same value}]
     By \pref{lem: convergence of several quantity}, $\calV_s^{(t)}$ and $(x_s^{(t)})^\top \calQ^{(t)}_s y_s^{(t)} $ both are convergent. Let $\calV_s^{\star}:=\lim_{t\,\rightarrow\,\infty}\calV_s^{(t)}$ and $\sigma_s^\star\DefinedAs\lim_{t\,\rightarrow\,\infty} (x_s^{(t)})^\top \calQ^{(t)}_s y_s^{(t)}$. We next show $\calV_s^{\star} = \sigma_s^\star$ by contradiction. Assume that there exists $\epsilon>0$ such that $|\calV_s^{\star} - \sigma_s^\star| = \epsilon$. Since $(x_s^{(t)})^\top \calQ^{(t)}_s y_s^{(t)}$ converges to $\sigma_s^\star$, there exists some $t_0>0$ such that for all $t\geq t_0$, 
     \begin{align}
      \label{eq: exist epsilon}
        \left| (x_s^{(t)})^\top \calQ^{(t)}_s y_s^{(t)} - \sigma^\star_s \right| 
        \; \leq \; 
        \frac{\epsilon}{3}. 
     \end{align}
     By our choice of $\alpha^{(t)}$,  $\sum_{t\,=\,t'}^\infty\alpha^{(t)}=\infty$ for any $t'$. Thus, there exists $t_1>0$ such that for all $t\geq t_1$ and all $\tau\leq t_0$,
     \begin{align}
     \label{eq: alpha_epsilon}
         \alpha^{(t, \tau)} 
         \; \leq \; 
         \prod_{i\,=\,\tau+1}^t (1-\alpha^{(i)}) 
         \; \overset{(a)}{\leq} \;
         \exp\left( - \sum_{i\,=\,t_0+1}^t \alpha^{(i)} \right) 
         \; \leq \;
         \frac{\epsilon(1-\gamma)}{3t_0}
     \end{align}
     where $\log(1-x) \leq -x$ for $x\in (0,1)$ is used in $(a)$.
     By the update of $\calV^{(t)}_s$ in \pref{alg: OPMA smooth Q},
     for all $t\geq \max(t_0,t_1)$, 
     \begin{align*}
         \left|\calV^{(t)}_s - \sigma_s^\star\right| 
         & \;=\; \left|\sum_{\tau\,=\,0}^t \alpha^{(t, \tau)} \left((x_s^{(\tau)})^\top \calQ^{(\tau)}_s y_s^{(\tau)} - \sigma_s^\star\right) \right| \\
         &\;\overset{(a)}{\leq}\; \left| \sum_{\tau\,=\,0}^{t_0-1} \alpha^{(t,\tau)} \left((x_s^{(\tau)})^\top \calQ^{(\tau)}_s y_s^{(\tau)} - \sigma_s^\star\right) \right| 
         \,+\, \left|\sum_{\tau\,=\,t_0}^{t} \alpha^{(t,\tau)} \left((x_s^{(\tau)})^\top \calQ^{(\tau)}_s y_s^{(\tau)} - \sigma_s^\star\right)\right| \\
         &\;\overset{(b)}{\leq} \; \left(\sum_{\tau\,=\,0}^{t_0-1} \alpha^{(t,\tau)}\right)\times \frac{1}{1-\gamma} 
         \,+\, \left(1-\sum_{\tau\,=\,1}^{t_0-1} \alpha^{(t,\tau)}\right) \times \frac{\epsilon}{3} \\
         &\;\leq \;t_0 \max_{\tau\,\leq\, t_0}\alpha^{(t, \tau)}\times \frac{1}{1-\gamma} \,+\, 
         \frac{\epsilon}{3} \\
         &\;\overset{(c)}{\leq} \;\frac{2\epsilon}{3} 
     \end{align*}
     where we apply the triangle inequality for $(a)$, $(b)$ is due to \pref{eq: exist epsilon} and $\sum_{\tau\,=\,1}^t  \alpha^{(t,\tau)}=1$, and $(c)$ follows \pref{eq: alpha_epsilon}. Since $\left|\calV_s^\star - \sigma_s^\star\right|=\epsilon$, it is impossible that $\calV_s^{(t)}$ converges to $\calV_s^\star$, and it must be that $\calV_s^\star = \sigma_s^\star$. Therefore, $\calV_s^{(t)} - (x_s^{(t)})^\top \calQ^{(t)}_s y_s^{(t)} $ converges to zero as $t\to\infty$.
     
     Equivalently, $\calV_s^{(t)} - (x_s^{(t)})^\top \calQ^{(t)}_s y_s^{(t)}$ can be expressed as
     \begin{align*}
         \left(\calV_s^{(t)} - \calV_s^{(t-1)} \right)
         \, + \,
         \calV_s^{(t-1)} - \sum_{a_1,a_2} x_s^{(t)}(a_1) y_s^{(t)}(a_2) \left(r(s,a_1,a_2) + \gamma\E_{s'\,\sim\, \mathbb{P}(\cdot\,|\,s,a_1,a_2)}\left[\calV_{s'}^{(t-1)}\right]\right).
     \end{align*}
     By letting $t\to 0$, since $\calV_s^{(t)} - \calV_s^{(t-1)} \rightarrow 0$, thus,
     \begin{align*}
         \calV_s^{(t-1)} \,-\, \sum_{a_1,a_2} x_s^{(t)}(a_1) y_s^{(t)}(a_2) \left(r(s,a_1,a_2) + \gamma\E_{s'\,\sim\, \mathbb{P}(\cdot\,|\,s,a_1,a_2)}\left[\calV_{s'}^{(t-1)}\right]\right)
     \end{align*}
     also converges to zero. Hence, $\calV_s^{(t)}$ converges to the unique fixed point of the Bellman equation. By the uniqueness, $\calV_s^{(t-1)} - V_s^{x^{(t)}, y^{(t)}}$ converges to zero. Therefore, $\lim_{t\rightarrow\infty} V_s^{x^{(t)}, y^{(t)}} = \lim_{t\rightarrow \infty}\calV_s^{(t-1)} = \calV_s^\star$. 
\end{proof}

\begin{lemma}\label{lem: gap convergence}
     In \pref{alg: OPMA smooth Q}, for every $s$, 
     \[
     \lim_{t\rightarrow \infty}\max_{x'} \; (x_s'-x_s^{(t)})^\top \calQ^{(t)}_s y_s^{(t)} 
     \; = \; 0.
     \]
\end{lemma}
\begin{proof}[Proof of \pref{lem: gap convergence}]
     By the optimality of $x_s^{(t+1)}$, 
    \[
        \big\langle
        x_s' - x_s^{(t+1)}
        , \, 
        \eta  \calQ^{(t)}_s y_s^{(t)}  - x^{(t+1)}_s + \bar x^{(t+1)}_s 
        \big\rangle
        \; \leq \;
        0,\;
         \text{ for any } x_s'. 
    \]
   Rearranging the inequality yields, for any $x_s'$, 
    \begin{align*}
        \langle x_s' - x_s^{(t+1)}, \calQ_s^{(t)}y_s^{(t)}\rangle & \;\leq\;  
        \frac{1}{\eta}\left( \left\langle x_s' - x_s^{(t+1)}, \,x_s^{(t+1)} - x_s^{(t)}\right\rangle 
        \,+\, 
        \left\langle x_s^{(t+1)} - x_s^{(t)}, \,  \eta  \calQ^{(t)}_s y_s^{(t)}  - x^{(t+1)}_s + \bar x^{(t+1)}_s  \right\rangle \right) \\
        &\;\lesssim\; \frac{1}{\eta} \norm{x_s^{(t+1)} - x_s^{(t)}} \\
        &\;\leq\; \frac{1}{\eta}\left(\norm{x_s^{(t+1)} - \bar x_s^{(t+1)}} + \norm{\bar x_s^{(t+1)} - \bar x_s^{(t)}} 
        \,+\, 
        \norm{\bar x_s^{(t)} 
        \,-\,  
        x_s^{(t)}}\right).
    \end{align*}
    By~(ii) of \pref{lem: convergence of several quantity}, the right-hand side above converges to zero, which completes the proof.
\end{proof}

\begin{lemma}\label{lem: little lemma for alpha}
     Let $\{\alpha^{(t)}\}_{t=1}^\infty$ be a non-increasing sequence  that satisfies $0<\alpha^{(t)}\leq \frac{1}{6}$ for all $t$. Then for any $t\geq i\geq 2$, 
     \begin{align*}
         \sum_{\tau\,=\,0}^{i} \alpha^{(t,\tau)} 
         \;\leq\;
         \frac{5}{3} \sum_{\tau\,=\,0}^{i-1} \alpha^{(t,\tau)}.
     \end{align*}
\end{lemma}
\begin{proof}[Proof of \pref{lem: little lemma for alpha}]
     Equivalently, we prove
     \begin{align*}
         \alpha^{(t,i)} 
         \;\leq\; \frac{2}{3}\sum_{\tau\,=\,0}^{i-1}\alpha^{(t,\tau)}.
     \end{align*}
     If suffices to show that $\alpha^{(t,i)}\leq \frac{2}{3}\alpha^{(t,i-1)} + \frac{2}{3}\alpha^{(t,i-2)}$. 
     We have the following two cases.
     
     \noindent\textbf{Case 1: $i>2$}. 
     By the definition of $\alpha^{(t,\tau)}$ and the monotonicity of $0<\alpha^{(t)}\leq\frac{1}{6}$,
     \begin{align*}
         \frac{\alpha^{(t,i)}}{\alpha^{(t,i-1)}} & \;=\; \frac{\alpha^{(i)}\prod_{j\,=\,i+1}^t (1-\alpha^{(j)})}{\alpha^{(i-1)}\prod_{j\,=\,i}^t (1-\alpha^{(j)})} 
         \;=\; \frac{\alpha^{(i)}}{\alpha^{(i-1)}(1-\alpha^{(i)})} 
         \;\leq\;
         \frac{1}{1-\alpha^{(i)}}
         \;\leq\;
         \frac{1}{1-\frac{1}{6}} 
         \;= \;
         \frac{6}{5} 
         \\
         \frac{\alpha^{(t,i)}}{\alpha^{(t,i-2)}} &\;=\; \frac{\alpha^{(i)}}{\alpha^{(i-2)}(1-\alpha^{(i)})(1-\alpha^{(i-1)})}
         \;\leq\;
         \frac{36}{25}.
     \end{align*}
     Therefore, 
     \begin{align*}
         \frac{2}{3}\alpha^{(t,i-1)} + \frac{2}{3}\alpha^{(t,i-2)} 
         \;\geq\;
         \frac{2}{3}\left(\frac{5}{6} + \frac{25}{36}\right)\alpha^{(t,i)} 
         \;\geq\;
         \alpha^{(t,i)}. 
     \end{align*}
     \noindent\textbf{Case 2: $i=2$}. By the definition of $\alpha^{(t,\tau)}$ and the monotonicity of $0<\alpha^{(t)}\leq\frac{1}{6}$,
     \begin{align*}
         \displaystyle\frac{\alpha^{(t,2)}}{\alpha^{(t,0)}} 
         \;= \;
         \frac{\alpha^{(2)}\prod_{j\,=\,3}^t (1-\alpha^{(j)})}{\prod_{j\,=\,1}^t (1-\alpha^{(j)})} 
         \;=\; \frac{\alpha^{(2)}}{(1-\alpha^{(1)})(1-\alpha^{(2)})}
         \;\leq\; \frac{\frac{1}{6}}{\frac{5}{6}\times \frac{5}{6}}
         \;=\;
         \frac{6}{25}.
     \end{align*}
     Therefore, 
     \begin{align*}
         \frac{2}{3}\alpha^{(t,1)} + \frac{2}{3}\alpha^{(t,0)} 
         \;\geq\; \frac{2}{3}\times\frac{25}{6}\alpha^{(t,2)}
         \;\geq\; 
         \alpha^{(t,2)}. 
     \end{align*}
\end{proof}

\begin{proof}[Proof of \pref{thm: bbw asymptotic}]
    \begin{align*}
        &\max_{x'} \left(\ V^{x', y^{(t)}}(\rho) - V^{x^{(t)}, y^{(t)}}(\rho)\right) \\
        & = \; \max_{x'} \frac{1}{1-\gamma} \sum_s d^{x', y^{(t)}}_\rho(s) \left(x_s' - x_s^{(t)}\right)^\top Q^{x^{(t)}, y^{(t)}}_s y_s^{(t)} \\
        &\leq \; \underbrace{ \max_{x'} \frac{1}{1-\gamma} \sum_s d^{x', y^{(t)}}_\rho(s) \left(x_s' - x_s^{(t)}\right)^\top \calQ^{(t)}_s y_s^{(t)}}_{\textbf{Diff}_P} + \underbrace{\max_{x'}\ \frac{1}{1-\gamma} \sum_s d^{x', y^{(t)}}_\rho(s) \left(x_s' - x_s^{(t)}\right)^\top \left(Q^{x^{(t)}, y^{(t)}}_s - \calQ^{(t)}_s\right) y_s^{(t)}}_{\textbf{Diff}_Q}.
    \end{align*}
    By \pref{lem: gap convergence}, $\textbf{Diff}_P\rightarrow 0$ when $t\rightarrow \infty$. For $\textbf{Diff}_Q$, we notice that 
    \begin{align*}
        \left|Q_s^{x^{(t)}, y^{(t)}} - \calQ_s^{(t)}\right| 
        \;\leq\;
        \gamma \max_{s'} \left| V^{x^{(t)}, y^{(t)}}_{s'} - \calV_{s'}^{(t)} \right|
    \end{align*}
    which converges to zero by \pref{lem: converge to the same value}. Therefore, $\textbf{Diff}_Q\rightarrow 0$ when $t\rightarrow \infty$. Therefore, $(x^{(t)}, y^{(t)})$ converges to a Nash equilibrium when $t\rightarrow \infty$.  
\end{proof}

\subsection{Proof of \pref{thm: convergence OPMA cooperative}}
\label{ap: convergence OPMA cooperative}

We first introduce a corollary of \pref{lem: potential increase lemma bbw}.
\begin{cor}\label{cor: potential increase lemma bbw}
     In \pref{alg: OPMA smooth Q}, for every state $s$, and any $T>0$,  
     \begin{align*}
         \sum_{t\,=\,1}^T \left(
         \norm{\bar z_s^{(t+1)}-\bar z_s^{(t)}}^2 + \norm{\bar z_s^{(t+1)}- z_s^{(t)}}^2 \right)
         \; \leq \;  \frac{8\eta}{1-\gamma} 
     \end{align*}
     where $z_s^{(t)} = (x_s^{(t)}, y_s^{(t)})$ and $\bar z_s^{(t)} = (\bar x_s^{(t)}, \bar y_s^{(t)})$. 
\end{cor}
\begin{proof}[Proof of \pref{cor: potential increase lemma bbw}]
    By (ii) of \pref{lem: potential increase lemma bbw},
    \begin{align*}
         &(x_s^{(t+1)})^\top \calQ_s^{(t+1)}y_s^{(t+1)}  - (x_s^{(t)})^\top \calQ_s^{(t)}y_s^{(t)} +\frac{15}{16\eta} \left( \norm{z_s^{(t)}- \bar z_s^{(t)}}^2
         -
         \norm{z_s^{(t+1)}-\bar z_s^{(t+1)}}^2\right) 
         \\
         &\geq \;
         \frac{7}{16\eta} \norm{\bar z_s^{(t+1)}-\bar z_s^{(t)}}^2 + \frac{6}{16\eta} \norm{z_s^{(t)}- \bar z_s^{(t)}}^2
         \\
         &\geq \;
         \frac{6}{16\eta} \norm{\bar z_s^{(t+1)}-\bar z_s^{(t)}}^2 + \frac{6}{16\eta} \norm{z_s^{(t)}- \bar z_s^{(t)}}^2.
     \end{align*}
    Thus, by the inequality $\norm{x+y}^2\leq 2\norm{x}^2 + 2 \norm{y}^2$,
     \[
     \begin{array}{rcl}
          \norm{\bar z_s^{(t+1)}-\bar z_s^{(t)}}^2 + \norm{\bar z_s^{(t+1)}- z_s^{(t)}}^2 
          & \leq  & 
          3\norm{\bar z_s^{(t+1)}-\bar z_s^{(t)}}^2 +2 \norm{\bar z_s^{(t)}-  z_s^{(t)}}^2 
          \\
          & \leq  & 
          3\norm{\bar z_s^{(t+1)}-\bar z_s^{(t)}}^2 + 3 \norm{\bar z_s^{(t)}-  z_s^{(t)}}^2 
          \\
          & \leq & 8\eta \left( (x_s^{(t+1)})^\top \calQ_s^{(t+1)}y_s^{(t+1)}  - (x_s^{(t)})^\top \calQ_s^{(t)}y_s^{(t)}\right) 
          \\
          && \displaystyle
          +\, \dfrac{15}{2} \left( \norm{z_s^{(t)}- \bar z_s^{(t)}}^2
         -
         \norm{z_s^{(t+1)}-\bar z_s^{(t+1)}}^2\right) 
     \end{array}
     \]
     which yields our desired result if we sum it over $t$, use $(x_s^{(T+1)})^\top \calQ_s^{(T+1)} y_s^{(T+1)} \leq \frac{1}{1-\gamma}$ and $z_s^{(1)}=\bar z_s^{(1)}$, and ignore a negative term.
\end{proof}

\begin{lemma}\label{lem: QQ difference}
In~\pref{alg: OPMA smooth Q}, the gap between the critic $\calQ_s^{(t)}$ and the true $Q_s^{(t)}$ satisfies
\[
    \sum_{t \,=\,1}^T  \max_s \norm{\calQ_s^{(t)} -Q_s^{(t)}}^2_{\infty}
    \; \lesssim \;
    \frac{A}{(\alpha^{(T)})^2(1-\gamma)^6}\sum_{t\,=\,1}^T \max_s \left(\norm{x_s^{(t)} - x_s^{(t-1)}}^2+\norm{y_s^{(t)} -y_s^{(t-1)}}^2\right). 
\]
\end{lemma}
\begin{proof}[Proof of \pref{lem: QQ difference}]
For notational simplicity, define $\calQ_s^{(0)}=Q^{(0)}_s=\mathbf{0}_{A\times A}$. 
\[
\begin{array}{rcl}
     & & \!\!\!\! \!\!\!\! \!\!
     \max_s \norm{Q_s^{(t)} -\calQ_s^{(t)}}_{\infty}^2 
     \\
     & \DefinedAs & 
     \displaystyle
     \max_{s,a_1,a_2} \left\vert Q^{(t)}_s(a_1,a_2) - \calQ^{(t)}_s(a_1,a_2) \right\vert^2 
     \\
     & \leq & \displaystyle \max_{s,a_1,a_2}  \Bigg| r(s,a_1, a_2) + \gamma \E_{s'\,\sim\, \mathbb{P}(\cdot\,|\,s,a_1,a_2)}\left[ (x_{s'})^\top Q_{s'}^{(t)} y_{s'} \right] \\
     &   & \displaystyle \qquad \qquad -\, \sum_{\tau\,=\,0}^{t-1} \alpha^{(t-1,\tau)}\left(r(s,a_1,a_2) + \gamma \E_{s'\sim \mathbb{P}(\cdot\,|\,s,a_1,a_2)}\left[ (x^{(\tau)}_{s'})^\top\calQ^{(\tau)}_{s'}y^{(\tau)}_{s'} \right]  \right) \Bigg|^2\\
     & \leq & \displaystyle \gamma^2 \max_{s'} \left|\sum_{\tau\,=\,0}^{t-1} \alpha^{(t-1,\tau)} \left((x_{s'})^\top Q_{s'}^{(t)} y_{s'} - (x^{(\tau)}_{s'})^\top\calQ^{(\tau)}_{s'}y^{(\tau)}_{s'}\right)\right|^2 \\
     & \overset{(a)}{\leq} & \displaystyle\frac{6 \gamma^2}{1-\gamma} \max_s \left(
     \sum_{\tau\,=\,0}^{t-1}\alpha^{(t-1,\tau)} (x^{(t)}_{s})^\top \left(Q^{(t)}_{s} - Q^{(\tau)}_{s}\right)y^{(t)}_{s} \right)^2 
     \\[0.2cm]
     &   &
     \displaystyle  +\, \frac{2\gamma^2}{1+\gamma} \max_s\left(
     \sum_{\tau\,=\,0}^{t-1}\alpha^{(t-1,\tau)} (x^{(t)}_{s})^\top \left(Q^{(\tau)}_{s}-\calQ^{(\tau)}_{s}\right)y^{(t)}_{s} \right)^2 
     \\[0.2cm]
     &  & \displaystyle  +\, \frac{6\gamma^2}{1-\gamma} \max_s\left(
     \sum_{\tau\,=\,0}^{t-1}\alpha^{(t-1,\tau)} (x^{(t)}_{s})^\top \calQ^{(\tau)}_{s} (y^{(t)}_{s}-y^{(\tau)}_{s}) \right)^2 
     \\[0.2cm]
     &  & \displaystyle  +\, \frac{6\gamma^2}{1-\gamma} \max_s \left(
     \sum_{\tau\,=\,0}^{t-1}\alpha^{(t-1,\tau)} (x^{(t)}_{s}-   x^{(\tau)}_{s})^\top \calQ^{(\tau)}_{s} y^{(\tau)}_{s} \right)^2
     \\
     & \overset{(b)}{\leq} & \displaystyle \frac{2\gamma^2}{1+\gamma} \max_{s}
     \left(\sum_{\tau\,=\,0}^{t-1} \alpha^{(t-1,\tau)}
     \right) \left(\sum_{\tau\,=\,0}^{t-1}\alpha^{(t-1,\tau)}\left(
     (x^{(t)}_{s})^\top \left(\calQ^{(\tau)}_{s}-Q^{(\tau)}_{s}\right)y^{(t)}_{s} \right)^2\right)
     \\
     &  & \displaystyle +\, c' \max_{s} \left(\sum_{\tau\,=\,0}^{t-1}\alpha^{(t-1,\tau)}\left(\|x^{(t)}_{s}-x^{(\tau)}_{s}\|_1+\|y^{(t)}_{s}-y^{(\tau)}_{s}\|_1\right)\right)^2   \\ 
     & \overset{(c)}{\leq} & \displaystyle \frac{2\gamma^2}{1+\gamma}
     \max_{s} 
     \left[
     \sum_{\tau\,=\,0}^{t-1}
     \alpha^{(t-1,\tau)}\norm{ \calQ^{(\tau)}_{s}-Q^{(\tau)}_{s} }_\infty^2\right] 
     \,+\,
     c'\max_{s}\left(\sum_{\tau\,=\,0}^{t-1} \alpha^{(t-1,\tau)}\sum_{h\,=\,\tau+1}^{t}\diff^{(h)}_{s}\right)^2  \\
     &\leq & \displaystyle \gamma  \max_{s} 
     \left[
     \sum_{\tau\,=\,0}^{t-1}
     \alpha^{(t-1,\tau)}\norm{ \calQ^{(\tau)}_{s}-Q^{(\tau)}_{s} }_\infty^2\right] 
     \,+\,
     c'\max_{s}
    \left(\sum_{h\,=\,1}^t\sum_{\tau\,=\,0}^{h-1}\alpha^{(t-1,\tau)} \diff^{(h)}_{s}\right)^2 \\
    &\overset{(d)}{\leq} & \displaystyle \gamma  \max_{s} 
     \left[
     \sum_{\tau\,=\,0}^{t-1}
     \alpha^{(t-1,\tau)}\norm{ \calQ^{(\tau)}_{s}-Q^{(\tau)}_{s} }_\infty^2\right] 
     \,+\, 
     c'\max_{s}
    \left(\sum_{h\,=\,1}^t \delta^{(t-1,h-1)} \diff^{(h)}_{s}\right)^2 \\
    & \overset{(e)}{\leq} &  \displaystyle \gamma  \max_{s} 
     \left[
     \sum_{\tau\,=\,0}^{t-1}
     \alpha^{(t-1,\tau)}\norm{ \calQ^{(\tau)}_{s}-Q^{(\tau)}_{s} }_\infty^2\right] 
     \,+\, 
     c'\max_{s}
    \left(\sum_{h\,=\,1}^t (1-\alpha^{(t)})^{t-h} \diff^{(h)}_{s}\right)^2
\end{array}
\]
where in $(a)$ we apply $(x+y+z+w)^2\leq \frac{6x^2}{1-\gamma} + \frac{2y^2}{1+\gamma} + \frac{6z^2}{1-\gamma} + \frac{6w^2}{1-\gamma}$ from the Cauchy-Schwarz inequality, in $(b)$ we use Lemma~\ref{lem: qiff-policy} and obtain $c' =O\left( \frac{1}{(1-\gamma)^5
}\right)$, in $(c)$ we introduce notation, 
\[
    \diff_s^{(h)} 
    \;\DefinedAs\; \norm{x^{(h)}_{s}-x^{(h-1)}_{s}}_1+\norm{y^{(h)}_{s}-y^{(h-1)}_{s}}_1
\]
and in $(d)$ we introduce notation,
\begin{align*}
    \delta^{(t,\tau)} 
    \; = \;
    \prod_{i \,=\, \tau+1}^t (1-\alpha^{(i)})
\end{align*}
and apply Lemma $35$ of \citet{Wei2021LastiterateCO}, $(e)$ is due to that $\{\alpha^{(t)}\}_{t\,=\,0}^\infty$ is a non-increasing sequence.  

Application of Lemma 33 of \citet{Wei2021LastiterateCO} to the recursion relation above yields
\begin{align}
    \max_s \norm{\calQ_s^{(t)} - Q^{(t)}_s}_{\infty}^2 
    &\; \leq \;
    c'\sum_{\tau \,=\, 1}^t \beta^{(t,\tau)}  \max_{s} \left(\sum_{q \,=\, 1}^{\tau}  (1-\alpha^{(\tau)})^{\tau-q}\diff_s^{(q)}\right)^2 \nonumber \\
    & \;\leq\; c'\sum_{\tau \,=\, 1}^t \beta^{(t,\tau)}   \left(\sum_{q \,=\, 1}^{\tau}  (1-\alpha^{(\tau)})^{\tau-q}\diff^{(q)}\right)^2 \label{eq: max diffQ}
\end{align}
where $\beta^{(t,\tau)} := \alpha^{(\tau)} \prod_{i=\tau}^{t-1}(1-\alpha^{(i)}+\alpha^{(i)}\gamma)$ for $1\leq \tau<t$ and $\beta^{(t,t)}:=1$, and $\diff^{(t)}:=\max_{s}\diff^{(t)}_s$. 

The right-hand side of \pref{eq: max diffQ} can be further upper bounded by  
\begin{align*}
    &c'\sum_{\tau\,=\,1}^t \beta^{(t,\tau)}  \left(\sum_{q\,=\,1}^\tau (1-\alpha^{(\tau)})^{\tau-q}\right)\left(\sum_{q\,=\,1}^\tau (1-\alpha^{(\tau)})^{\tau-q}(\diff^{(q)})^2\right) \\
    &\overset{(a)}{\leq} \; c' \sum_{\tau\,=\,1}^t \frac{\beta^{(t,\tau)}}{\alpha^{(\tau)}} \sum_{q\,=\,1}^\tau (1-\alpha^{(\tau)})^{\tau-q}(\diff^{(q)})^2 \\
    & = \; c' \sum_{q\,=\,1}^t\sum_{\tau\,=\,q}^{t} \frac{\beta^{(t,\tau)}}{\alpha^{(\tau)}} (1-\alpha^{(\tau)})^{\tau-q} (\diff^{(q)})^2 \\
    & = \; c' \sum_{q\,=\,1}^t\left[\sum_{\tau\,=\,q}^{t-1} (1-\alpha^{(t)}+\alpha^{(t)}\gamma)^{t-\tau} (1-\alpha^{(t)})^{\tau-q} (\diff^{(q)})^2 + \frac{(1-\alpha^{(t)})^{t-q}}{\alpha^{(t)}}  (\diff^{(q)})^2 \right]  \\
    & = \; c'\sum_{q\,=\,1}^t \left[(1-\alpha^{(t)}+\alpha^{(t)}\gamma)^{t-q} \sum_{\tau\,=\,q}^{t-1}\left(\frac{1-\alpha^{(t)}}{1-\alpha^{(t)}+\alpha^{(t)}\gamma}\right)^{\tau-q} + \frac{(1-\alpha^{(t)})^{t-q}}{\alpha^{(t)}} \right](\diff^{(q)})^2   \\
    & = \; c'\sum_{q\,=\,1}^t \left[(1-\alpha^{(t)}+\alpha^{(t)}\gamma)^{t-q} \frac{1-\left(\frac{1-\alpha^{(t)}}{1-\alpha^{(t)}+\alpha^{(t)}\gamma}\right)^{t-q}}{\frac{\alpha^{(t)}\gamma}{1-\alpha^{(t)}+\alpha^{(t)}\gamma}} + \frac{(1-\alpha^{(t)})^{t-q}}{\alpha^{(t)}}\right](\diff^{(q)})^2 \\
    &\leq \;\frac{2c'}{\alpha^{(t)}\gamma}\sum_{q\,=\,1}^t (1-\alpha^{(t)}+\alpha^{(t)}\gamma)^{t-q}(\diff^{(q)})^2
\end{align*}
where $(a)$ is due to that $\sum_{q\,=\,1}^\tau (1-\alpha^{(\tau)})^{\tau-q} \leq  \frac{1}{\alpha^{(\tau)}}.$

Substitution of the upper bound above into \pref{eq: max diffQ} yields,
\begin{align*}
    \sum_{t \,=\,1}^T  \max_s \norm{\calQ_s^{(t)} -Q_s^{(t)}}^2_{\infty}
    &\; \lesssim \; \sum_{t\,=\,1}^T \frac{c'}{\alpha^{(t)}} \sum_{q\,=\,1}^t  (1-\alpha^{(t)}+\alpha^{(t)}\gamma)^{t-q} (\diff^{(q)})^2 \\ 
    &\; \overset{(a)}{\leq} \; \sum_{q \,=\, 1}^T \sum_{t \,=\, q}^T \frac{c'}{\alpha^{(T)}}(1-\alpha^{(T)}+\alpha^{(T)}\gamma)^{t-q}(\diff^{(q)})^2 \\
    &\;\overset{(b)}{\leq}\; \frac{c'}{(\alpha^{(T)})^2(1-\gamma)}\sum_{q\,=\,1}^T (\diff^{(q)})^2
\end{align*}
where $(a)$ is due to that $\alpha^{(t)}$ is non-increasing, and $(b)$ is due to that $\sum_{t \,=\, q}^T (1-\alpha^{(T)}+\alpha^{(T)}\gamma)^{t-q} \leq \frac{1}{\alpha^{(T)}(1-\gamma)}$.

Finally, using the definition of $c'$ and applying $\|\cdot\|_1\leq \sqrt{A}\|\cdot\|$ to $\diff^{(q)}$ lead to the desired result. 
\end{proof}

\begin{proof}[Proof of \pref{thm: convergence OPMA cooperative}]

The proof consists of two parts: Markov cooperative games and Markov competitive games, separately.

\noindent\textbf{Markov cooperative games}.
    Fix $s$, the optimality of $\bar x_s^{(t+1)}$ in \pref{alg: OPMA smooth Q} yields
    \begin{equation}\label{eq: optimality 222}
    \begin{array}{rcl}
         \big\langle \eta \calQ_s^{(t)} y_s^{(t)} -  (\bar x_s^{(t+1)} - \bar x_s^{(t)} ), x_s' -  \bar x_{s}^{(t+1)} \big\rangle 
         & \leq &
         0, \; \text{ for any } x_s' \in \Delta(\calA_1).
    \end{array}
    \end{equation}
    Thus, for any $x_s'\in\Delta(\calA_1)$,
    \[
    \begin{array}{rcl}
         & & \!\!\!\!  \!\!\!\!  \!\!
          (x_s'-x_s^{(t)})^\top Q_s^{(t)} y_s^{(t)}
         \\[0.2cm]
         & = & \displaystyle
         (x_s' - \bar x_s^{(t+1)})^\top \calQ_s^{(t)} y_s^{(t)} 
         \,+\,
         (x_s' -\bar x_s^{(t+1)})^\top (Q_s^{(t)} -\calQ_s^{(t)}) y_s^{(t)} 
         \,+\,
         (\bar x_s^{(t+1)}-x_s^{(t)})^\top Q_s^{(t)} y_s^{(t)}
         \\[0.2cm]
         & \overset{(a)}{\lesssim} & \displaystyle
         \frac{1}{\eta}(x_s' -\bar x_s^{(t+1)})^\top (\bar x_s^{(t+1)}-\bar x_s^{(t)})
         \,+\,
         \norm{Q_s^{(t)}-\calQ_s^{(t)}}
         \,+\,
         \frac{\sqrt{A}}{1-\gamma} 
         \norm{\bar x_s^{(t+1)}-x_s^{(t)}}
         \\[0.2cm]
         & \overset{(b)}{\lesssim} & \displaystyle
         \frac{1}{\eta} \left(
         \norm{ \bar x_s^{(t+1)}-\bar x_s^{(t)}} 
         \,+\,
         \norm{\bar x_s^{(t+1)}-x_s^{(t)}}\right)
         \,+\,
         \norm{Q_s^{(t)}-\calQ_s^{(t)}}
    \end{array}
    \]
    where we use \pref{eq: optimality 222} and $\norm{Q_s^{(t)}} \leq \frac{\sqrt{A}}{1-\gamma}$ in $(a)$, and $(b)$ is due to the Cauchy-Schwarz inequality and the choice of $\eta\leq \frac{1-\gamma}{32\sqrt{A}}$. 
    Hence,
    \begin{align*}
             &\displaystyle
             \sum_{t \,=\, 1}^T \left(\max_{x'} V^{x', \, y^{(t)}}(\rho) -  V^{x^{(t)}, y^{(t)}}(\rho) \right) \\[0.2cm]         
             & = \; \displaystyle
             \frac{1}{1-\gamma}\sum_{t\, = \, 1}^T \max_{x'} \sum_{s} d_\rho^{ x',y^{(t)}}(s)
             \left(x_s' - x_s^{(t)}\right)^\top 
             Q_s^{{(t)}}y_s^{(t)}
             \\[0.2cm]
             & \lesssim\; \displaystyle
             \frac{1}{\eta(1-\gamma)} \sum_{t \,=\, 1}^T \sum_{s} d_\rho^{x', y^{(t)}}(s) \left(\left\|\bar x_s^{(t+1)} - \bar x_s^{(t)}\right\| + \left\|\bar x_s^{(t+1)} - x_s^{(t)}\right\|\right) 
             \,+\,
             \frac{1}{(1-\gamma)} \sum_{t \,=\, 1}^T \sum_{s} d_\rho^{x', y^{(t)}}(s) \norm{Q_s^{(t)} - \calQ_s^{(t)}} 
             \\[0.2cm]
             & \overset{(a)}{\lesssim}  \; \displaystyle
             \frac{1}{\eta(1-\gamma)} \sqrt{
             \sum_{t \,=\, 1}^T \sum_{s} d_\rho^{x', y^{(t)}}(s)
             }
             \times\,
             \sqrt{
             \sum_{t \,=\, 1}^T \sum_{s} d_\rho^{x', y^{(t)}}(s) \left(\left\|\bar x_s^{(t+1)} - \bar x_s^{(t)}\right\|^2 + \left\|\bar x_s^{(t+1)} - x_s^{(t)}\right\|^2\right)
             }
             \\[0.2cm]
             & \displaystyle \qquad \qquad +\, \frac{1}{(1-\gamma)} \sum_{t \,=\, 1}^T \sum_{s} d_\rho^{x', y^{(t)}}(s) \norm{Q_s^{(t)} - \calQ_s^{(t)}}  
             \\[0.2cm]
             & \overset{(b)}{\leq}  \; \displaystyle
             \frac{\sqrt{T}}{\eta(1-\gamma)}
             \times\,
             \sqrt{
             \sum_{t \,=\, 1}^T \sum_{s} \left(\left\|\bar x_s^{(t+1)} - \bar x_s^{(t)}\right\|^2 + \left\|\bar x_s^{(t+1)} - x_s^{(t)}\right\|^2\right)
             } 
             \,+\,
             \frac{1}{(1-\gamma)} \sum_{t \,=\, 1}^T \sum_{s} d_\rho^{x', y^{(t)}}(s) \norm{Q_s^{(t)} - \calQ_s^{(t)}} 
             \\[0.2cm]
             & \overset{(c)}{\lesssim}  \; \displaystyle \frac{\sqrt{T}}{\eta(1-\gamma)}\sqrt{\frac{\eta S}{1-\gamma}} 
             \,+\, \frac{1}{1-\gamma}\sqrt{\sum_{t=1}^T \sum_s d_\rho^{x',y^{(t)}}(s)}\sqrt{\sum_{t=1}^T \sum_s \norm{Q_s^{(t)} - \calQ_s^{(t)}}^2} 
             \\[0.2cm]
             & \overset{(d)}{\lesssim}  \; \displaystyle  \sqrt{\frac{ST}{\eta(1-\gamma)^3}} 
             \,+\, \frac{1}{1-\gamma}\sqrt{T}\sqrt{\frac{SA}{(\alpha^{(T)})^2(1-\gamma)^6}\sum_{t=1}^T \max_s \left(\norm{x_s^{(t)} - x_s^{(t-1)}}^2+\norm{y_s^{(t)} -y_s^{(t-1)}}^2\right) } 
             \\[0.2cm]
             & \overset{(e)}{\lesssim}  \; \displaystyle \sqrt{\frac{ST}{\eta(1-\gamma)^3}} 
             \,+\, \frac{1}{1-\gamma}\sqrt{T}\sqrt{\frac{\eta S^2A}{{(\alpha^{(T)})}^2(1-\gamma)^7} } 
    \end{align*}
    where we apply the Cauchy-Schwarz inequality for $(a)$, $(b)$ follows the state distribution $d_\rho^{x',y^{(t)}}(s)$, $(c)$ is due to \pref{cor: potential increase lemma bbw}, $(d)$ is because of \pref{lem: QQ difference}, and $(e)$ again is due to \pref{cor: potential increase lemma bbw}.
    By taking $\eta=\frac{(1-\gamma)^2}{32\sqrt{SA}}$ and $\alpha^{(t)} = \frac{1}{6\sqrt[3]{t}}$, the last upper bound above is of order, 
    $$
    O\left(\frac{(S^3A)^{\frac{1}{4}}\sqrt{T}}{(1-\gamma)^{\frac{7}{2}}\alpha^{(T)}}\right)
    \; =\;
    O\left(\frac{(S^3A)^{\frac{1}{4}}T^{\frac{5}{6}}}{(1-\gamma)^{\frac{7}{2}}}\right).
    $$ 

\noindent\textbf{Markov competitive game}.
We start from an intermediate step in the proof of Theorem~1 of \cite{Wei2021LastiterateCO}. Specifically, they have shown that if both players use \pref{alg: OPMA smooth Q} in a two-player zero-sum Markov game, then,
\begin{align*}
    \sum_{t\,=\,1}^T \left(\max_{x', y'} V^{x', y^{(t)}}(\rho) - V^{x^{(t)}, y'}(\rho) \right) 
    \; = \;
    O\left(\frac{S\sqrt{C_\alpha C_\beta T}}{\eta(1-\gamma)}\right)
\end{align*}
where $C_\alpha \DefinedAs 1+\sum_{t\,=\,1}^T \alpha^{(t)}$ and $C_\beta$ is an upper bound for $\sum_{t \,=\, \tau}^T \beta^{(t,\tau)}$ with $\beta^{(t,\tau)} \DefinedAs \alpha^{(\tau)}\prod_{i\,=\,\tau}^{t-1}(1-\alpha^{(i)}+\alpha^{(i)}\gamma)$ if $\tau<t$ and $\beta^{(t,t)} \DefinedAs 1$. We next calculate the upper bounds for $C_\alpha$ and $C_\beta$. 

\noindent\textbf{Bounding $C_\alpha$}.
Recall that $\alpha^{(t)}=\frac{1}{6} t^{-\frac{1}{3}}$. By the definition of $C_\alpha$, 
\[
C_\alpha 
\; = \;
1 \,+\, \frac{1}{6}\sum_{t\,=\,1}^T t^{-\frac{1}{3}}
\; = \; 
O\left(T^{\frac{2}{3}}\right).
\]

\noindent\textbf{Bounding $C_\beta$}.
Using $\alpha^{(t)}=\frac{1}{6} t^{-\frac{1}{3}}$, for any $\tau \geq 1$, we have
\begin{align}
    \sum_{t \, =\, \tau}^T \beta^{(t,\tau)} 
    & \;\leq\; 1 \, + \, \sum_{t\,=\,\tau+1}^T \alpha^{(\tau)}\prod_{i\,=\,\tau}^{t-1}(1-\alpha^{(i)} + \alpha^{(i)}\gamma)  \nonumber \\
    & \;=\; 1 \,+\, \frac{1}{6}\sum_{t\,=\,\tau+1}^T \tau^{-\frac{1}{3}}\left(1-\frac{1}{6}t^{-\frac{1}{3}}(1-\gamma)\right)^{t-\tau} \nonumber  \\
    & \; \leq \; 1 \,+\, \frac{1}{6}\sum_{t\,=\,\tau+1}^{t_0} \tau^{-\frac{1}{3}} + \frac{1}{6}\sum_{t\,=\,t_0+1}^T \tau^{-\frac{1}{3}} \left(1-\frac{1}{6}t^{-\frac{1}{3}}(1-\gamma)\right)^{t-\tau} \tag{for some $t_0$ defined below} \\
    & \;\leq\; 1 \,+\, \frac{1}{6}\tau^{-\frac{1}{3}}(t_0-\tau) \,+\, \frac{1}{6}\tau^{-\frac{1}{3}} \sum_{t\,=\,t_0+1}^T \exp\left(-\frac{1}{6}t^{-\frac{1}{3}}(1-\gamma)(t-\tau)\right). \label{eq: intermediate 1} 
\end{align} 
Define $t_0 \DefinedAs \tau + H(\tau+c)^{\frac{1}{3}} \ln (\tau+c)+c$, where $H:=\frac{48}{1-\gamma}$ and  $c:=2\left(\frac{2H}{1-\frac{1}{3}}\ln\frac{H}{1-\frac{1}{3}}\right)^{\frac{1}{1-\frac{1}{3}}}$ (if $t_0>T$, we simply ignore the second term in \pref{eq: intermediate 1}).  By \pref{lem: messy lemma} with $q=\frac{1}{3}$, for all $t \geq t_0$,
\begin{align*}
    t-\tau 
    \; \geq \; \frac{H}{2}\left(\frac{t}{2}\right)^{\frac{1}{3}} \ln\left(\frac{t}{2}\right).   
\end{align*}
Hence, we can continue to bound the right-hand side of \pref{eq: intermediate 1}  by 
\begin{align*}
    &O\left(H\left(\frac{\tau+c}{\tau}\right)^{\frac{1}{3}} \ln(\tau+c) + \frac{c}{\tau^{\frac{1}{3}}}\right) \,+\, \frac{1}{6}\tau^{-\frac{1}{3}} \sum_{t\,=\,t_0+1}^T \exp\left(-\frac{1}{12}t^{-\frac{1}{3}}(1-\gamma)\frac{H}{2}\left(\frac{t}{2}\right)^{\frac{1}{3}}  \ln\left(\frac{t}{2}\right)\right) \\
    &\leq\; \tilde{O}\big( H (1+c)^{\frac{1}{3}} + c\big) \,+\, \frac{1}{6}\tau^{-\frac{1}{3}}\sum_{t\,=\,t_0+1}^T \frac{2}{t} \\
    &=\; \tilde{O}\left(\frac{1}{(1-\gamma)^{\frac{3}{2}}}\right)
\end{align*}
which proves that $C_\beta=\tilde{O}\left(\frac{1}{(1-\gamma)^{\frac{3}{2}}}\right)$.

Therefore, 
\begin{align*}
    \sum_{t\,=\,1}^T \left(\max_{x', y'} V^{x', y^{(t)}}(\rho) - V^{x^{(t)}, y'}(\rho) \right)
    & \; = \; O\left(\frac{S\sqrt{C_\alpha C_\beta T}}{\eta(1-\gamma)}\right) 
    \; = \; 
    \tilde{O}\left( \frac{ST^{\frac{5}{6}}}{\eta (1-\gamma)^{\frac{7}{4}}} \right)
\end{align*}
which completes the proof by taking $\eta=\frac{(1-\gamma)^2}{32\sqrt{SA}}$. 
\end{proof}

\begin{lemma}\label{lem: messy lemma}
     Fix $\tau\in\mathbb{N}$, $0<q<1$, $H\geq 1$. Let
     \begin{align*}
         t_0 
         \; \DefinedAs \; 
         \tau 
         \, + \,
         H(\tau+c)^q\ln(\tau + c) 
         \, + \,
         c
     \end{align*}
     where $c\DefinedAs2\left(\frac{2H}{1-q}\ln\frac{H}{1-q}\right)^{\frac{1}{1-q}}$. 
     Then for all $t\geq t_0$, 
     \begin{align*}
         t \, - \, H\left(\frac{t}{2}\right)^q\ln \left(\frac{t}{2}\right) 
         \; \geq \; 
         \tau. 
     \end{align*}
\end{lemma}
\begin{proof}[Proof of \pref{lem: messy lemma}]
     We first show that for all $t\geq c$, 
     \begin{itemize}
         \item $H t^q\ln t\leq t$;
         \item $t - H\left(\frac{t}{2}\right)^q \ln \left(\frac{t}{2}\right)$ is non-decreasing.
     \end{itemize}
     To show the two items above, we apply Lemma A.1 of~\cite{shalev2014understanding} which states that $x\geq 2a\ln(a) \Rightarrow x\geq a\ln(x)$ for any $a>0$. By the definition of $c$, for all $t\geq c$, $t^{1-q}\geq \left(\frac{t}{2}\right)^{1-q}\geq \frac{2H}{1-q}\ln\frac{H}{1-q}$ and thus
     \begin{align*}
         t^{1-q} & \; \geq \; \frac{H}{1-q} \ln (t^{1-q}) 
         \; = \; 
         H\ln t
     \end{align*}
     which proves the first item, and that 
     \begin{align*}
        \left(\frac{t}{2}\right)^{1-q} \; \geq \;
        \frac{H}{1-q} \ln\left(\frac{t}{2}\right)^{1-q} 
        \; = \; 
        H\ln\frac{t}{2}
     \end{align*}
     which gives
     \begin{align*}
         \frac{\mathrm{d}}{\mathrm{d} t}\left(t- \frac{H}{2}\left(\frac{t}{2}\right)^q\ln\left(\frac{t}{2}\right)\right) 
         & \; = \; 1 \, - \,  \frac{H}{2}\cdot\frac{q}{2}\left(\frac{t}{2}\right)^{q-1}\ln\left(\frac{t}{2}\right) \, - \, \frac{H}{2}\cdot\left(\frac{t}{2}\right)^q\frac{1}{t} \\
         & \; \geq \; 1 \,-\, \frac{1}{2}\cdot\frac{q}{2} \,-\, \frac{1}{2} 
         \; \geq \; 0 
     \end{align*}
     which proves the second item.
     
     By the first item and the definition of $t_0$, $t_0\leq \tau + (\tau+c)+c=2\tau+2c$. Then by the second item, for all $t\geq t_0$ we have
     \begin{align*}
         t \, - \, \frac{H}{2}\left(\frac{t}{2}\right)^q\ln \left(\frac{t}{2}\right) 
         \; \geq \;
         t_0 \, - \, \frac{H}{2}\left(\frac{t_0}{2}\right)^q \ln \left(\frac{t_0}{2}\right)
         & \; \geq \; 
         t_0 \, - \, \frac{H}{2}\left(\tau + c\right)^q\ln\left(\tau+c\right)
         \; \geq \; \tau
     \end{align*}
     which completes the proof.
\end{proof}

\begin{lemma}\label{lem: qiff-policy}
For any two policies $(x',y')$ and $(x,y)$,
\[
\max_s\norm{Q_s^{x',y'}
     - Q_s^{x,y}}_{\max}
     \;\leq\;
     \frac{\gamma}{(1-\gamma)^2}
     \max_{s'}
     \left(
     \norm{x_{s'}'-x_{s'}}_1+\norm{y_{s'}'-y_{s'}}_1
     \right).
\]
\end{lemma}
\begin{proof}[Proof of \pref{lem: qiff-policy}]
By the definition,
\begin{equation}\label{eq:qiff}
\begin{array}{rcl}
     && \!\!\!\! \!\!\!\! \!\!
     \norm{Q_s^{x',y'}
     - Q_s^{x,y}}_{\max}
     \\[0.2cm]
     & \overset{(a)}{\DefinedAs} & 
     \displaystyle
     \max_{a_1,a_2}\left\vert
     Q_s^{x',y'}(a_1,a_2)
     - Q_s^{x,y}(a_1,a_2)
     \right\vert
     \\[0.2cm]
     & \overset{(b)}{\leq} & 
     \displaystyle
     \gamma \sum_{s'} \mathbb{P}(s'\,\vert\, s, \bar a_1, \bar a_2) 
     \left\vert
    (x_{s'}')^\top Q_{s'}^{x',y'}y_{s'}'
     -
     (x_{s'})^\top Q_{s'}^{x,y} y_{s'}
     \right\vert
     \\[0.2cm]
     & \leq & 
     \displaystyle
     \gamma 
     \max_{s'} 
     \underbrace{
     \left\vert
    (x_{s'})^\top Q_{s'}^{x',y'}y_{s'}
     -
     (x_{s'})^\top Q_{s'}^{x,y} y_{s'}
     \right\vert
     }_{\textbf{Qiff}}
\end{array}
\end{equation}
where $\bar a_1$ and $\bar a_2$ achieve the maximum in $(a)$, and $(b)$ is due to the Bellman equation,
\[
\begin{array}{rcl}
     Q_s^{x,y}(a_1,a_2) 
    &=& \displaystyle
    r(s,a_1,a_2)
    \,+\,
    \gamma\mathbb{E}_{s' \, \sim \, \mathbb{P}(\cdot\,\vert\, s, a_1, a_2)} \left[V_{s'}^{x,y}\right]
    \\[0.2cm]
    &=& \displaystyle
    r(s,a_1,a_2)
    \,+\,
    \gamma\sum_{s'} \mathbb{P}(s'\,\vert\, s, a_1, a_2)\sum_{a_1',a_2'} x_{s'}(a_1' )y_{s'}(a_2') Q_{s'}^{x,y}(a_1',a_2').
\end{array}
\]

Fix $s'$, we next subtract and add $ (x_{s'})^\top Q_{s'}^{x',y'}y_{s'}$ in $\text{Qiff}$ and apply $|a+b| \leq |a|+|b|$ to reach, 
\[
\begin{array}{rcl}
     \textbf{Qiff} 
     & \leq & 
     \displaystyle 
     \left\vert
     (x_{s'}')^\top Q_{s'}^{x',y'}y_{s'}'
     -
     (x_{s'})^\top Q_{s'}^{x',y'}y_{s'}
     \right\vert
      \, + \,
     \left\vert
     (x_{s'})^\top Q_{s'}^{x',y'} y_{s'}
     -
     (x_{s'})^\top Q_{s'}^{x,y} y_{s'}
     \right\vert
     \\[0.2cm]
     & \leq & 
     \displaystyle
     \frac{1}{1-\gamma} \left\vert
     \sum_{a_1',a_2'} 
     \left(x_{s'}'(a_1)y_{s'}'(a_2')
     -
     x_{s'}(a_1')y_{s'}(a_2')
     \right) Q_{s'}^{x',y'}(a_1',a_2')
     \right\vert
      \\
     && \displaystyle
     +\,
     \left\vert
     (x_{s'})^\top\left(Q_{s'}^{x',y'}
     -
      Q_{s'}^{x,y}\right) y_{s'} 
     \right\vert
     \\[0.2cm]
     & \leq & 
     \displaystyle
     \frac{1}{1-\gamma}
     \sum_{a_1',a_2'} 
     \left\vert
     x_{s'}'(a_1')y_{s'}'(a_2')
     -
     x_{s'}(a_1')y_{s'}(a_2')
     \right\vert
      \,+\,
     \left\vert
     (x_{s'})^\top\left(Q_{s'}^{x',y'}
     - Q_{s'}^{x,y}\right) y_{s'}
     \right\vert.
\end{array}
\]
We also notice that 
\[
\begin{array}{rcl}
     \norm{ x_{s'}'\circ y_{s'}' - x_{s'}\circ y_{s'} }_{1}
     & \DefinedAs & \displaystyle
     \sum_{a_1',a_2'} 
     \left\vert
     x_{s'}'(a_1' ) y_{s'}'(a_2')
     -
     x_{s'}(a_1' ) y_{s'}(a_2')
     \right\vert
     \\[0.2cm]
     & \leq &  
     \displaystyle
     \sum_{a_1',a_2'} 
     \left\vert
     x_{s'}'(a_1' ) y_{s'}'(a_2')
     -
     x_{s'}(a_1' ) y_{s'}'(a_2')
     \right\vert
     \\[0.2cm]
     && 
     \displaystyle
     +\, \sum_{a_1',a_2'} 
     \left\vert
     x_{s'}(a_1' ) y_{s'}'(a_2')
     -
     x_{s'}(a_1' ) y_{s'}(a_2')
     \right\vert
     \\[0.2cm]
     & \leq &  
     \displaystyle
     \sum_{a_1'} 
     \left\vert
     x_{s'}'(a_1' )
     -
     x_{s'}(a_1' ) 
     \right\vert
     \,+\,
     \sum_{a_2'} 
     \left\vert
     y_{s'}'(a_2')
     -
     y_{s'}(a_2')
     \right\vert
     \\[0.2cm]
     & = &  
     \norm{x_{s'}' - x_{s'}}_1
     \,+\,
     \norm{y_{s'}' - y_{s'}}_1
\end{array}
\]
and 
\[
    \left\vert
    (x_{s'})^\top\left(Q_{s'}^{x',y'}
    - Q_{s'}^{x,y}\right) y_{s'}
    \right\vert
    \; \leq \;
    \max_{a_1,a_2}\left\vert Q_{s'}^{x',y'}(a_1,a_2)-Q_{s'}^{x,y}(a_1,a_2) \right\vert
    \;\DefinedAs\;
    \norm{Q_{s'}^{x',y'}- Q_{s'}^{x,y}}_{\max}.
\]
By substituting the upper bound on $\textbf{Qiff}$ above into~\pref{eq:qiff}, 
\[
\begin{array}{rcl}
     && \!\!\!\! \!\!\!\! \!\!
     \norm{Q_s^{x',y'}
     - Q_s^{x,y}}_{\max}
     \\[0.2cm]
     & \leq & \displaystyle
     \gamma\max_{s'}
     \dfrac{1}{1-\gamma} 
     \left(
     \norm{x_{s'}'-x_{s'}}_1+\norm{y_{s'}'-y_{s'}}_1
     \right)
     \,+\,
     \gamma\max_{s'}  \norm{Q_{s'}^{x',y'} - Q_{s'}^{x,y}}_{\max}.
\end{array}
\]
Therefore,
\[
\begin{array}{rcl}
     && \!\!\!\! \!\!\!\! \!\!
     \displaystyle
     \max_s\norm{Q_s^{x',y'}
     - Q_s^{x,y}}_{\max}
     \\[0.2cm]
     & \leq & \displaystyle
     \gamma\max_{s'}
     \dfrac{1}{1-\gamma} 
     \left(
     \norm{x_{s'}'-x_{s'}}_1+\norm{y_{s'}'-y_{s'}}_1
     \right)
     \,+\,
     \gamma\max_{s'}  \norm{Q_{s'}^{x',y'} - Q_{s'}^{x,y}}_{\max}.
\end{array}
\]
which yields the desired result. 
\end{proof}

\section{Auxiliary Lemmas}
\label{ap: auxiliary lemmas}

In this section, we provide some auxiliary lemmas that are helpful in our analysis. 

\subsection{Auxiliary lemmas for potential functions}

\begin{lemma}\label{lem: bounded phimax}
    For any $N$-player Markov potential game with instantaneous reward bounded in $[0,1]$, it holds that
    \begin{align*}
        \left|\Phi^{\pi}(\mu) - \Phi^{\pi'}(\mu)\right| 
        \; \leq \;
        \frac{N}{1-\gamma}
    \end{align*}
    for any $\pi, \pi'\in\Pi$ and $\mu\in\Delta(\calS)$. 
\end{lemma}
\begin{proof}[\pfref{lem: bounded phimax}]
    By the potential property,
    \[
    \begin{array}{rcl}
         \Phi^{\pi}(\mu) \,-\, \Phi^{\pi'}(\mu) &  = & (\Phi^{\pi} - \Phi^{\pi_1',\pi_{-1}} ) 
         \,+\,
         (\Phi^{\pi_1',\pi_{-1}} - \Phi^{\pi_{\{1,2\}}',\pi_{-{\{1,2\}}}} )
         \,+\,
         \ldots
         \,+\,
         (\Phi^{\pi_N,\pi_{-N}'} - \Phi^{\pi'} )
         \\
         & = & (V_1^{\pi} - V_1^{\pi_1',\pi_{-1}} ) 
         \,+\,
         (V_2^{\pi_1',\pi_{-1}} - V_2^{\pi_{\{1,2\}}',\pi_{-{\{1,2\}}}} )
         \,+\,
         \ldots
         \,+\,
         (V_N^{\pi_N,\pi_{-N}'} - V_N^{\pi'} )
         \\
         & \leq & \displaystyle\frac{N}{1-\gamma}
    \end{array}
    \]
    where the last inequality is due to $V_i^{\pi} - V_i^{\pi'} \leq \frac{1}{1-\gamma}$ for any $\pi$ and $\pi'$. By symmetry, $\Phi^{\pi'}(\mu) - \Phi^{\pi}(\mu) \leq \frac{N}{1-\gamma}$. 
\end{proof}

\subsection{Auxiliary lemmas for single-player MDPs}

We provide some auxiliary lemmas in the context of single-player MDPs. 

\begin{lemma}[Action value function difference] \label{lem: Q diff lemma}
    Suppose that two MDPs have the same state/action spaces, but different reward and transition functions: $(r,p)$ and $(\tilde{r},\tilde{p})$. Then, for a given policy $\pi$, two action value functions associated with two MDPs satisfy  
    \begin{align*}
        \max_{s,\,a}
        |Q^{\pi}(s,a)-\tilde{Q}^\pi(s,a)|
        \; \leq \;
        \frac{1}{1-\gamma}
        \max_{s,\, a}|r(s,a)-\tilde{r}(s,a)| 
        \, + \, 
        \frac{\gamma}{(1-\gamma)^2} \max_{s,\, a} \left\|p(\cdot \,|\, s,a) - \tilde{p}(\cdot \,|\, s,a)\right\|_1.
    \end{align*}
\end{lemma}
\begin{proof}[\pfref{lem: Q diff lemma}]
     By the Bellman equations, 
     \[
     \begin{array}{rcl}
         Q^\pi(s,a) 
         & = & \displaystyle
         r(s,a) \, + \, \gamma \sum_{s',\, a'} p(s' \,|\, s,a) \, \pi(a'\, |\, s') \, Q^\pi(s',a') 
         \\
         \tilde{Q}^\pi(s,a) 
         & = & \displaystyle
         \tilde{r}(s,a) \, + \, \gamma \sum_{s',a'} \tilde{p}(s' \,|\, s,a) \, \pi(a' \,|\, s') \, \tilde{Q}^\pi(s',a').
     \end{array}
     \]
     Subtracting equalities above on both sides yields 
     \[
     \begin{array}{rcl}
         && \!\!\!\! \!\!\!\! \!\!
         |Q^\pi(s,a) - \tilde{Q}^\pi(s,a)| 
         \\[0.2cm]
         & \leq & \displaystyle
         |r(s,a) - \tilde{r}(s,a)| \, + \, \gamma \left|\sum_{s',\, a'}\left(p(s'\,|\,s,a) - \tilde{p}(s' \,|\, s,a)\right)\pi(a'\,|\,s')Q^\pi(s',a') \right| 
         \\[0.2cm]
         && \displaystyle
         +\, \gamma \left| \sum_{s',\, a'}\tilde{p}(s'\,|\,s,a)\pi(a'\,|\,s')\left(Q^\pi(s',a')-\tilde{Q}^\pi(s',a')\right)\right| 
         \\[0.2cm]
         & \leq & \displaystyle
         |r(s,a) - \tilde{r}(s,a)| \,+\, \frac{\gamma}{1-\gamma}\left\|p(\cdot\,|\,s,a) - \tilde{p}(\cdot\,|\,s,a)\right\|_1  \, + \, \gamma\max_{s',\, a'} \left|Q^\pi(s',a') - \tilde{Q}^\pi(s',a')\right|.
     \end{array}
     \]
     Taking the maximum over $(s,a)$ leads to  
     \[
     \begin{array}{rcl}
         & & \!\!\!\! \!\!\!\! \!\!
         \displaystyle
         \max_{s,\, a}|Q^\pi(s,a) - \tilde{Q}^\pi(s,a)| 
         \\[0.2cm]
         & \leq & \displaystyle
         \max_{s,\,a}|r(s,a) - \tilde{r}(s,a)| 
         \, + \,  \frac{\gamma}{1-\gamma}\max_{s,\, a}\norm{p(\cdot \,|\, s,a) - \tilde{p}(\cdot \,|\, s,a)}_1 + \gamma \max_{s,\,a}|Q^\pi(s,a) - \tilde{Q}^\pi(s,a)|   
     \end{array}
     \]
     which leads to the desired inequality after rearrangement.
\end{proof}

\begin{lemma}[Visitation measure difference]\label{lem: sum of occupancy diff}
    Let $\pi$ and $\pi'$ be two policies for a MDP, and $\mu$ be an initial state distribution. Then, 
    \[
    \sum_{s} \left|d^{\pi}_\mu(s) - d^{\pi'}_\mu(s)\right| 
    \; \leq \;
    \max_s  \norm{\pi(\cdot \,|\, s)-\pi'(\cdot \,|\, s)}_1.
    \]
\end{lemma}
\begin{proof}[\pfref{lem: sum of occupancy diff}]
     By the definition, for a fixed state $s^\sharp$,
     \[
        d^{\pi}_\mu(s^\sharp) 
        \; = \; 
        (1-\gamma)\E\left[\,\sum_{t \,=\, 0}^\infty \gamma^t \one_{\{ s_t\,=\,s^\sharp\}}~\Big|~  s_0\sim \mu, ~\pi\,\right].
     \]
     By taking reward function $r(s,a) = (1-\gamma) \one_{\{s\,=\,s^\sharp\}}$, we
     can view $d^{\pi}_\mu(s^\sharp)$ as a value function under the policy $\pi$ and the initial distribution $\mu$. With a slight abuse of notation, we denote such a value function by $V^{\pi}(\mu; s^\sharp) = d^{\pi}_\mu(s^\sharp)$. Similarly, we can define $V^\pi(s; s^\sharp)$ and $Q^\pi(s,a; s^\sharp)$, using the same reward function. 
     
     By the performance difference lemma (a single-player version of \pref{lem: performance difference}), 
     \[
        d^{\pi}_\mu(s^\sharp) - d^{\pi'}_\mu(s^\sharp) 
        \; = \;
        V^{\pi}(\mu; s^\sharp) - V^{\pi'}(\mu; s^\sharp) 
        \; = \;
        \sum_{s,\, a} d_\mu^{\pi}(s)\left(\pi(a \,|\, s) - \pi'(a \,|\, s)\right)Q^{\pi'}(s,a;s^\sharp). 
     \]
     Therefore, 
     \begin{equation}\label{eq: distribution diff}
         \sum_{s^\sharp}\left|d^{\pi}_\mu(s^\sharp) - d^{\pi'}_\mu(s^\sharp)\right| 
         \; \leq \;
         \sum_{s^\sharp}\sum_{s,a} d^{\pi}_\mu(s)
         \left|\pi(a\,|\,s) - \pi'(a\,|\,s)\right|
         Q^{\pi'}(s,a;s^\sharp).  
     \end{equation}
     
     We also note that $Q^{\pi'}(\cdot,\cdot;s^\sharp)$ is the action value function associated with the reward function $r(s,a)=(1-\gamma)\one_{\{s\,=\,s^\sharp\}}$. Thus,
     \[
     \begin{array}{rcl}
        \displaystyle
         \sum_{s^\sharp} Q^{\pi'}(s,a;s^\sharp) 
         & = & \displaystyle
         \sum_{s^\sharp} \E\left[(1-\gamma)\sum_{t\,=\,0}^\infty \gamma^{t}\one_{\{s_t\,=\,s^\sharp\}}~\Big|~(s_0,a_0)=(s,a),~\pi' \right] 
         \\[0.2cm]
         &=&\displaystyle 
         \E\left[(1-\gamma)\sum_{t\,=\,0}^\infty \gamma^{t}~\Big|~(s_0,a_0)=(s,a),~\pi' \right] 
         \\[0.2cm]
         &=& 1.
     \end{array}
     \]
     Therefore, we can arrange~\pref{eq: distribution diff} as follows,
     \[
     \begin{array}{rcl}
        \displaystyle
        \sum_{s^\sharp}\left|d^{\pi}_\mu(s^\sharp) - d^{\pi'}_\mu(s^\sharp)\right| 
        & \leq & \displaystyle \sum_{s,\,a} d^{\pi}_\mu(s)\left|\pi(a \,|\, s)-\pi'(a \,|\, s)\right| 
        \\[0.2cm]
        & = & \displaystyle
        \sum_s d^\pi_\mu(s)\norm{\pi(\cdot \,|\, s)-\pi'(\cdot \,|\, s)}_1 
        \\[0.2cm]
        &\leq & \displaystyle
        \max_s  \norm{\pi(\cdot \,|\, s)-\pi'(\cdot \,|\, s)}_1. 
     \end{array}
     \]
\end{proof}

\subsection{Auxiliary lemmas for multi-player MDPs}

We first extend \pref{lem: performance difference} in the 1st-order form to the 2nd-order performance difference, which is useful to measure the joint policy improvement from multiple players.

\begin{lemma}[The 2nd-order performance difference]\label{lem: second-order PDL}
    Consider a two-player common-payoff Markov game with state space $\calS$ and action sets $\calA_1$, $\calA_2$. Let $r: \calS\times \calA_1\times \calA_2\rightarrow [0, 1]$ be the reward function, and $p: \calS\times \calA_1\times \calA_2\rightarrow \Delta(\calS)$ be the transition function. Let $\Pi_1=(\Delta(\calA_1))^{\calS}$ and  $\Pi_2=(\Delta(\calA_2))^{\calS}$ be player~1 and player~2's policy sets, respectively. Then, for any $x, x'\in\Pi_1$ and $y, y'\in \Pi_2$,
    \begin{align*}
        &V^{x, y}(\mu) \,-\, V^{x', y}(\mu) \,-\, V^{x, y'}(\mu) \,+\, V^{x' ,y'}(\mu) \\
        &\leq\; \frac{2\kappa_\mu^2  A}{(1-\gamma)^4}\sum_s d_\mu^{x',y'}(s)\left(\norm{x(\cdot\,|\,s) - x'(\cdot\,|\,s)}^2 + \norm{y(\cdot\,|\,s) - y'(\cdot\,|\,s)}^2\right)
    \end{align*}
    where $\kappa_\mu$ is the distribution mismatch coefficient relative to $\mu$ (see $\kappa_\mu$ in \pref{def: distribution mismatch coeff}).
\end{lemma}
\begin{proof}[Proof of \pref{lem: second-order PDL}]
     We define the following non-stationary policies: 
     \begin{align*}
         \bar{x}_i:~ \text{a Player 1's policy where in steps from $0$ to $i-1$, $x'$ is executed; in steps from $i$ to $\infty$, $x$ is executed. }
     \end{align*}
     With this definition, $\bar{x}_0=x$ and $\bar{x}_\infty=x'$. We define $\bar{y}_i$ similarly. Since $\bar{x}_i$ is non-stationary, we specify its action distribution as $\bar{x}_i(\cdot~|~s, h)$ where $h$ is the step index. The joint value function for these non-stationary policies can be defined as usual: 
     \begin{align*}
         V^{\bar{x}_i, \bar{y}_j}(\mu) \;\DefinedAs\; 
         \E\left[\,
         \sum_{t\,=\,0}^\infty  \gamma^{t}r(s_t, a_t, b_t)~\bigg|~ s_0\sim \mu,\; a_t\sim \bar{x}_i(\cdot\,|\,s_t, t),\; b_t\sim \bar{y}_j(\cdot\,|\,s_t, t)
         \,\right]. 
     \end{align*}
     For simplicity, we omit the initial distribution $\mu$ in writing the value function. 
     We first show that for any $H\in\mathbb{N}$, 
     \begin{align*}
         V^{\bar{x}_0, \bar{y}_0} 
         \,-\,
         V^{\bar{x}_H, \bar{y}_0} 
         \,-\,
         V^{\bar{x}_0, \bar{y}_H} 
         \,+\, 
         V^{\bar{x}_H, \bar{y}_H} 
         \;=\;
         \sum_{i\,=\,0}^{H-1}\sum_{j\,=\,0}^{H-1} \left(V^{\bar{x}_i, \bar{y}_j} 
         \,-\, 
         V^{\bar{x}_{i+1}, \bar{y}_j} 
         \,-\,
         V^{\bar{x}_{i}, \bar{y}_{j+1}} \,+\, 
         V^{\bar{x}_{i+1}, \bar{y}_{j+1}}\right).
     \end{align*}
     In fact, the right-hand side above is equal to
     \begin{align*}
         & \sum_{j\,=\,0}^{H-1}\sum_{i\,=\,0}^{H-1}\left(V^{\bar{x}_i, \bar{y}_j} - V^{\bar{x}_{i+1}, \bar{y}_j}\right) 
         \,+\, 
         \sum_{j\,=\,0}^{H-1}\sum_{i\,=\,0}^{H-1} \left( - V^{\bar{x}_{i}, \bar{y}_{j+1}} + 
         V^{\bar{x}_{i+1}, \bar{y}_{j+1}}\right) \\
         &=\; \sum_{j\,=\,0}^{H-1}\left(V^{\bar{x}_0, \bar{y}_j} - V^{\bar{x}_H, \bar{y}_j}\right) 
         \,+\,
         \sum_{j\,=\,0}^{H-1}\left( -V^{\bar{x}_0, \bar{y}_{j+1}} + V^{\bar{x}_H, \bar{y}_{j+1}} \right) \\
         &=\; \sum_{j\,=\,0}^{H-1}\left(V^{\bar{x}_0, \bar{y}_j}  -V^{\bar{x}_0, \bar{y}_{j+1}}\right) 
         \,+\, \sum_{j\,=\,0}^{H-1}\left(- V^{\bar{x}_H, \bar{y}_j} + V^{\bar{x}_H, \bar{y}_{j+1}} \right) \\
         &=\; V^{\bar{x}_0, \bar{y}_0}  \,-\,
         V^{\bar{x}_0, \bar{y}_{H}} 
         \,-\,
         V^{\bar{x}_H, \bar{y}_{0}} 
         \,+\,
         V^{\bar{x}_H, \bar{y}_{H}}.
     \end{align*}
     Sending $H$ to infinity and recalling that $\bar{x}_0=x$, $\bar{x}_\infty=x'$, $\bar{y}_0=y$, $\bar{y}_\infty=y'$ lead to 
     \begin{align*}
         V^{x, y} 
         \,-\,
         V^{x', y} 
         \,-\,
         V^{x, y'} 
         \,+\, 
         V^{x' ,y'} 
         \;=\;
         \sum_{i\,=\,0}^{\infty}\sum_{j\,=\,0}^{\infty} \left(V^{\bar{x}_i, \bar{y}_j} 
         \,-\,
         V^{\bar{x}_{i+1}, \bar{y}_j} 
         \,-\,
         V^{\bar{x}_{i}, \bar{y}_{j+1}} \,+\, 
         V^{\bar{x}_{i+1}, \bar{y}_{j+1}}\right). 
     \end{align*}
     
     We next focus on the particular summand above with index $(i,j)$ and discuss three cases.
     
     \paragraph{Case 1: $i<j$. } We first re-write $V^{\bar{x}_i, \bar{y}_j} - V^{\bar{x}_{i+1}, \bar{y}_j}$. Notice that the value difference between the policy pairs $(\bar{x}_i, \bar{y}_j)$ and $(\bar{x}_{i+1}, \bar{y}_j)$ starts at step $i$, since  both policy pairs are equal to $(x', y')$ from step $0$ to step $i-1$. At the $i$th step, $\bar{x}_i$ changes to $x$ while $\bar{x}_{i+1}$ remains as $x'$. Therefore, 
     \begin{align*}
         &V^{\bar{x}_i, \bar{y}_j} \,-\, V^{\bar{x}_{i+1}, \bar{y}_j} \\
         &= \; \frac{1}{1-\gamma}\sum_{s,a,b} d_\mu^{x',y'}(s; i) x(a\,|\,s)y'(b\,|\,s)\left( r(s,a,b) + \E\left[ \sum_{t\,=\,i+1}^\infty \gamma^{t-i} r(s_t,a_t,b_t) ~\bigg|~  s_{i+1}\sim p(\cdot\,|\,s,a,b), \bar{x}_i, \bar{y}_j\right]\right) \\
         &\quad - \frac{1}{1-\gamma}\sum_{s,a,b} d_\mu^{x',y'}(s; i) x'(a\,|\,s)y'(b\,|\,s)\left( r(s,a,b) + \E\left[ \sum_{t\,=\,i+1}^\infty \gamma^{t-i} r(s_t,a_t,b_t) ~\bigg|~  s_{i+1}\sim p(\cdot\,|\,s,a,b), \bar{x}_i, \bar{y}_j\right]\right) \\
         &= \; \frac{1}{1-\gamma}\sum_{s,a,b} d_\mu^{x',y'}(s; i) \Big(x(a\,|\,s) - x'(a\,|\,s)\Big)y'(b\,|\,s)\Bigg(r(s,a,b) 
         \\
         & \qquad \qquad \qquad \qquad \qquad \qquad \qquad \qquad
         \qquad \qquad 
         +\, \E\left[ \sum_{t\,=\,i+1}^\infty \gamma^{t-i} r(s_t,a_t,b_t) ~\bigg|~  s_{i+1}\sim p(\cdot\,|\,s,a,b), \bar{x}_i, \bar{y}_j\right]\Bigg)
     \end{align*}
     where we define 
     \begin{align*}
         d_\mu^{x,y}(s; i) \; = \; (1-\gamma)\E\left[\gamma^i\one[s_i=s]~|~s_0\sim \mu\right]. 
     \end{align*}
     (note that $d_\mu^{x,y}(s) = \sum_{i\,=\,0}^\infty d_\mu^{x,y}(s;i)$). \\
     Similarly, 
     \begin{align*}
         &V^{\bar{x}_i, \bar{y}_{j+1}} \,-\, V^{\bar{x}_{i+1}, \bar{y}_{j+1}}\\
         &= \; \frac{1}{1-\gamma}\sum_{s,a,b} d_\mu^{x',y'}(s; i) \Big(x(a\,|\,s) - x'(a\,|\,s)\Big)y'(b|s)\Bigg(r(s,a,b) \\
         &
          \qquad \qquad \qquad \qquad
          \qquad \qquad \qquad \qquad
          \qquad \qquad
         +\, \E\left[ \sum_{t\,=\,i+1}^\infty \gamma^{t-i} r(s_t,a_t,b_t) ~\bigg|~  s_{i+1}\sim p(\cdot\,|\,s,a,b), \bar{x}_i, \bar{y}_{j+1}\right]\Bigg).
     \end{align*}
     We notice that the  following difference: 
     \begin{align*}
         \E\left[ \sum_{t\,=\,i+1}^\infty \gamma^{t-i} r(s_t,a_t,b_t) ~\bigg|~  s_{i+1}\sim p(\cdot\,|\,s,a,b), \bar{x}_i, \bar{y}_j\right] 
         \,-\,
         \E\left[ \sum_{t\,=\,i+1}^\infty \gamma^{t-i} r(s_t,a_t,b_t) ~\bigg|~  s_{i+1}\sim p(\cdot\,|\,s,a,b), \bar{x}_i, \bar{y}_{j+1}\right]
     \end{align*}
     is equivalent to 
     \begin{align*}
         \frac{\gamma}{1-\gamma}\sum_{\tilde{s},\tilde{a}, \tilde{b}} d_{p(\cdot\,|\,s,a,b)}^{x,y'}(\tilde{s}; j-i-1)  x(\tilde{a}\,|\,\tilde{s})\Big(y(\tilde{b}\,|\,\tilde{s}) - y'(\tilde{b}\,|\,\tilde{s})\Big)Q^{x,y}(\tilde{s},\tilde{a},\tilde{b}).
     \end{align*}
     Hence,
     \begin{align*}
         &V^{\bar{x}_i, \bar{y}_j} - V^{\bar{x}_{i+1}, \bar{y}_j} - V^{\bar{x}_i, \bar{y}_{j+1}} + V^{\bar{x}_{i+1}, \bar{y}_{j+1}} 
         \\
         &= \; \frac{\gamma}{(1-\gamma)^2}\sum_{s,a,b}\sum_{\tilde{s},\tilde{a}, \tilde{b}} d_\mu^{x',y'}(s; i)d_{p(\cdot\,|\,s,a,b)}^{x,y'}(\tilde{s}; j-i-1)  \Big(x(a\,|\,s) - x'(a\,|\,s)\Big)y'(b\,|\,s)  x(\tilde{a}\,|\,\tilde{s})
         \\
         & \qquad\qquad\qquad\qquad
         \qquad\qquad\qquad\qquad
         \qquad\qquad\qquad\qquad
         \times\,         
         \Big(y(\tilde{b}\,|\,\tilde{s}) - y'(\tilde{b}\,|\,\tilde{s})\Big)Q^{x,y}(\tilde{s},\tilde{a},\tilde{b}) 
         \\
         &\leq \; \frac{\gamma}{2(1-\gamma)^3}\sum_{s,a,b}\sum_{\tilde{s},\tilde{a}, \tilde{b}} d_\mu^{x',y'}(s; i)d_{p(\cdot\,|\,s,a,b)}^{x,y'}(\tilde{s}; j-i-1) y'(b\,|\,s)  x(\tilde{a}\,|\,\tilde{s}) \Big(x(a\,|\,s) - x'(a\,|\,s)\Big)^2 
         \\
         &\qquad +\, \frac{\gamma }{2(1-\gamma)^3}\sum_{s,a,b}\sum_{\tilde{s},\tilde{a}, \tilde{b}} d_\mu^{x',y'}(s; i)d_{p(\cdot\,|\,s,a,b)}^{x,y'}(\tilde{s}; j-i-1)  y'(b\,|\,s)  x(\tilde{a}\,|\,\tilde{s})\Big(y(\tilde{b}\,|\,\tilde{s}) - y'(\tilde{b}\,|\,\tilde{s})\Big)^2   \tag{bounding $|Q^{x,y}(\cdot,\cdot,\cdot)|$ by $\frac{1}{1-\gamma}$ and using AM-GM}
         \\
         &=\; \frac{\gamma A}{2(1-\gamma)^3}\sum_{s,a}\sum_{\tilde{s}} d_\mu^{x',y'}(s; i)d_{p(\cdot\,|\,s,a,y')}^{x,y'}(\tilde{s}; j-i-1)  \Big(x(a\,|\,s) - x'(a\,|\,s)\Big)^2  \tag{define $p(\cdot\,|\,s,a,y) = \sum_b p(\cdot\,|\,s,a,b)y(b|s)$}
         \\
         &\qquad +\, \frac{\gamma A}{2(1-\gamma)^3}\sum_{s}\sum_{\tilde{s}, \tilde{b}} d_\mu^{x',y'}(s; i)d_{p(\cdot\,|\,s,\text{unif},y')}^{x,y'}(\tilde{s}; j-i-1)   \Big(y(\tilde{b}\,|\,\tilde{s}) - y'(\tilde{b}\,|\,\tilde{s})\Big)^2.   \tag{define uniform distribution $\text{unif}=\frac{1}{A}\one$}
     \end{align*}
     Summing the inequality above over $i<j$ yields
     \begin{align*}
         &\sum_{i\,=\,0}^\infty\sum_{j\,=\,i+1}^\infty \left(V^{\bar{x}_i, \bar{y}_j} - V^{\bar{x}_{i+1}, \bar{y}_j} - V^{\bar{x}_i, \bar{y}_{j+1}} + V^{\bar{x}_{i+1}, \bar{y}_{j+1}}\right) \\
         &\leq\; \frac{\gamma A}{2(1-\gamma)^3}\sum_{i\,=\,0}^\infty \sum_{s,a} d_\mu^{x',y'}(s;i) \Big(x(a\,|\,s) - x'(a\,|\,s)\Big)^2\left(\sum_{\tilde{s}}\sum_{j\,=\,i+1}^\infty d^{x,y'}_{p(\cdot\,|\,s,a,y')}(\tilde{s}; j-i-1) \right) \\
         &\qquad +\, \frac{\gamma A}{2(1-\gamma)^3} \sum_{i\,=\,0}^\infty\sum_{\tilde{s}, \tilde{b}} \sum_s d_\mu^{x',y'}(s;i) \Big(y(\tilde{b}\,|\,\tilde{s}) - y'(\tilde{b}\,|\,\tilde{s})\Big)^2 \left( \sum_{j\,=\,i+1}^\infty  d^{x,y'}_{p(\cdot\,|\,s,  \text{unif}, y')}(\tilde{s}; j-i-1)\right) 
         \\
         &= \; \frac{\gamma A}{2(1-\gamma)^3}\sum_{s,a} d_\mu^{x',y'}(s) \Big(x(a\,|\,s) - x'(a\,|\,s)\Big)^2 \\
         &\qquad +\, \frac{\gamma A}{2(1-\gamma)^3} \sum_{\tilde{s}, \tilde{b}}\sum_s d_{\mu}^{x',y'}(s) d^{x,y'}_{p(\cdot\,|\,s,  \text{unif}, y')}(\tilde{s})\Big(y(\tilde{b}\,|\,\tilde{s}) - y'(\tilde{b}\,|\,\tilde{s})\Big)^2 \tag{using the property: $\sum_{i\,=\,0}^\infty d_\mu^{x,y}(s;i) = d_\mu^{x,y}(s)$}\\
         &= \; \frac{\gamma A}{2(1-\gamma)^3}\sum_s d_\mu^{x',y'}(s)\norm{x(\cdot\,|\,s) - x'(\cdot\,|\,s)}^2 + \frac{\gamma A}{2(1-\gamma)^3} \sum_s d^{x, y'}_{\mu'}(s)\norm{y(\cdot\,|\,s) - y'(\cdot\,|\,s)}^2 
     \end{align*}
     where $\mu'$ is a state distribution that generates the state by the following procedure: first sample a state $s_0$ according to $d_\mu^{x',y'}(\cdot)$, then execute $(\text{unif}, y')=(\frac{1}{A}\one, y')$ for one step, and then output the next state.  
     
     By \pref{lem: mu and mu'} (with $\pi=(x',y')$, $\pi'=(x,y')$, and $\bar{\pi}=(\text{unif}, y')$), we have  $\frac{d_{\mu'}^{x, y'}(s)}{d_\mu^{x',y'}(s)}\leq \frac{d_{\mu'}^{x, y'}(s)}{\mu(1-\gamma)}\leq \frac{\kappa_\mu^2}{\gamma(1-\gamma)}$. 
     Therefore,  
     \begin{align*}
         &\sum_{i\,=\,0}^\infty\sum_{j\,=\,i+1}^\infty \left(V^{\bar{x}_i, \bar{y}_j} - V^{\bar{x}_{i+1}, \bar{y}_j} - V^{\bar{x}_i, \bar{y}_{j+1}} + V^{\bar{x}_{i+1}, \bar{y}_{j+1}}\right) 
         \\
         &\leq\; \frac{\kappa_\mu^2  A}{2(1-\gamma)^4}\sum_s d_\mu^{x',y'}(s)\left(\norm{x(\cdot\,|\,s) - x'(\cdot\,|\,s)}^2 + \norm{y(\cdot\,|\,s) - y'(\cdot\,|\,s)}^2\right). 
     \end{align*}
     
     \paragraph{Case 2: $i>j$.} This case is symmetric to the case of $i<j$, and can be handled similarly. 
     \paragraph{Case 3: $i=j$.} In this case, 
     \begin{align*}
         &\sum_{i\,=\,0}^\infty \left( V^{\bar{x}_i, \bar{y}_i} - V^{\bar{x}_{i+1}, \bar{y}_i} - V^{\bar{x}_i, \bar{y}_{i+1}} + V^{\bar{x}_{i+1}, \bar{y}_{i+1}}\right)  \\
         &=\;\frac{1}{1-\gamma}\sum_{i\,=\,0}^\infty \sum_{s,a,b}d_\mu^{x',y'}(s;i)\Big(x'(a\,|\,s) - x(a\,|\,s)\Big)\Big(y'(a\,|\,s) - y(a\,|\,s)\Big)Q^{x,y}(s,a,b) \\
         &=\; \frac{1}{1-\gamma} \sum_{s,a,b}d_\mu^{x',y'}(s)\Big(x'(a|s) - x(a|s)\Big)\Big(y'(a|s) - y(a|s)\Big)Q^{x,y}(s,a,b) \\
         &\leq\; \frac{1}{2(1-\gamma)^2} \sum_{s,a,b}d_\mu^{x',y'}(s)\Big(x'(a\,|\,s) - x(a\,|\,s)\Big)^2 + \frac{1}{2(1-\gamma)^2} \sum_{s,a,b}d_\mu^{x',y'}(s)\Big(y'(a\,|\,s) - y(a\,|\,s)\Big)^2 \\
         &=\; \frac{A}{2(1-\gamma)^2}\sum_s d_\mu^{x',y'}(s)\norm{x'(\cdot\,|\,s) - x(\cdot\,|\,s)}^2 + \frac{A}{2(1-\gamma)^2}\sum_s d_\mu^{x',y'}(s)\norm{y'(\cdot\,|\,s) - y(\cdot\,|\,s)}^2. 
     \end{align*}
     Summing the bounds in all three cases above completes the proof. 
\end{proof}

\begin{lemma}\label{lem: mu and mu'}
    Let $\pi$, $\pi'$ and $\bar{\pi}$ be three policies, and $\mu$ be some initial distribution. Let $\mu'$ be a state distribution that generates a state according to the following: first sample an $s_0$ from $d_\mu^{\pi}(\cdot)$, then execute $\bar{\pi}$ for one step, and then output the next state. Then,
    \begin{align*}
        \norm{\frac{d^{\pi'}_{\mu'}}{\mu}}_\infty 
        \; \leq \; \frac{\kappa_\mu^2}{\gamma} \;\triangleq\; \frac{1}{\gamma}\left(\sup_{\tilde{\pi}}\norm{\frac{d^{\tilde{\pi}}_\mu}{\mu}}_\infty \right)^2.
    \end{align*}
\end{lemma}
\begin{proof}[Proof of \pref{lem: mu and mu'}]
     For a particular state $s^\sharp$, we view the supremum $\sup_{\tilde{\pi}}\frac{d_\mu^{\tilde{\pi}}(s^\sharp)}{\mu(s^\sharp)}$ as the optimal value of an MDP whose reward function is $r(s,a) = \frac{1-\gamma}{\mu(s^\sharp)}\one[s=s^\sharp]$ and initial state is generated by $\mu$. The optimal value of this MDP is upper bounded by $\kappa_\mu$ by \pref{def: distribution mismatch coeff}. We next consider the following non-stationary policy for this MDP: first execute $\bar{\pi}$ for one step, and then execute $\pi'$ in the rest of the steps. The discounted value of this non-stationary policy is lower bounded by
     \begin{align*}
         \gamma \sum_s \Pr\left(s_1=s~|~s_0\sim \mu, \,a_0\sim \bar{\pi}(\cdot\,|\,s_0)\right)\times \frac{d_s^{\pi'}(s^\sharp)}{\mu(s^\sharp)} 
         \;=\;
         \gamma \sum_{s_0, a_0, s} \mu(s_0)\bar{\pi}(a_0|s_0) p(s~|s_0, a_0)\times \frac{d_s^{\pi'}(s^\sharp)}{\mu(s^\sharp)}. 
     \end{align*}
     We can upper and lower bound the discounted sum above as the following: 
     \begin{align*}
         \frac{\gamma}{\kappa_\mu}\sum_{s_0, a_0, s} d_{\mu}^{\pi}(s_0) \bar{\pi}(a_0|s_0) p(s~|s_0, a_0)\times \frac{d_s^{\pi'}(s^\sharp)}{\mu(s^\sharp)}  
         \;\leq\;
         \gamma \sum_{s_0, a_0, s} \mu(s_0) \bar{\pi}(a_0|s_0) p(s~|s_0, a_0)\times \frac{d_s^{\pi'}(s^\sharp)}{\mu(s^\sharp)} 
         \;\leq\;
         \kappa_\mu. 
     \end{align*}
     where the right inequality is due to that this discounted value must be upper bounded by the optimal value of this MDP, which has an upper bound $\kappa_\mu$, and the left inequality is by the definition of $\kappa_\mu$. Now notice that 
     \begin{align*}
         \mu'(s) 
         \;=\;
         \sum_{s_0, a_0}d^{\pi}_\mu(s_0)\bar{\pi}(a_0|s_0)p(s|s_0, a_0)
     \end{align*}
     by the definition of $\mu'$. Plugging this into the previous inequality, we get
     \begin{align*}
         \frac{\gamma}{\kappa_\mu}\times \frac{d_{\mu'}^{\pi'}(s^\sharp)}{\mu(s^\sharp)}
         \;\leq\;
         \kappa_\mu. 
     \end{align*}
     Since this holds for any $s^\sharp$, this gives 
     \begin{align*}
         \norm{\frac{d^{\pi'}_{\mu'}}{\mu}}_\infty
         \;\leq\;
        \frac{\kappa_\mu^2}{\gamma}. 
     \end{align*}
\end{proof}

\section{Auxiliary Lemmas for Stochastic Projected Gradient Descent }
\label{ap: stochastic projected gradient descent}

\pref{alg: PMA fa} serves a sample-based algorithm if we solve the empirical risk minimization problem \pref{eq:  linear regression} via a stochastic projected gradient descent,
\begin{equation}\label{eq: stochastic projected gradient descent}
{w}_i^{(k+1)} \;=\; \mathcal{P}_{\norm{w}\,\leq\,W}
\left(
{w}_i^{(k)} \, - \, \lambda^{(k)}\, \hat{\nabla}_i^{(t)}(s^{(k)}, a_i^{(k)}))
\right)
\end{equation}
where $\hat{\nabla}_i^{(t)} \DefinedAs 2( \langle \phi_i, w_i^{(k)} \rangle - R_i^{(k)} )  \phi_i$ is the $k$th gradient of \pref{eq:  linear regression} and $\lambda^{(k)}>0$ is the stepsize.  We assume that the smallest eigenvalue of correlation matrix $\mathbb{E}_{s, a_i} \left[\phi_i(s,a_i) \phi_i(s,a_i)^\top \right]$ is positive. 

For a constrained convex optimization, $\minimize_{w\,\in\,\{ w\,\vert\, \norm{w}\,\leq\,W\}} f(w)$, where $f(w)$ is a convex function and $W>0$, we consider a basic method for solving this problem: the stochastic projected gradient descent in \pref{alg: stochastic projected gradient descent with weighted averaging}, where $\mathcal{P}_{\norm{w}\,\leq\,W}$ is a Euclidean projection in $\mathbb{R}^d$ to the constraint set $\norm{w}\leq W$.

\begin{algorithm}[!t]
	\caption{ Stochastic projected gradient descent with weighted averaging}
	\label{alg: stochastic projected gradient descent with weighted averaging}
	\begin{algorithmic}[1]
		\STATE \textbf{Parameters:} $W$, $\lambda^{(k)}$, and $\beta_{k}^{(K)}$.
		\STATE \textbf{Input}: Stepsize $\alpha$, total number of iterations $K>0$.
		\STATE \textbf{Initialization}: $w^{(0)} = 0$.
		\FOR{step $k=1,\ldots,K$} 
		\STATE Draw $\nabla^{(k)}$ form a distribution such that $\mathbb{E}[ \nabla^{(k)} \,\vert\, w^{(k)}] \in\partial f(w^{(k)})$.
		\STATE Update $w^{(k+1)} = \mathcal{P}_{\norm{w}\,\leq\,W}\left( w^{(k)} - \lambda^{(k)} \nabla^{(k)} \right)$.
		\ENDFOR
		\STATE \textbf{Output}: $ \sum_{k\,=\,0}^K\beta_{k}^{(K)}w^{(k)}$.
	\end{algorithmic}
\end{algorithm}

\begin{lemma}\label{lem: stochastic projected gradient descent with weighted averaging}
    Let $w^\star \DefinedAs \argmin_{w\,\in\,\{ w\,\vert\, \norm{w}\,\leq\,W\}} f(w)$. Suppose $\text{\normalfont Var}(\nabla^{(k)}) \leq \sigma^2$. If we run \pref{alg: stochastic projected gradient descent with weighted averaging} with stepsize $\lambda^{(k)} = O(\frac{1}{1+k})$ and $\beta_k^{(K)} = \frac{1/\lambda^{(k)}}{\sum_{r\,=\,0}^K 1/\lambda^{(r)}}$, then,
    \[
        \mathbb{E}\left[ f\left( \sum_{k\,=\,0}^K \beta_k^{(K)}w^{(k)}\right)\right] 
        -
        f(w^\star)
        \; \lesssim\; 
        \frac{\sigma^2 W^2 d}{K}.
    \]
\end{lemma}
\begin{proof}[Proof of \pref{lem: stochastic projected gradient descent with weighted averaging}]
    See the proof of Theorem 1 in \citep{cohen2017projected}.
\end{proof}

\section{Additional Experiments}
\label{ap: experiments}

We provide details about our experiments as follows. 

For illustration, we consider the state space $\calS = \{\textit{safe}, \textit{distancing}\}$ and action space $\calA_i = \{A,B,C,D\}$, and the number of players $N=8$. In each state $s\in \calS$, the reward for player $i$ taking an action $a \in \calA_i$ is the $w_s^a$-weighted number of players using the action $a$, where $w_s^a$ specifies the action preference $w_s^A<w_s^B<w_s^C<w_s^D$. The reward in state {\it distancing} is less than that in state {\it safe} by a large amount $c>0$. For state transition, if more than half of players find themselves using the same action, then the state transits to the state {\it distancing}; transition back to the state {\it safe} whenever no more than half of players take the same action.

In our experiments, we implement our independent policy gradient method based on the code for the projected stochastic gradient ascent \citep{leonardos2021global}. At each iteration, we collect a batch of $20$ trajectories to estimate the action-value function and (or) the stationary state distribution under current policy. We choose the discount factor $\gamma = 0.99$, and different the stepsize $\eta$, and initial state distributions as we report next. 

Continuing \pref{sec: experiments}, we further report our computational results using stepsize $\eta = 0.001$ in \pref{fig: policy distance and trajectories unif1}, larger stepsize $\eta = 0.002$ in \pref{fig: policy distance and trajectories unif2} and stepsize $\eta = 0.005$ in \pref{fig: policy distance and trajectories unif5}. We notice that the stepsize $\eta = 0.001$ for the projected stochastic gradient ascent \citep{leonardos2021global} does not yield convergence while our independent policy gradient converges as shown in \pref{fig: policy distance and trajectories unif1}. As demonstrated in \pref{sec: experiments}, our independent policy gradient permits larger stepsizes with fast convergence, e.g., $\eta = 0.002$ in \pref{fig: policy distance and trajectories unif2} and $\eta = 0.005$ in \pref{fig: policy distance and trajectories unif5}. Compared \pref{fig: policy distance and trajectories unif2} with \pref{fig: policy distance and trajectories unif5}, we see an improved convergence of our independent policy gradient using a larger stepsize. We also remark that the learnt policies for all these experiments can generate the same Nash policy that matches the result in \citet{leonardos2021global}.

\begin{figure}
		\begin{center}
		    \begin{tabular}{cccc}
		    & (a) & & (b)
		    \\
			{\rotatebox{90}{ \;\;\;\; \;\;\;\; \;\;\;\; \; \; accuracy}} 
			\!\!\!\!\!\!
			& {\includegraphics[scale=0.55]{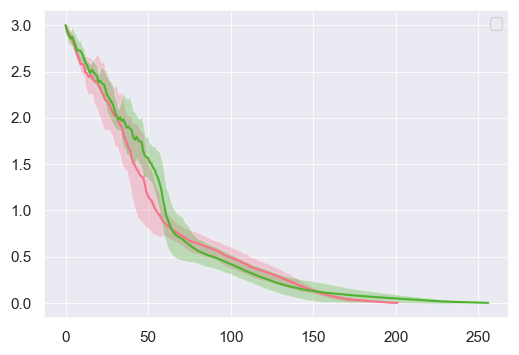}}
			\quad&\quad
			{\rotatebox{90}{ \;\;\;\; \;\;\;\; \;\;\;\; \; \;  accuracy}} 
			\!\!\!\!\!\!
			& {\includegraphics[scale=0.55]{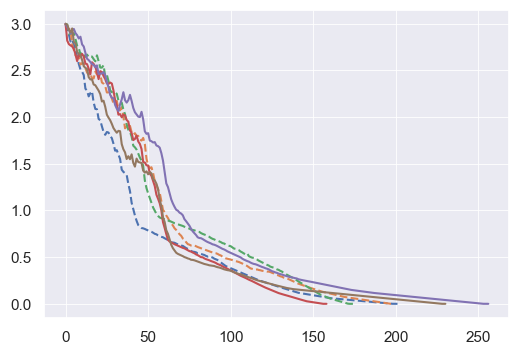}}
			\\[-0.1cm]
			& {\centering \;\;\;\; iteration} & & {\centering \;\;\;\; iteration}
		    \end{tabular}
		    \begin{tabular}{cc}
		        & (c)
		        \\
			    {\rotatebox{90}{ \;\;\;\; \;\;\;\; \;\;\;\; \; \# players}} 
			    \!\!\!\!\!\!
			    & {\includegraphics[scale=0.5]{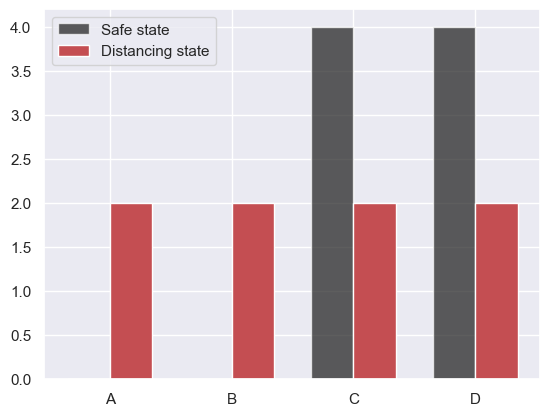}}
			    \\[-0.1cm]
			    & \;\; action
		    \end{tabular}
		\end{center}
	\caption{ Convergence performance.
	(a) Learning curves for our independent policy gradient (\textbf{\color{applegreen}---}) with stepsize $\eta=0.001$ and the projected stochastic gradient ascent (\textbf{\color{awesome}---}) with $\eta=0.0001$ \citep{leonardos2021global}. Each solid line is the mean of trajectories over three random seeds and each shaded region displays the confidence interval.
	(b) Learning curves for six individual runs of our independent policy gradient (solid line) and the projected stochastic gradient ascent (dash line) three each. 
	(c) Distribution of players in one of two states taking four actions.
	In (a) and (b), we measure the accuracy by the absolute distance of each iterate to the converged Nash policy, i.e., $\frac{1}{N}\sum_{i\,=\,1}^N \Vert{\pi_i^{(t)}-\pi_i^{\text{\normalfont Nash}}}\Vert_1$. In our computational experiments, the initial distribution $\rho$ is uniform.
	}
	\label{fig: policy distance and trajectories unif1}
\end{figure}

\begin{figure}
		\begin{center}
		    \begin{tabular}{cccc}
		    & (a) & & (b)
		    \\
			{\rotatebox{90}{ \;\;\;\; \;\;\;\; \;\;\;\; \; \;  accuracy}} 
			\!\!\!\!\!\!
			& {\includegraphics[scale=0.55]{./figures/Fig_avg_runs_n8_unif2}}
			\quad &\quad
			{\rotatebox{90}{ \;\;\;\; \;\;\;\; \;\;\;\; \; \; accuracy}} 
			\!\!\!\!\!\!
			& {\includegraphics[scale=0.55]{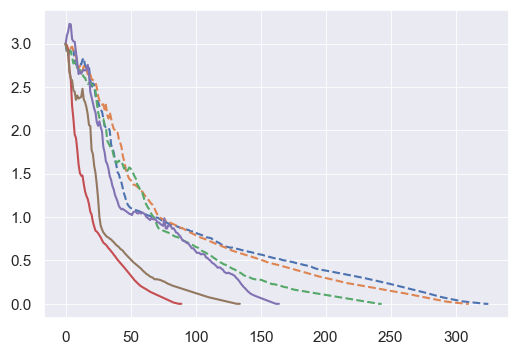}}
			\\[-0.1cm]
			& {\centering \;\;\;\; iteration} & & {\centering \;\;\;\; iteration}
		    \end{tabular}
		    \begin{tabular}{cc}
		        & (c)
		        \\
			    {\rotatebox{90}{ \;\;\;\; \;\;\;\; \;\;\;\; \;\;\;\; \# players}} 
			    \!\!\!\!\!\!
			    & {\includegraphics[scale=0.5]{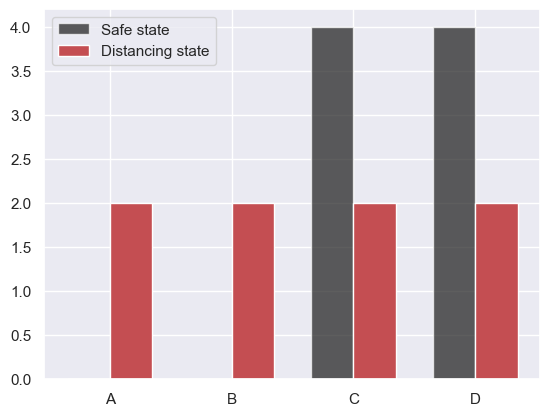}}
			    \\[-0.1cm]
			    & \;\; action
		    \end{tabular}
		\end{center}
	\caption{ Convergence performance.
	(a) Learning curves for our independent policy gradient (\textbf{\color{applegreen}---}) with stepsize $\eta=0.002$ and the projected stochastic gradient ascent (\textbf{\color{awesome}---}) with $\eta=0.0001$ \citep{leonardos2021global}. Each solid line is the mean of trajectories over three random seeds and each shaded region displays the confidence interval.
	(b) Learning curves for six individual runs of our independent policy gradient (solid line) and the projected stochastic gradient ascent (dash line) three each. 
	(c) Distribution of players in one of two states taking four actions.
	In (a) and (b), we measure the accuracy by the absolute distance of each iterate to the converged Nash policy, i.e., $\frac{1}{N}\sum_{i\,=\,1}^N \Vert{\pi_i^{(t)}-\pi_i^{\text{\normalfont Nash}}}\Vert_1$. In our computational experiments, the initial distribution $\rho$ is uniform.
	}
	\label{fig: policy distance and trajectories unif2}
\end{figure}

\begin{figure}
		\begin{center}
		    \begin{tabular}{cccc}
		    & (a) & & (b)
		    \\
			{\rotatebox{90}{ \;\;\;\; \;\;\;\; \;\;\;\; \; \;  accuracy}} 
			\!\!\!\!\!\!
			& {\includegraphics[scale=0.55]{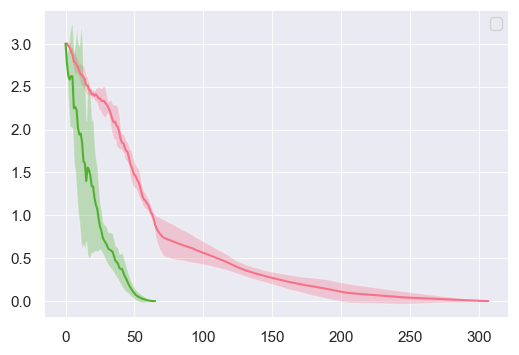}}
			\quad & \quad
			{\rotatebox{90}{ \;\;\;\; \;\;\;\; \;\;\;\; \; \;  accuracy}} 
			\!\!\!\!\!\!
			& {\includegraphics[scale=0.55]{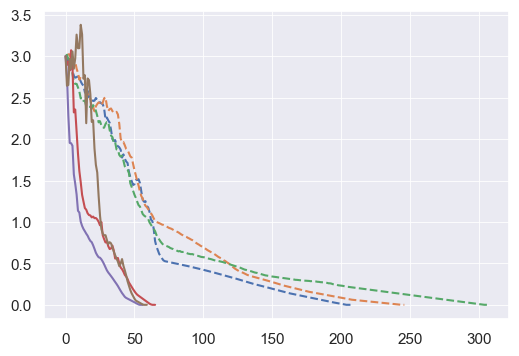}}
			\\[-0.1cm]
			& {\centering \;\;\;\; iteration} & & {\centering \;\;\;\; iteration}
		    \end{tabular}
		    \begin{tabular}{cc}
		        & (c)
		        \\
			    {\rotatebox{90}{ \;\;\;\; \;\;\;\; \;\;\;\; \;\;\;\; \# players}} 
			    \!\!\!\!\!\!
			    & {\includegraphics[scale=0.5]{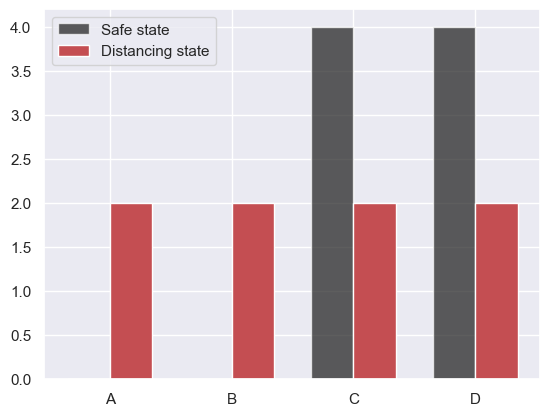}}
			    \\[-0.1cm]
			    & \;\;  action
		    \end{tabular}
		\end{center}
	\caption{ Convergence performance.
	(a) Learning curves for our independent policy gradient (\textbf{\color{applegreen}---}) with stepsize $\eta=0.005$ and the projected stochastic gradient ascent (\textbf{\color{awesome}---}) with $\eta=0.0001$ \citep{leonardos2021global}. Each solid line is the mean of trajectories over three random seeds and each shaded region displays the confidence interval.
	(b) Learning curves for six individual runs of our independent policy gradient (solid line) and the projected stochastic gradient ascent (dash line) three each. 
	(c) Distribution of players in one of two states taking four actions.
	In (a) and (b), we measure the accuracy by the absolute distance of each iterate to the converged Nash policy, i.e., $\frac{1}{N}\sum_{i\,=\,1}^N \Vert{\pi_i^{(t)}-\pi_i^{\text{\normalfont Nash}}}\Vert_1$. In our computational experiments, the initial distribution $\rho$ is uniform.
	}
	\label{fig: policy distance and trajectories unif5}
\end{figure}

\begin{figure}
		\begin{center}
		    \begin{tabular}{cccc}
		    & (a) & & (b)
		    \\
			{\rotatebox{90}{ \;\;\;\; \;\;\;\; \;\;\;\; \;\;  accuracy}} 
			\!\!\!\!\!\!
			& {\includegraphics[scale=0.55]{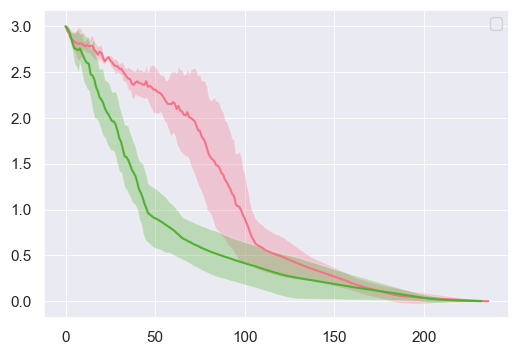}}
			\quad & \quad
			{\rotatebox{90}{ \;\;\;\; \;\;\;\; \;\;\;\; \;\;  accuracy}} 
			\!\!\!\!\!\!
			& {\includegraphics[scale=0.55]{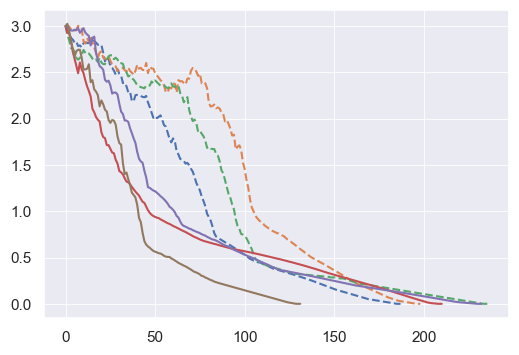}}
			\\[-0.1cm]
			& {\centering \;\;\;\; iteration} & & {\centering \;\;\;\; iteration}
		    \end{tabular}
		    \begin{tabular}{cc}
		        & (c)
		        \\
			    {\rotatebox{90}{ \;\;\;\; \;\;\;\; \;\;\;\; \;\;\;\; \# players}} 
			    \!\!\!\!\!\!
			    & {\includegraphics[scale=0.5]{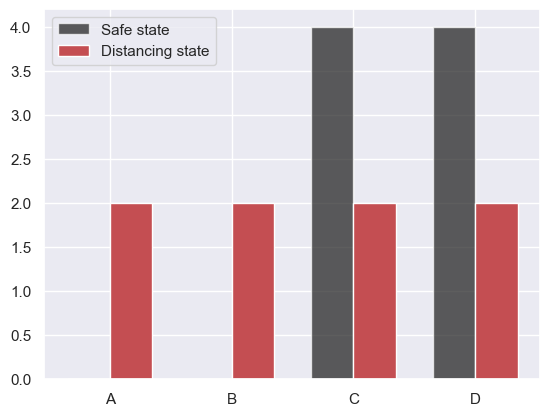}}
			    \\[-0.1cm]
			    & \;\; action
		    \end{tabular}
		\end{center}
	\caption{ Convergence performance.
	(a) Learning curves for our independent policy gradient (\textbf{\color{applegreen}---}) with stepsize $\eta=0.001$ and the projected stochastic gradient ascent (\textbf{\color{awesome}---}) with $\eta=0.0001$ \citep{leonardos2021global}. Each solid line is the mean of trajectories over three random seeds and each shaded region displays the confidence interval.
	(b) Learning curves for six individual runs of our independent policy gradient (solid line) and the projected stochastic gradient ascent (dash line) three each. 
	(c) Distribution of players in one of two states taking four actions.
	In (a) and (b), we measure the accuracy by the absolute distance of each iterate to the converged Nash policy, i.e., $\frac{1}{N}\sum_{i\,=\,1}^N \Vert{\pi_i^{(t)}-\pi_i^{\text{\normalfont Nash}}}\Vert_1$. In our computational experiments, the initial distribution is nearly degenerate $\rho = (0.9999,0.0001)$.
	}
	\label{fig: policy distance and trajectories deg1}
\end{figure}

\begin{figure}
		\begin{center}
		    \begin{tabular}{cccc}
		    & (a) & & (b)
		    \\
			{\rotatebox{90}{ \;\;\;\; \;\;\;\; \;\;\;\; \;\;  accuracy}} 
			\!\!\!\!\!\!
			& {\includegraphics[scale=0.55]{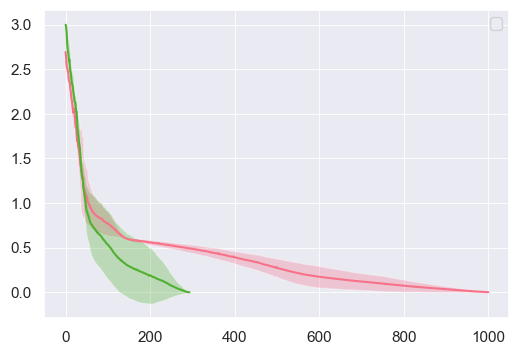}}
			\quad & \quad
			{\rotatebox{90}{ \;\;\;\; \;\;\;\; \;\;\;\; \;\;  accuracy}} 
			\!\!\!\!\!\!
			& {\includegraphics[scale=0.55]{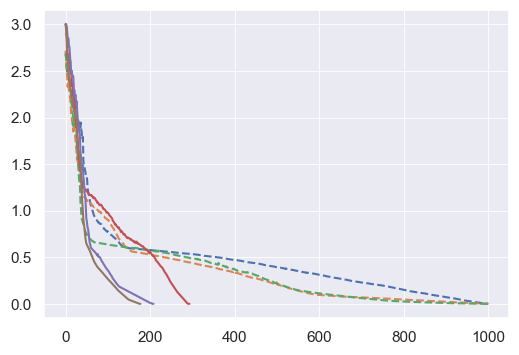}}
			\\[-0.1cm]
			& {\centering \;\;\;\; iteration} & & {\centering \;\;\;\; iteration}
		    \end{tabular}
		    \begin{tabular}{cc}
		        & (c)
		        \\
			    {\rotatebox{90}{ \;\;\;\; \;\;\;\; \;\;\;\; \;\;\;\; \# players}} 
			    \!\!\!\!\!\!
			    & {\includegraphics[scale=0.5]{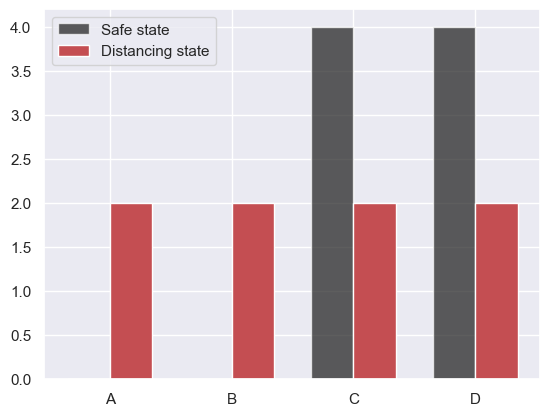}}
			    \\[-0.1cm]
			    & \;\; action
		    \end{tabular}
		\end{center}
	\caption{ Convergence performance.
	(a) Learning curves for our independent policy gradient (\textbf{\color{applegreen}---}) with stepsize $\eta=0.001$ and the projected stochastic gradient ascent (\textbf{\color{awesome}---}) with $\eta=0.0001$ \citep{leonardos2021global}. Each solid line is the mean of trajectories over three random seeds and each shaded region displays the confidence interval.
	(b) Learning curves for six individual runs of our independent policy gradient (solid line) and the projected stochastic gradient ascent (dash line) three each. 
	(c) Distribution of players in one of two states taking four actions.
	In (a) and (b), we measure the accuracy by the absolute distance of each iterate to the converged Nash policy, i.e., $\frac{1}{N}\sum_{i\,=\,1}^N \Vert{\pi_i^{(t)}-\pi_i^{\text{\normalfont Nash}}}\Vert_1$. In our computational experiments, the initial distribution is nearly degenerate $\rho = (0.0001,0.9999)$.
	}
	\label{fig: policy distance and trajectories deg2}
\end{figure}

\newpage
We also examine how sensitive the performance of algorithms depends on initial state distributions. As discussed in \pref{sec: gradient plays}, our independent policy gradient method \pref{eq: our policy gradient} is different from the projected policy gradient \pref{eq: policy gradient} by removing the dependence on the initial state distribution. In the policy gradient theory \citep{agarwal2021theory}, convergence of projected policy gradient methods is often restricted by how explorative the initial state distribution is. To be fair, we choose stepsize $\eta = 0.001$ for our algorithm since it achieves a similar performance as the projected stochastic gradient ascent \citep{leonardos2021global} in \pref{fig: policy distance and trajectories unif1}. We choose two different initial state distributions $\rho  = (0.9999,0.0001)$ and $\rho  = (0.0001,0.9999)$ and report our computational results in \pref{fig: policy distance and trajectories deg1} and \pref{fig: policy distance and trajectories deg2}, respectively. Compared \pref{fig: policy distance and trajectories deg1} with \pref{fig: policy distance and trajectories unif1}, both algorithms become a bit slower, but our algorithm is relatively insusceptible to the change of $\rho$. This becomes more clearer in \pref{fig: policy distance and trajectories deg2} for another $\rho  = (0.0001,0.9999)$. This demonstrates that practical performance of our independent policy gradient method \pref{eq: our policy gradient} indeed is invariant to the initial distribution $\rho$.


\end{document}